\documentclass{article}

\usepackage{amsthm}
\PassOptionsToPackage{numbers,compress}{natbib}
\usepackage[preprint]{neurips_2026}
\usepackage{enumitem}
\usepackage{amsthm}
\newtheorem*{claim*}{Claim}
\usepackage{amssymb}

\usepackage[utf8]{inputenc}
\usepackage[T1]{fontenc}
\usepackage{amsmath}
\usepackage{mathtools}
\usepackage{hyperref}
\hypersetup{hypertexnames=false}
\usepackage[capitalize]{cleveref}
\usepackage{url}
\usepackage{booktabs}
\usepackage{multirow}
\usepackage{placeins}
\usepackage{amsfonts}
\usepackage{nicefrac}
\usepackage{microtype}
\usepackage{xcolor}
\usepackage{calc}
\usepackage{csquotes}
\usepackage{doi}

\usepackage{algorithm}
\usepackage{algpseudocode}

\DeclareMathOperator{\sign }{sign}
\DeclareMathOperator{\diag}{diag}

\newcommand{\GF}{\ZZ_2}

\newcommand{\RR}{\mathbb{R}}

\newcommand{\ZZ}{\mathbb{Z}}

\newcommand{\NN}{\mathbb{N}}

\newcommand{\prism}{\mathrm{PRiSM}}
\newcommand{\fastsign}{\textsc{FastSign}}

\usepackage{thmtools}
\usepackage{thm-restate}
\newtheorem{lemma}{Lemma}
\newtheorem{theorem}{Theorem}
\newtheorem{corollary}{Corollary}

\theoremstyle{definition}
\newtheorem{definition}{Definition}
\newtheorem{remark}{Remark}

\theoremstyle{remark}

\newtheorem{example}{Example}
\newtheorem{claim}{Claim}
\Crefname{claim}{Claim}{Claims}
\newenvironment{claimproof}[1][Proof of Claim]{\begin{proof}[#1] }{ \end{proof}}
\title{Weisfeiler--Leman Is Incomplete on Simple Spectrum Graphs, so Canonicalize Them}

\author{%
 Snir Hordan\\
  Faculty of Mathematics\\
  Technion - Israel Institute of Technology
    \And 
  Nadav Dym \\
  Faculty of Mathematics\\
  Technion - Israel Institute of Technology
  \And
  Tim Seppelt\\
  IT University of Copenhagen
}

\begin{document}

\maketitle

\begin{abstract}
Graphs with a simple spectrum admit cubic-time isomorphism testing, yet we prove that for every natural number $k$, the $k$-Weisfeiler--Leman ($k$-WL) test cannot distinguish all non-isomorphic graphs with a simple spectrum. As the WL hierarchy upper-bounds the distinguishing power of widely-used Graph Neural Networks (GNNs), this incompleteness applies to all such GNNs, ruling out completeness for every $k$-WL-aligned GNN family. To close this gap, we introduce PRiSM (Partition, Refine, Solve, Match), the first provably complete canonicalization of simple-spectrum eigendecompositions. PRiSM obtains the completeness guarantee that prior canonicalizations provably lack, and resolves the open problem of achieving complete expressivity on simple-spectrum graphs. When composed with DeepSets or a Transformer, PRiSM achieves universal approximation on simple-spectrum graphs, justifying the use of canonicalized Laplacian positional encodings. Empirically, PRiSM performs comparably to or outperforms existing spectral canonicalizations on graph regression, classification, and expressivity benchmarks.
\end{abstract}

\section{Introduction}
\label{sec:int}
Graph Neural Networks (GNNs) are a leading machine learning paradigm for learning on graph-structured data such as chemical and molecular graphs \cite{duvenaud2015convolutional,gilmer2017neural, corso2024graph}. These networks typically maintain node representations and refine them using their local graph neighborhoods \cite{scarselli2008graph,kipf2017semisupervised, velivckovic2018graph}. To understand GNNs' expressive power, or their ability to distinguish among non-isomorphic graphs, it is common practice to compare the distinguishing power of GNNs with that of the Weisfeiler--Leman (WL) graph isomorphism test hierarchy \cite{xu2018powerful,maron2019provably, Morris_Ritzert_Fey_Hamilton_Lenssen_Rattan_Grohe_2019}. The WL hierarchy \cite{weisfeiler1968reduction} consists of combinatorial graph isomorphism tests that, analogously to GNNs, operate by a color refinement procedure, a procedure naturally aligned with the inner workings of many widely-used GNNs \cite{xu2018powerful,maron2019provably,morris2020weisfeiler, morris2023weisfeiler, pmlr-v235-muller24c}. Consequently, any incompleteness result on the WL hierarchy directly translates into expressivity limitations of the GNN architectures aligned with it.

We focus on understanding the expressivity of WL tests on graphs with a simple spectrum. Graphs with a simple spectrum are graphs whose adjacency matrix has distinct eigenvalues. These graphs are known to be distinguishable in cubic time \cite{LeightonMiller79, neuen_parameterized_2026}, yet it has been unclear whether $k$-WL tests can distinguish among such graphs \cite{hordanspectral}. Our first main result in the manuscript  provides a definitive answer: we prove that no test in the WL hierarchy, despite the hierarchy's ever-increasing expressive power and computational complexity, can distinguish among all graphs with a simple spectrum.

Our proof is constructive. For each $k$, we start from a $3$-regular base graph of sufficiently high treewidth and apply the Cai--F\"urer--Immerman (CFI) \cite{cai_optimal_1992} construction (\cref{fig:cfi-c3}) to obtain a pair of non-isomorphic graphs that $k$-WL cannot distinguish. Unfortunately, these graphs are not simple spectrum. To circumvent this, we introduce a novel mechanism of encoding these CFI adjacency patterns to obtain a pair of simple-spectrum multigraphs, on which $k$-WL still fails. This reveals a fundamental gap as each $k$-WL test runs in $\Theta(kn^{k+1}\log(n))$ time \cite{Grohe_2017}, yet these tests are not able to solve a combinatorial problem that can be solved in cubic time.

Our incompleteness result justifies the common practice of incorporating Laplacian positional encodings to GNNs and Graph Transformers \cite{dwivedi2023benchmarking, pmlr-v235-muller24c}, as these encodings add spectral information that WL tests and corresponding GNNs natively lack. However, naively using positional encoding introduces a sign ambiguity in addition to the inherent permutation ambiguity of the model. To address this ambiguity, sign-invariant models such as SignNet \cite{lim2023sign} and canonicalizations such as OAP and MAP \cite{ma2023laplacian, ma2024a} are used. Unfortunately, the sign invariance of these models comes with restricted expressivity, and such models provably cannot distinguish all simple-spectrum graphs \cite{hordanspectral}. Accordingly, \citet{hordanspectral} conclude that achieving a complete invariant model for simple-spectrum graphs is an open question. 

Our second main result in the manuscript  resolves this open question, by defining  a novel complete canonicalization for eigendecompositions of graphs with a simple spectrum, which we name $\prism$. Composing our canonicalization with a universal set network, such as DeepSets \cite{deepsets} or a Transformer \cite{vaswani2017attention, kim2021transformers}, yields a model that universally approximates any invariant function on simple-spectrum graphs. 

To allow $\prism$ to process arbitrary graphs, we define a heuristic extension of our canonicalization to graphs with repeated eigenvalues. Empirically, we find that $\prism$ outperforms competing canonicalizations and sign and basis invariant models. 

In summary, our main contributions are as follows:

\begin{enumerate}
    \item \textbf{Entire WL hierarchy is incomplete on graphs with simple spectrum} (\cref{thm:main-incomp}). We  prove that for every natural number $k$, there exist non-isomorphic simple-spectrum multigraphs on $\Theta(k)$ nodes that $k$-WL cannot distinguish. In a single step, this proves the incompleteness of every WL-aligned GNN family (PPGN, $k$-GNNs, $k$-WL-aligned Graph Transformers, N2GNNs, d-DRFWL(2) on the class of simple-spectrum graphs, see \cref{cl:wl-nets}).
    \item \textbf{First complete canonicalization} (\cref{thm:s-s-canon}, \cref{thm:uni-bdd-ev}). We introduce $\prism$, the first provably complete sign canonicalization of simple-spectrum eigendecompositions. Composed with DeepSets or a Transformer, $\prism$ achieves universal approximation on simple-spectrum graphs (\cref{thm:uni-bdd-ev}), supplying the completeness guarantee absent from \citet{ma2023laplacian, ma2024a, lim2023sign} and resolving the open problem of universality on simple-spectrum graphs raised by \citet{hordanspectral}.
    \item \textbf{Empirical results} (\cref{sec:exps}). $\prism$ improves ability to distinguish non-isomorphic graph pairs on the BREC \cite{brec} expressivity benchmark, and performs on par or better compared with competing canonicalizations on widely-used molecular datasets (ogbg \cite{hu2020open}, ZINC \cite{irwin2012zinc}, Alchemy \cite{chen2019alchemy}).
\end{enumerate}

The remainder of this paper is organized as follows. \Cref{sec:rw} surveys related work and previous canonicalizations for graph positional encodings. \Cref{sec:pre} introduces the necessary background on graphs with bounded eigenvalue multiplicity, the Weisfeiler--Leman hierarchy, and the CFI construction. \Cref{sec:incomp} presents our incompleteness result, showing that no level of the WL hierarchy suffices for simple-spectrum graphs. \Cref{sec:canon-simple-spec} develops the canonicalization procedure and its completeness and universality guarantees. \Cref{sec:exps} evaluates our method empirically on expressivity and molecular benchmarks.

\section{Related Work}\label{sec:rw}

\paragraph{Weisfeiler--Leman expressivity.}
The WL hierarchy is the primary framework for GNN expressivity: message-passing GNNs are equivalent to $1$-WL \cite{xu2018powerful, Morris_Ritzert_Fey_Hamilton_Lenssen_Rattan_Grohe_2019}, higher-order architectures simulate higher levels (PPGN \cite{maron2019provably}, $k$-GNNs \cite{Morris_Ritzert_Fey_Hamilton_Lenssen_Rattan_Grohe_2019}, N2GNNs \cite{zhou2023distance}, d-DRFWL(2) \cite{feng2023extending}), and Subgraph GNNs are upper-bounded by $3$-WL \cite{frasca2022understanding}.
The Cai--F\"urer--Immerman construction proves the $k$-WL incompleteness on general graphs for every $k$ \cite{cai_optimal_1992,neuen_homomorphism-distinguishing_2024}. We show that this incompleteness exists even for  graphs with a simple spectrum, a class for which graph isomorphism is polynomial-time solvable \cite{LeightonMiller79, babai}.

\paragraph{Spectral positional encodings and ambiguity.}
Laplacian eigenvector positional encodings underlie modern graph transformers \cite{kreuzer2021rethinking, dwivedi2021generalization, dwivedi2023benchmarking}, but are defined only up to sign (simple eigenvalues) or orthogonal basis (repeated eigenvalues) ambiguity. To address these ambiguities, methods such as SignNet, BasisNet, MAP and OAP were introduced \cite{lim2023sign,ma2023laplacian, ma2024a}, but they are incomplete on simple-spectrum graphs. \citet{gai2025homomorphism, zhangexpressive} characterize spectral invariant GNNs via homomorphism counts and prove that spectral invariant GNNs lie strictly between $1$- and $3$-WL. \citet{hordanspectral} prove that spectral invariant GNNs are incomplete on simple-spectrum graphs. As spectral invariant GNN expressivity is bounded by $3$-WL \cite{zhangexpressive}, our hierarchy-wide incompleteness (\cref{thm:main-incomp}) subsumes the spectral GNN incompleteness of \citet{hordanspectral} as an immediate corollary.

\section{Preliminaries}
\label{sec:pre}

\subsection{Graphs with Bounded Eigenvalue Multiplicity}
We use the standard notation  $G=(V,E) $ for combinatorial graphs. All graphs in this manuscript are undirected. A weighted graph is a triplet $g=(V,E,A) $ where $A$ is a symmetric $|V|\times |V| $ matrix with real values, satisfying that  $A_{uv}=0$  if $(u,v) \not \in E$. A \emph{multigraph} is a weighted graph with edge weights $A_{u,v}\in \NN $ for all $(u,v)\in E$. We note that our definitions allow a weighted graph matrix $A$ to have negative entries, and non-zero diagonal entries. Thus the diagonal elements of $A$ can encode node features, and $A$ can represent the graph's Laplacian, or normalized Laplacian, as well as its adjacency matrix. We denote the set of weighted graphs with $n$ vertices by $\mathcal{G}_n $. 

An \emph{eigendecomposition} of a weighted graph $g=(V,E,A)$ in $\mathcal{G}_n $ is an eigendecomposition of the matrix $A$. Namely, it is a  triplet $(U, \vec{\lambda}, \vec{m})$, where $U \in \mathbb{R}^{n \times n}$ is an orthogonal matrix of eigenvectors (rows indexed by nodes, columns by eigenvectors), $\vec{\lambda} = (\lambda_1, \ldots, \lambda_r) \in \mathbb{R}^r$ is the vector of $r \leq n$ distinct eigenvalues with $\lambda_1 < \cdots < \lambda_r$, and $\vec{m} = (m_1, \ldots, m_r) \in [n]^r$ is the \emph{multiplicity vector} with $\sum_{i=1}^r m_i = n$, such that $$A = U \, \mathrm{diag}(\underbrace{\lambda_1, \ldots, \lambda_1}_{m_1}, \ldots, \underbrace{\lambda_r, \ldots, \lambda_r}_{m_r}) \, U^\top.$$

For $k \in \{1, \ldots, n\}$, the \emph{$k$-eigendecomposition} $(U_k, \vec{\lambda}_k, \vec{m}_k)$ consists of the first $k$ columns $U_k \in \mathbb{R}^{n \times k}$, the $r_k \leq k$ distinct eigenvalues among the $k$ smallest, and their multiplicities with $\sum_{i=1}^{r_k} m_i = k$.

An eigendecomposition is not unique: for each eigenvalue $\lambda_i$ with multiplicity $m_i$, any orthogonal transformation $Q_i \in O(m_i)$ applied to the corresponding eigenvectors yields an equally valid eigenbasis. The ambiguity group is therefore $\mathcal{A}(\vec{m}_k) = \prod_{i=1}^{r_k} O(m_i)$. 

When all eigenvalues are distinct ($m_i = 1$ for all $i$), this reduces to $\mathcal{A}(\vec{m}_k) =\{-1,1\}^k \cong \mathbb{Z}_2^k$, which can be viewed as a sign flip $u_i \mapsto -u_i$ ambiguity on each eigenvector, often called \emph{sign ambiguity}. We say such a graph has a \emph{simple spectrum} and write $g \in \mathcal{G}_n^{\mathrm{simple}}$. Graph isomorphism is solvable in cubic time on $\mathcal{G}_n^{\mathrm{simple}}$ \cite{LeightonMiller79, neuen_parameterized_2026}, and Erd\H{o}s--R\'enyi random graphs almost surely lie in $\mathcal{G}_n^{\mathrm{simple}}$ as $n \to \infty$ \cite{tao2017random}.

\subsection{The Weisfeiler--Leman Test Hierarchy}
The subject of our incompleteness results is the Weisfeiler--Leman test hierarchy \cite{weisfeiler1968reduction}. It is a classic combinatorial heuristic for graph isomorphism testing, central to analysing the expressivity of graph neural networks \cite{xu2018powerful}. We give the explicit $1$-dimensional case, also known as \emph{color refinement}, and defer the general $k$-WL formulation to \cref{app:wl-defs}.

\begin{definition}[$1$-WL / Color Refinement]\label{def:1-wl}
Let $g=(V,E,A) \in \mathcal{G}_n$ be a weighted graph. The $1$-WL algorithm maintains a coloring $C^{(t)} : [n] \to \mathcal{C}$ over a discrete color set $\mathcal{C}$. It is initialized by   $C^{(0)}(v) = \mathrm{HASH}\left(A(v,v)\right)$, where $\mathrm{HASH}$ is an injective encoding of the node features $A(v,v) $.  This coloring is then refined  via
\[
    C^{(t+1)}(v) = \mathrm{HASH}\!\left(C^{(t)}(v),\, \left\{\!\!\left\{ \bigl(C^{(t)}(w),\, A_{vw}\bigr) \;\middle|\; w \in [n] \right\}\!\!\right\}\right),
\]
where $\mathrm{HASH}$ is a fixed injective map. The algorithm terminates at the smallest $T$ for which $C^{(T+1)}$ induces the same partition of $[n]$ as $C^{(T)}$, and outputs the graph-level label
\[
    C_{\mathrm{global}}(g) = \mathrm{HASH}\!\left(\left\{\!\!\left\{ C^{(T)}(v) \;\middle|\; v \in [n] \right\}\!\!\right\}\right).
\]

Two graphs $g, g' \in \mathcal{G}_n$ are \emph{$1$-WL indistinguishable}, written $g \sim_{1} g'$, if $C_{\mathrm{global}}(g) = C_{\mathrm{global}}(g')$.
\end{definition}

For each $k \geq 1$ the $k$-dimensional generalisation $k$-WL colors $k$-tuples of vertices and yields a strictly finer indistinguishability relation $\sim_k$, where $g \sim_{k+1} g'$ implies $g \sim_{k} g'$, with the converse failing in general \cite{cai_optimal_1992, grohe2015pebble, Grohe_2017}.

\subsection{CFI graph construction}
We review the Cai--F\"urer--Immerman \cite{cai_optimal_1992} (CFI) graph construction, a standard method for producing pairs of graphs indistinguishable by $k$-WL, and the entrance point for our proof of \cref{thm:main-incomp}.

\begin{definition}[CFI construction \cite{roberson_oddomorphisms_2022}]\label{def:cfi}
Let $G$ be a combinatorial graph and $U \subseteq V(G)$. Write $E(v) \coloneqq \{e \in E(G) : v \in e\}$ for edges incident to~$v$. We define the \emph{CFI graph} $G_{U}$ to be the graph with  vertex set $\{(v, S) \mid v \in V(G),\; S \subseteq E(v),\; |S| \equiv |\{v\} \cap U| \pmod 2\}$. We connect two vertices $(v, S)$ and $(u, T)$ in $G_U$ by an edge if and only if $(u,v) \in E(G)$ and $(u,v) \notin S \mathbin\triangle T$.  We call $G$ the \emph{base graph}, and its vertices the \emph{base vertices}, while vertices $(v,S)$ in $G_U$ we call \emph{fiber vertices}.
\end{definition}

The subset $U$ controls a parity twist: fibers at $v \notin U$ use even-cardinality subsets and at $v \in U$ odd-cardinality subsets. By \cite[Corollary~3.7]{roberson_oddomorphisms_2022}, for any connected $G$, there are exactly two non-isomorphic CFI graphs: $G_0$ ($|U|$ even) and $G_1$ ($|U|$ odd). When $G$ has treewidth at least $k+1$, $G_0$ and $G_1$ are $k$-WL indistinguishable \cite{neuen_homomorphism-distinguishing_2024}, \cite[Corollary V.7]{grohe_logic_2021}.

\begin{example}[CFI construction on $C_3$]\label{ex:cfi-c3}
Consider $G = C_3$ on $\{0, 1, 2\}$. Each vertex has degree~$2$, yielding two fiber vertices per base vertex and six CFI vertices in total. In the trivial lift ($U = \varnothing$), all fibers use even-cardinality subsets. The vertices split into two groups, the $\varnothing$-type and the full-set type, with each forming an independent triangle, so $G_0 \cong 2\,C_3$. In the non-trivial lift ($U = \{0\}$), the parity constraint at vertex~$0$ shifts to odd-cardinality subsets, rewiring edges incident to vertex~$0$ and bridging the two triangles into a single $6$-cycle: $G_{\{0\}} \cong C_6$. See \cref{fig:cfi-c3} for an illustration. We note that $G_0$ and $G_1$ are the classic examples of graphs that are $1$-WL indistinguishable. Full details are given in \cref{app:cfi-c3-details}.
\end{example}

\begin{figure}[!h]
    \centering
    \includegraphics[width=0.95\textwidth]{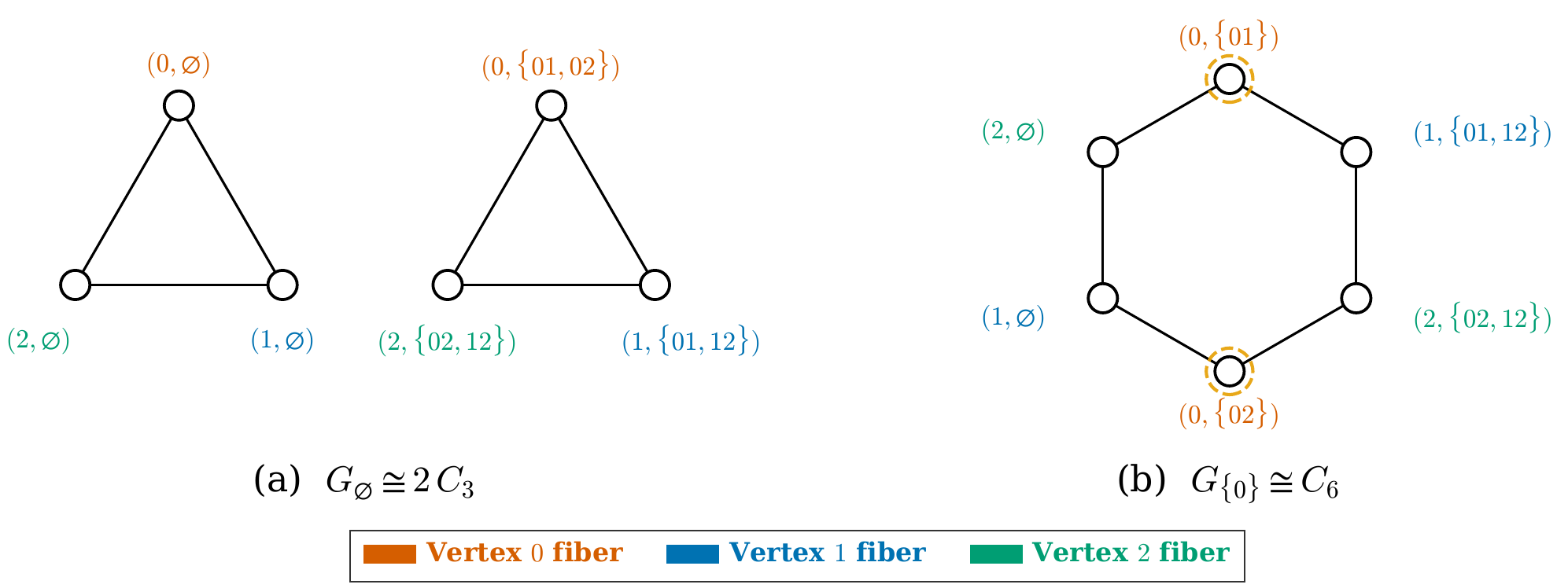}
    \caption{\it The CFI construction applied to the base graph $G = C_3$.
    \textbf{(a)}~Trivial lift $G_\varnothing \cong 2\,C_3$:
    the $(i, \varnothing)$ vertices and $(i, E(i))$ vertices
    each form an independent triangle.
    \textbf{(b)}~Non-trivial lift $G_{\{0\}} \cong C_6$:
    the parity twist at vertex~$0$ bridges the two
    triangles into a single $6$-cycle (dashed circles mark the twisted fiber).
    Colors indicate base-graph vertex correspondence and are for visualization purposes only.}
    \label{fig:cfi-c3}
\end{figure}

\section{Incompleteness of the WL hierarchy on graphs with a simple spectrum}
\label{sec:incomp}
We introduce our main result, which proves that for any $k$-WL test, there exists a simple-spectrum pair of multigraphs that it cannot distinguish.

\begin{restatable}{theorem}{mainIncomp}
    \label{thm:main-incomp}
    For every $k \in \mathbb{N}$, there exist non-isomorphic simple-spectrum multigraphs with $\Theta(k)$ nodes that are not distinguished by the $k$-dimensional Weisfeiler--Leman algorithm.
\end{restatable}

The full proof is rather involved, and is provided in \cref{app:incomp-full}. Our main goal in this section is to describe the main ideas of the proof. 

\begin{proof}[Proof outline]

For every $k\in \NN$, we will construct a pair of simple-spectrum graphs by encoding the adjacency patterns of the CFI graphs, which are constructed with a carefully chosen base graph, into a valid simple-spectrum eigendecomposition via an integral encoding scheme. The novelty over the classical CFI construction lies in Step~2: an integral encoding $\tilde X = [X \mid X' \mid I]$ that turns CFI's $k$-WL-indistinguishable adjacency pattern into an orthogonal eigenvector matrix paired with pairwise distinct integer eigenvalues, simultaneously a valid eigendecomposition and one that inherits the corresponding CFI graphs' $k$-WL indistinguishability. Steps~1 and~3 are otherwise standard CFI machinery.

\textit{Step 1: CFI Graph Construction.}
Our proof begins with a CFI base graph $G$ that is  $3$-regular, has $2n$ vertices, and has treewidth of at least $\Omega(n)$. Such graphs exist by \cite{dvorak_strongly_2016, kolesnik_lower_2014}. By \citet{neuen_homomorphism-distinguishing_2024}, the $\Omega(n)$ treewidth of the base graph $G$ implies that the CFI pair $(G_0, G_1)$ is $\Omega(n)$-WL-indistinguishable. 
By \cite{roberson_oddomorphisms_2022}, the two CFI graphs are non-isomorphic.
The reason for the assumption that $G$ is 3-regular will be apparent later on.

We note that if the CFI graphs $G_0, G_1$ were already simple spectra, the proof would be immediate.  However, we find that the adjacency matrix in these constructions is typically far from having a simple spectrum (see \cref{sub:not_simple} for a formal proof). Therefore, we encode CFI graphs as orthogonal matrices, and use these matrices to define a new multigraph, for which these orthogonal matrices will serve as an eigendecomposition.

\textit{Step 2: Orthogonal Encoding of the CFI Graphs}\label{def:integral-enc}
In this step we construct an \emph{orthogonal integral encoding} of a CFI graph $G_U$ as a matrix $$\tilde X_{G,U} = [\,X \mid X' \mid I\,] \in \mathbb{Z}^{8n\times 8n}.$$ 
The blocks $X$ and $X' $ are indexed by a fiber vertex $(w,S)$ and a base edge $e$, and are
defined via 
\[
\resizebox{\textwidth}{!}{$%
X((w,S),e) = \begin{cases} +1 & e\in E(w)\cap S, \\ -1 & e\in E(w)\setminus S, \\ \phantom{+}0 & e\notin E(w), \end{cases}
\quad
X'((w,S),e) = \begin{cases} +X((w,S),e) & w=u, \\ -X((w,S),e) & w=v, \\ \phantom{+}0 & \text{otherwise,} \end{cases}
\quad
\text{ where } e=(u,v)
$}
\]
The matrix $I_{G,U}$ is indexed by a fiber vertex $(w,S)$ and a base vertex $v$, and is of the general form $I((w,S),v) = W(w, v)$, where the matrix $W$ is defined in \cref{eq:matrix-A} in the appendix. An important property of this matrix is that its first column is the all-ones vector. 

An example of the matrix $\tilde X_{G,U} = [X \mid X' \mid I]$ for the graph $G=C_3$ is provided in the appendix in  \cref{app:integral-c3-details} and  \cref{fig:integral-encoding}.

The columns of the matrix $\tilde X_{G,U}$ are orthogonal. This is shown in \cref{lem:orthogonal} in the appendix, and this result relies on the fact that the base graph $G$ is $3$-regular. To get some intuition for the importance of the regularity assumption, we explain here why the $3$-regularity assumption leads to the matrix $\tilde X_{G,U}$ being a square $8n\times 8n$ dimensional matrix.

In general, the matrix $\tilde X_{G,U}$ has $|V(G_U)|$ rows and $2|E|+|V|$ columns. Since the base graph $G$ is a $3$-regular graph with $|V|=2n$ vertices, it has $|E|=3n$ edges, and so the number of columns of $\tilde X_{G,U}$ will be $2|E|+|V|=8n$. The matrix will also have  $|V(G_U)|=8n $ rows. This is because for each of the $2n$ vertices $v$ in the base graph, there are in total $2^3=8$ sets $S$ in $ E(v) $, as $|E(v)|=3$. Half of these sets $S$ can be paired with $v$ to form a fiber vertex $(v,S)$. Therefore $|V(G_U)|=2n\cdot 4=8n$.

\textit{Step 3: Multigraph Construction.} We now use the orthogonal matrix $ \tilde X_{G,U} $ to form a new \emph{weighted graph} encoded by an adjacency matrix $A_{G,U} = \tilde X_{G,U}\operatorname{diag}(\vec{\lambda})(\tilde X_{G,U})^\top$, where we choose pairwise distinct integer eigenvalues $\lambda_1 < \cdots < \lambda_{8n}$. This yields a symmetric matrix $A_{G,U}$ with integer values and a simple spectrum (see the proof of \cref{thm:main-incomp} in \cref{app:main-incomp}). By assigning a sufficiently large eigenvalue to the constant eigenvector, the entries of $A_{G,U}$ become non-negative integers, producing a proper multigraph.

In conclusion, we obtain a pair of symmetric matrices $A_{G,0}$ and $A_{G,1}$ which encode simple-spectrum multigraphs. These graphs were obtained from non-isomorphic graphs $G_0, G_1$ which are $k$-WL-indistinguishable, and in the appendix we show how this implies that the multigraphs represented by $A_{G,0}$ and $A_{G,1}$ are also $k$-WL-indistinguishable, which then concludes the proof of the theorem.    
\end{proof}
As many GNNs are upper-bounded in their expressivity by a $k$-WL graph isomorphism test, we obtain an immediate corollary.
\begin{corollary}[Incompleteness of GNNs on Graphs with a Simple Spectrum]\label{cl:wl-nets}
    The widely-used Graph Neural Networks PPGN \cite{maron2019provably}, $k$-GNNs \cite{Morris_Ritzert_Fey_Hamilton_Lenssen_Rattan_Grohe_2019}, $k$-WL-aligned Graph Transformers \cite{pmlr-v235-muller24c}, N2GNNs \cite{zhou2023distance}, and d-DRFWL(2) GNNs \cite{feng2023extending} cannot distinguish all simple-spectrum graphs. 
\end{corollary}

\section{Canonicalizing eigendecompositions of graphs with a simple spectrum}\label{sec:canon-simple-spec}
We saw in \cref{cl:wl-nets} that even powerful high order GNNs cannot distinguish between all simple-spectrum graphs. This can be seen as theoretical justification for the necessity to integrate the spectral eigendecomposition directly into GNNs to obtain improved expressivity. This is exemplified by the prevalent practice of using spectral positional encodings (e.g. Laplacian Positional Encodings \cite{dwivedi2023benchmarking}) to enhance the expressivity of GNNs and graph transformers \cite{dwivedi2023benchmarking, pmlr-v235-muller24c}, as this spectral encoding adds information which $k$-WL methods cannot capture. However, the spectral encoding contains sign ambiguities (in the simple-spectrum case), and previous sign-invariant methods for handling these ambiguities cannot distinguish between all simple-spectrum graphs \cite{hordanspectral}. 

In this section we provide a solution: a complete canonicalization for the eigendecomposition of simple-spectrum graphs.  We begin by defining such a  canonicalization. 

The input to a  canonicalization of a simple-spectrum eigendecomposition  will be a column-wise orthonormal $n\times k$ matrix $U$, where $n \geq k$. The columns of the matrix represent the top $k$ eigenvectors of a symmetric matrix representing the graph, such as the graph adjacency or Laplacian matrices.  Under the simple-spectrum assumption, the matrix $U$ is uniquely defined up to a sign ambiguity in choosing each of the eigenvectors, and a permutation ambiguity coming from the ambiguity in ordering  the nodes of the graph. Thus, the ambiguity group of $U$ is $ S_n \times \{-1,1\}^k$, where $\{-1,1\}^k $  corresponds to the ambiguity group $\mathcal{A}(\vec{m}_k)$ when all eigenmultiplicities are $1$ (see \cref{sec:pre}).

\begin{definition}\label{def:canon}
Let $k\leq n$ be natural numbers. A canonicalization of a simple-spectrum eigendecomposition $c:\RR^{n \times k}\to \RR^{n \times k} $ is a mapping which is $S_n\times \{-1,1\}^k $ invariant, and satisfies that $c(U) $ is in the orbit of $U$,  for all column-wise orthonormal matrices $U\in \RR^{n \times k} $. 
\end{definition}
A canonicalization $c$ can be composed with any `backbone' model $\psi$ to obtain a model $\psi\circ c $ which is invariant. Such models will be able to distinguish between simple-spectrum graphs, which also leads to universality guarantees, as we will discuss in \cref{thm:uni-bdd-ev} below. Motivated by these observations, we now discuss how to construct  a canonicalization. 

Our canonicalization is inspired by ideas from the classical polynomial time algorithm by \citet{babai} for distinguishing between simple spectrum and bounded eigenmultiplicity graphs  (see also  \cite{spielman}). It is based on four stages, abbreviated $\prism$ = \textbf{P}artition, \textbf{R}efine, \textbf{S}olve, \textbf{M}atch. In the following we give a short overview of these stages, using the illustrative example below which we will explain as we go along. Arrow labels in the example use the short-form stage names $\{$Partition, Refine, Solve, Match$\}$.
{\setlength{\arraycolsep}{2pt}
\newcommand{\stagelbl}[1]{\makebox[\widthof{\text{Partition}}][c]{\text{#1}}}
\begin{align}
U = {} & \begin{pmatrix}
     +1 & +1 & +2 & +1 \\
     -1 & 0 & -1 & -2 \\
     -1 & -1 & -2 & -1 \\
     -1 & 0 & +1 & +2
   \end{pmatrix}
   \xrightarrow{\stagelbl{Partition}}
   \begin{pmatrix}
     +1 & +1 & +2 & +1 \\
     -1 & -1 & -2 & -1  \\ \hline
     -1 & 0 & -1 & -2 \\
     -1 & 0 & +1 & +2
   \end{pmatrix}
   \xrightarrow{\stagelbl{Refine}}
   \begin{pmatrix}
     +1 & +1 & +2 & +1 \\
     -1 & -1 & -2 & -1  \\ \hline
     -1 & 0 & -1 & -2 \\
     -1 & 0 & +1 & +2
   \end{pmatrix} \nonumber\\
\xrightarrow{\stagelbl{Solve}} {} & \begin{pmatrix}
      -1 & -1 & -2 & -1 \\
     +1 & +1 & +2 & +1  \\ \hline
     +1 & 0 & +1 & +2 \\
     +1 & 0 & -1 & -2
   \end{pmatrix}
   \xrightarrow{\stagelbl{Match}}
      \begin{pmatrix}
     +1 & +1 & +2 & +1  \\
     +1 & 0 & +2 & +1 \\
     +1 & 0 & -1 & -2\\
        -1 & -1 & -1 & -2 \\
   \end{pmatrix} = \prism(U)
   \label{eq:example}
\end{align}}

\textbf{Stage~1: \textsc{Partition}} (\cref{alg:abs-value-partition} in the appendix).

In the first stage, we define a (partial) order on the vertices of the graph. Each vertex $v$ is assigned the sign-invariant signature $\sigma(v)=(|U_{v,1}|,\ldots,|U_{v,k}|)$, and then the vertices are sorted lexicographically according to this signature. When each vertex is assigned a unique signature, this defines a unique ordering on the vertices. In the general case, this only groups vertices into classes $\{C_\ell\}$ on which $|U_i|=|U_j|$ for all $i,j\in C_\ell$. Thus, there is a canonical ordering between the different $C_\ell$, but not within each class.

In our example, after applying the partial ordering according to $\sigma$ we remain with a matrix whose first two  rows share absolute-value signature $(1,1,2,1)$ while the remaining rows share the absolute value signature $(1,0,1,2)$, thus yielding a partition of the nodes into two classes $C_1=\{1,2\}, C_2=\{3,4\}$.

\textbf{Stage~2: \textsc{Refine}} (\cref{alg:balanced-refine} in the appendix).
In the second stage, we refine the partition from Stage~1 by running a spectral variant of the Weisfeiler--Leman (WL) color refinement procedure \cite{zhangexpressive}: each node $v$ is initialized with the absolute-value signature $h_v^{(0)}=U_v\odot U_v$ (where $\odot$ represents elementwise multiplication, which is invariant to sign ambiguities), and colors are then iteratively refined using the multiset of edge features $\{(h_u^{(t)},\,U_v\odot U_u):u\in[n]\}$ until the partition stabilizes (\cref{alg:balanced-refine}). At the fixed point, every two nodes $i,j$ in the same class share the same multiset of pairwise products $\{U_i\odot U_k:k\in C_\ell\}=\{U_j\odot U_k:k\in C_\ell\}$, which is the structural property we will use in the canonicalization proof. In our running example, this stage does not further refine the initial partition from Stage~1.

\textbf{Stage~3: \textsc{Solve}} (\cref{alg:sign-solve} in the appendix). The Solve step determines the signs of the eigenvectors by requiring that each node partition $C$ contains a node whose corresponding row has only non-negative entries. In our example, this requirement on $C_1$ and $C_2$ simultaneously forces $s_1=s_2=s_3=s_4=-1$, the unique sign choice that yields a non-negative row in both blocks.

In general, we will show in \cref{apdx:canon-simple} that the constraints that each $C_k$ contains a row with non-negative entries can be written as linear equations in the field $\ZZ_2=\{0,1\}$, and  that (after eliminating linearly dependent equations) these equations uniquely determine the sign of all eigenvectors\footnote{This is true when the graph has no automorphisms. When the graph has automorphisms, then the equations have a unique solution up to these automorphisms. In either case, we obtain a valid canonicalization as discussed in the appendix.}.

\textbf{Stage~4: \textsc{Match}} Once the sign of the eigenvectors is determined, we can apply standard lexicographical sorting to the rows of the matrix to fix the vertex permutation. This concludes the canonicalization process.

This procedure induces a general canonicalization of eigendecompositions of any simple-spectrum graph. We call the resulting canonicalization $\prism$, short for \textbf{P}artition \textbf{R}efine \textbf{S}olve \textbf{M}atch, naming the four pipeline stages above (the final \textsc{Match} stage uses lex-sort to match each row to its canonical position). In the appendix (\cref{apdx:canon-simple}) we give the full formal construction (assembling per-class parity-check matrices, embedding them in a global linear system over $\ZZ_2$, and lex-minimizing over its solution coset) and prove invariance to vertex permutations and sign flips, encapsulated by the theorem below.

\begin{theorem}[Simple Spectrum Canonicalization] \label{thm:s-s-canon}
   For every natural $k\leq n$, the base procedure $\prism$ of \cref{alg:canon-one-full}, applied to a $k$-eigenvector matrix $U\in\RR^{n\times k}$ of a simple-spectrum graph, is a valid canonicalization in the sense of \cref{def:canon}.
\end{theorem}

\textbf{Universality.}
The theoretical significance of canonicalization is that it implies universal approximation, as shown in previous papers on canonicalizations and frames \cite{puny2022frame,pmlr-v202-kaba23a}. In our context, we can prove  universality on simple-spectrum graphs, when composing the canonicalization with any function which is universal on sets of vector features, such as DeepSets \cite{deepsets,amir2023neural, wagstaff2022universal} and Transformers \cite{kim2021transformers}.

\begin{restatable}[Universal Approximation over Graphs with a Simple Spectrum]{theorem}{universalapprox}\label{thm:uni-bdd-ev}
    Fix $M>0$ and $\delta>0$. Let $\mathcal{K}(N,M,\delta)$ denote the set of simple-spectrum weighted graphs with at most $N$ vertices, edge weights bounded by $M$, and minimum eigenvalue gap at least $\delta$ (i.e.\ $|\lambda_i-\lambda_j|\geq\delta$ for $i\neq j$). Let $f:\mathcal{K}(N,M,\delta)\to\RR$ be an isomorphism-invariant function and let $\epsilon>0$. Denote the eigendecomposition of a graph $g$ by $(U_g,\vec{\lambda}_g)$. Then for any universal set model there exists a model instantiation $\psi$ such that
    \begin{equation*}
\forall g \in \mathcal{K}(N,M,\delta), \quad        |f(g)-\psi \circ \prism(U_g,\vec{\lambda}_g)   | < \epsilon.
    \end{equation*}
\end{restatable}
The eigengap assumption $\delta>0$ is necessary, not restrictive, since simple-spectrum graphs form an open, non-compact subset of $\mathcal{G}_n$ while uniform universal approximation requires compactness. In \cref{app:universality} we formalize the notion of universal set models and prove \cref{thm:uni-bdd-ev}, settling an open problem of \citet{hordanspectral}.

\textbf{Relation to previously proposed canonicalizations.}
Other canonicalizations for eigendecompositions of simple-spectrum graphs were proposed in the past,  such as MAP~\cite{ma2023laplacian} and OAP~\cite{ma2024a}. However, these canonicalizations operate on each eigenvector independently, and as a result they are not able to choose a canonical sign for \emph{uncanonicalizable eigenvectors}, which are  eigenvectors $v$ which can be reordered to obtain $-v$. For example, the last three eigenvectors (columns) of the matrix $U$  in \eqref{eq:example} are uncanonizable, so that MAP and OAP will not be able to choose a canonical sign to either of them. In contrast, our canonicalization, which is valid for all simple-spectrum graphs, does define a global canonical sign for all eigenvalues as shown in \eqref{eq:example}.

\textbf{Practical aspects.}
The canonicalization $\prism$ can be computed using $O(nk\log n + k^3 + nk^2)$ basic operations (see \cref{apdx:complexity}). While this complexity is cubic in $k$, we note that $k$ is typically not very large in our experiments, and the complexity is linear in $n$. Moreover, the canonicalization is only carried out as a preprocessing step and not during training, and its complexity is comparable with the $O(km + k^2 n)$ complexity of computing the eigendecomposition of a graph with $m$ edges using the Lanczos algorithm, which is necessary for all canonicalizations.

Moreover, we note that the $O(nk\log n + k^3 + nk^2)$ complexity is a worst-case bound which is typically not attained. When the absolute-value signature is injective, which holds on $14$--$45\%$ of validation graphs across ZINC and the eleven \textsc{OGB-Mol} datasets (see \cref{apdx:complexity} and \cref{tab:rbound-stats} in the appendix), canonicalization reduces to an algorithm we call $\fastsign$ (see \cref{alg:fast-sign} in the appendix for a full description) that has only $O(nk \log n)$ run-time. On the remaining graphs, the $\ZZ_2$ system solved by \textsc{Solve} has median row-count at most $3$ across every dataset, so the cubic $k^3$ factor reduces to a small constant in practice, see \cref{apdx:complexity} for details.

Finally, we extend $\prism$ heuristically to graphs which are not simple spectrum: we first apply  $\prism$  on all $m=1$ eigenspaces, producing a canonical sign assignment for these eigenspaces, together with a node ordering. This node ordering  is further refined by generalizations of Stages~1 and~2 to higher-order eigenspaces. The rotation degree of freedom in high-order eigenspaces ($m>1$) is then canonicalized according to the first $m$ nodes obtained from the ordering, using   a QR decomposition. On a simple-spectrum graph this extension coincides with the base $\prism$ of \cref{alg:canon-one-full} and inherits its completeness, whereas on graphs with eigenvalue multiplicities greater than $1$ it is heuristic. Full pseudocode for the extension is in \cref{alg:babai-hybrid} in the Appendix.

\section{Experiments}\label{sec:exps}

We evaluate $\prism$ on graph learning benchmarks, and compare against vanilla Laplacian Positional Encodings \cite{dwivedi2023benchmarking} (which ignore sign ambiguity), two canonicalizations: MAP \cite{ma2023laplacian}, OAP \cite{ma2024a}, and the sign-invariant SignNet \cite{lim2023sign}. \textbf{Across all three benchmarks $\prism$ performs comparably to or beats every baseline}. We hypothesize that $\prism$ outperforms the incomplete canonicalizations because they either lose information \cite{lim2023sign} or are non-deterministic \cite{ma2023laplacian, ma2024a} on graphs with simple spectrum. We first benchmark BREC graph distinguishability (\cref{tab:brec-comparison}), then molecular property prediction with GNN backbones (\cref{tab:molecular-results}) and Transformer backbones (\cref{tab:transformer-results}). Full experiment details are in \cref{app:exp-details}.

\subsection{BREC expressivity benchmark}

We evaluate the impact of the completeness or lack thereof of canonicalizations on the downstream base model's ability to distinguish among highly regular and $1$-WL indistinguishable graphs in the BREC dataset \cite{brec}. For that end, using BREC's Basic, Regular, Extension, and CFI \cite{cai_optimal_1992} graph categories, we measure graph-pair distinguishability of each canonicalization using a base DeepSets model with matched parameter budgets, see \cref{tab:brec-comparison}.
$\prism$ scores $212/360$ vs.\ $199/360$ for MAP and OAP, with gains on Regular ($+8$) and Extension ($+5$) categories.  We note that distinguishing CFI graphs requires high-order $k$-WL (and equivalently expressive GNNs), and that CFI graphs have eigenvalue multiplicities up to $\Theta(n)$ (\cref{thm:cfi-spectrum}), and therefore we do not expect $\prism$, MAP, or OAP to distinguish them well.

\begin{table}[ht]
\centering
\caption{\it{BREC expressivity benchmark (360 pairs). All methods use matched parameter budget ($\sim$12K) with a spectral positional encoding processed by a DeepSets embedding.}}
\label{tab:brec-comparison}
\begin{tabular}{lccccc}
\toprule
\textbf{Method} & \textbf{Basic} & \textbf{Regular} & \textbf{Ext.} & \textbf{CFI} & \textbf{Total} \\
 & /60 & /100 & /100 & /100 & /360 \\
\midrule
$\prism$ \textbf{(ours)} & \textbf{60} & \textbf{50} & \textbf{99} & 3 & \textbf{212} \\
MAP   & \textbf{60} & 42 & 94 & 3 & 199 \\
OAP   & \textbf{60} & 42 & 94 & 3 & 199 \\
\bottomrule
\end{tabular}
\end{table}

\FloatBarrier
\subsection{Molecular Property Prediction}

Following the setup by \citet{ma2024a}, we compare molecular property prediction accuracy of $\prism$ against MAP, OAP and SignNet on ZINC 12K \cite{irwin2012zinc, dwivedi2023benchmarking} and OGB \textsc{MolTox21}, \textsc{MolPCBA} \cite{hu2020open} across GatedGCN \cite{bresson2017residual} and PNA \cite{corso2020principal} backbones at matched parameter budgets (\cref{tab:molecular-results}). We report mean $\pm$ std over four seeds. The per-cell hyperparameter configuration is in \cref{tab:hp-zinc,tab:hp-ogb-gnn}. We find that
$\prism$ is best or tied-best on five of six categories (\cref{tab:molecular-results}), supporting the hypothesis that a complete, deterministic canonicalization translates to downstream gains.
\begin{table}[ht]
\centering
\setlength{\abovecaptionskip}{3pt}
\setlength{\belowcaptionskip}{1pt}
\caption{\it Molecular property prediction of competing canonicalizations with  GatedGCN/PNA backbones.}
\label{tab:molecular-results}
\setlength{\tabcolsep}{1pt}
\renewcommand{\arraystretch}{0.85}
\scriptsize
\begin{tabular}{lcccccc}
\toprule
 & \multicolumn{2}{c}{\textbf{ZINC MAE}\,$\downarrow$} & \multicolumn{2}{c}{\textbf{\textsc{MolTox21} AUC}\,$\uparrow$} & \multicolumn{2}{c}{\textbf{\textsc{MolPCBA} AP}\,$\uparrow$} \\
\cmidrule(lr){2-3}\cmidrule(lr){4-5}\cmidrule(lr){6-7}
\textbf{Method} & \textbf{GatedGCN} & \textbf{PNA} & \textbf{GatedGCN} & \textbf{PNA} & \textbf{GatedGCN} & \textbf{PNA} \\
\midrule
$\prism$ \textbf{(ours)} & $\mathbf{0.110 \pm 0.002}$ & $\mathbf{0.103 \pm 0.003}$ & $\mathbf{0.783 \pm 0.006}$ & $0.762 \pm 0.004$ & $\mathbf{0.272 \pm 0.003}$ & $\mathbf{0.278 \pm 0.001}$ \\
SignNet               & $0.121 \pm 0.002$ & $0.105 \pm 0.007$ & $0.769 \pm 0.007$ & $0.744 \pm 0.002$ & $0.260 \pm 0.002$           & OOM \\
MAP                   & $0.123 \pm 0.003$ & $0.104 \pm 0.003$ & $0.777 \pm 0.004$ & $\mathbf{0.763 \pm 0.004}$ & $0.268 \pm 0.002$           & $0.274 \pm 0.003$ \\
OAP                   & $0.126 \pm 0.004$ & $0.106 \pm 0.002$ & $\mathbf{0.783 \pm 0.006}$ & $0.761 \pm 0.008$ & $0.270 \pm 0.002$           & $0.272 \pm 0.002$ \\
\bottomrule
\end{tabular}
\end{table}
\subsection{Transformer Backbone on OGB and Alchemy}

\label{subsec:transformer-ogb} We next evaluate the efficacy of our complete $\prism$-canonicalized Laplacian Positional Encoding (LapPE) on tasks with a Transformer backbone. Moreover, we leverage an artifact from our canonicalization, a node ordering that is used as positional encoding supplementary to the canonicalized LapPE. Concretely, we pair $\prism$ with the graph Transformer of \citet{kim2021transformers} on OGB \cite{hu2020open} (\textsc{MolHIV}, \textsc{MolClinTox}, \textsc{MolBBBP}) and \textsc{Alchemy} \cite{chen2019alchemy}, a 12-task quantum-chemistry regression benchmark. Per-dataset hyperparameter configuration is in \cref{app:exp-transformer}. As shown in \cref{tab:transformer-results}, $\prism$ is best on all four datasets.
\begin{table}[ht]
\centering
\setlength{\abovecaptionskip}{3pt}
\setlength{\belowcaptionskip}{1pt}
\caption{\it Transformer backbone results on OGB \cite{hu2020open} and \textsc{Alchemy} \cite{chen2019alchemy} molecular datasets.}
\label{tab:transformer-results}
\setlength{\tabcolsep}{3pt}
\renewcommand{\arraystretch}{0.85}
\footnotesize
\begin{tabular}{lcccc}
\toprule
 & \multicolumn{3}{c}{\textbf{ROC-AUC}\,$\uparrow$} & \textbf{MAE}\,$\downarrow$ \\
\cmidrule(lr){2-4}\cmidrule(lr){5-5}
\textbf{Method} & \textbf{\textsc{MolHIV}} & \textbf{\textsc{MolClinTox}} & \textbf{\textsc{MolBBBP}} & \textbf{\textsc{Alchemy}} \\
\midrule
$\prism$ \textbf{(ours)} & $\mathbf{0.7430 \pm 0.0233}$ & $\mathbf{0.8576 \pm 0.0038}$ & $\mathbf{0.7088 \pm 0.0293}$ & $\mathbf{0.11956 \pm 0.00130}$ \\
LPE                   & $0.6891 \pm 0.0068$ & $0.7905 \pm 0.0444$ & $0.6798 \pm 0.0179$ & $0.12655 \pm 0.00056$ \\
MAP                   & $0.7221 \pm 0.0190$ & $0.7900 \pm 0.0440$ & $0.7029 \pm 0.0019$ & $0.13058 \pm 0.00038$ \\
OAP                   & $0.6878 \pm 0.0197$ & $0.7820 \pm 0.0530$ & $0.6825 \pm 0.0323$ & $0.13216 \pm 0.00488$ \\
\bottomrule
\end{tabular}
\end{table}

\section{Conclusion, Limitations and Future Work}\label{sec:conclusion}

In this paper, we uncovered that $k$-WL tests are incomplete on graphs with a simple spectrum and presented a novel canonicalization to achieve completeness. Empirically, we show that our complete canonicalization outperforms competing canonicalizations. One limitation is that we do not address completeness for higher-order eigenmultiplicities, and another is that the canonicalization we suggest is discontinuous. This latter limitation is probably unavoidable as continuous canonicalizations are often mathematically impossible (see \cite{dymequivariant}).  An additional limitation is that our incompleteness result is for the class of multigraphs. Attaining similar results for the smaller class of combinatorial graphs is an interesting avenue for future work.

\section*{Acknowledgements}

N.D. and S.H. are supported by the Israeli Science Foundation grant no. 272/23. S.H. is supported by the Gloria and Ken Levy Foundation Fellowship.

T.S. is supported by the European Union (CountHom, 101077083). Views and opinions expressed are however those of the author(s) only and do not necessarily reflect those of the European Union or the European Research Council Executive Agency. Neither the European Union nor the granting
authority can be held responsible for them.

\bibliography{literature}

@article{feng2023extending,
  title={Extending the design space of graph neural networks by rethinking folklore Weisfeiler-Lehman},
  author={Feng, Jiarui and Kong, Lecheng and Liu, Hao and Tao, Dacheng and Li, Fuhai and Zhang, Muhan and Chen, Yixin},
  journal={Advances in Neural Information Processing Systems},
  volume={36},
  pages={9029--9064},
  year={2023}
}

@inproceedings{dymequivariant,
  title={Equivariant Frames and the Impossibility of Continuous Canonicalization},
  author={Dym, Nadav and Lawrence, Hannah and Siegel, Jonathan W},
  booktitle={Forty-first International Conference on Machine Learning}
}

@article{morris2020weisfeiler,
  title={Weisfeiler and leman go sparse: Towards scalable higher-order graph embeddings},
  author={Morris, Christopher and Rattan, Gaurav and Mutzel, Petra},
  journal={Advances in Neural Information Processing Systems},
  volume={33},
  pages={21824--21840},
  year={2020}
}

@article{morris2023weisfeiler,
  title={Weisfeiler and leman go machine learning: The story so far},
  author={Morris, Christopher and Lipman, Yaron and Maron, Haggai and Rieck, Bastian and Kriege, Nils M and Grohe, Martin and Fey, Matthias and Borgwardt, Karsten},
  journal={Journal of Machine Learning Research},
  volume={24},
  number={333},
  pages={1--59},
  year={2023}
}

@article{duvenaud2015convolutional,
  title={Convolutional networks on graphs for learning molecular fingerprints},
  author={Duvenaud, David K and Maclaurin, Dougal and Iparraguirre, Jorge and Bombarell, Rafael and Hirzel, Timothy and Aspuru-Guzik, Al{\'a}n and Adams, Ryan P},
  journal={Advances in neural information processing systems},
  volume={28},
  year={2015}
}

@article{corso2024graph,
  title={Graph neural networks},
  author={Corso, Gabriele and Stark, Hannes and Jegelka, Stefanie and Jaakkola, Tommi and Barzilay, Regina},
  journal={Nature Reviews Methods Primers},
  volume={4},
  number={1},
  pages={17},
  year={2024},
  publisher={Nature Publishing Group UK London}
}

@inproceedings{
kipf2017semisupervised,
title={Semi-Supervised Classification with Graph Convolutional Networks},
author={Thomas N. Kipf and Max Welling},
booktitle={International Conference on Learning Representations},
year={2017},
url={https://openreview.net/forum?id=SJU4ayYgl}
}

@inproceedings{velivckovic2018graph,
  title={Graph Attention Networks},
  author={Veli{\v{c}}kovi{\'c}, Petar and Cucurull, Guillem and Casanova, Arantxa and Romero, Adriana and Li{\`o}, Pietro and Bengio, Yoshua},
  booktitle={International Conference on Learning Representations},
  year={2018}
}

@article{scarselli2008graph,
  title={The graph neural network model},
  author={Scarselli, Franco and Gori, Marco and Tsoi, Ah Chung and Hagenbuchner, Markus and Monfardini, Gabriele},
  journal={IEEE transactions on neural networks},
  volume={20},
  number={1},
  pages={61--80},
  year={2008},
  publisher={IEEE}
}

@inproceedings{gilmer2017neural,
  title={Neural message passing for quantum chemistry},
  author={Gilmer, Justin and Schoenholz, Samuel S and Riley, Patrick F and Vinyals, Oriol and Dahl, George E},
  booktitle={International conference on machine learning},
  pages={1263--1272},
  year={2017},
  organization={Pmlr}
}

@article{frasca2022understanding,
  title={Understanding and extending subgraph gnns by rethinking their symmetries},
  author={Frasca, Fabrizio and Bevilacqua, Beatrice and Bronstein, Michael and Maron, Haggai},
  journal={Advances in Neural Information Processing Systems},
  volume={35},
  pages={31376--31390},
  year={2022}
}

@inproceedings{
lim2023sign,
title={Sign and Basis Invariant Networks for Spectral Graph Representation Learning},
author={Derek Lim and Joshua David Robinson and Lingxiao Zhao and Tess Smidt and Suvrit Sra and Haggai Maron and Stefanie Jegelka},
booktitle={The Eleventh International Conference on Learning Representations },
year={2023},
url={https://openreview.net/forum?id=Q-UHqMorzil}
}

@article{irwin2012zinc,
  title={ZINC: a free tool to discover chemistry for biology},
  author={Irwin, John J and Sterling, Teague and Mysinger, Michael M and Bolstad, Erin S and Coleman, Ryan G},
  journal={Journal of chemical information and modeling},
  volume={52},
  number={7},
  pages={1757--1768},
  year={2012},
  publisher={ACS Publications}
}

@article{Weyl1912DasAV,
  title={Das asymptotische Verteilungsgesetz der Eigenwerte linearer partieller Differentialgleichungen (mit einer Anwendung auf die Theorie der Hohlraumstrahlung)},
  author={Hermann Von Weyl},
  journal={Mathematische Annalen},
  year={1912},
  volume={71},
  pages={441-479},
  url={https://api.semanticscholar.org/CorpusID:120278241}
}

@article{chen2019alchemy,
  title={Alchemy: A quantum chemistry dataset for benchmarking ai models},
  author={Chen, Guangyong and Chen, Pengfei and Hsieh, Chang-Yu and Lee, Chee-Kong and Liao, Benben and Liao, Renjie and Liu, Weiwen and Qiu, Jiezhong and Sun, Qiming and Tang, Jie and others},
  journal={arXiv preprint arXiv:1906.09427},
  year={2019}
}

@article{corso2020principal,
  title={Principal neighbourhood aggregation for graph nets},
  author={Corso, Gabriele and Cavalleri, Luca and Beaini, Dominique and Li{\`o}, Pietro and Veli{\v{c}}kovi{\'c}, Petar},
  journal={Advances in neural information processing systems},
  volume={33},
  pages={13260--13271},
  year={2020}
}

@article{bresson2017residual,
  title={Residual gated graph convnets},
  author={Bresson, Xavier and Laurent, Thomas},
  journal={arXiv preprint arXiv:1711.07553},
  year={2017}
}

@article{hu2020open,
  title={Open graph benchmark: Datasets for machine learning on graphs},
  author={Hu, Weihua and Fey, Matthias and Zitnik, Marinka and Dong, Yuxiao and Ren, Hongyu and Liu, Bowen and Catasta, Michele and Leskovec, Jure},
  journal={Advances in neural information processing systems},
  volume={33},
  pages={22118--22133},
  year={2020}
}

@Unpublished{LeightonMiller79,
  author = 	 {F. Thomson Leighton and Gary l. Miller},
  title = 	 {Certificates for Graphs with Distinct Eigen Values},
  note = 	 {Orginal Manuscript},
  OPTkey = 	 {},
  OPTmonth = 	 {},
  year = 	 {1979},
  bib2html_rescat = {Spectral Graph Theory,Graph Isomorphism},
  OPTannote = 	 {Parallel Algorithms,Graph Algorithms}
}

@article{neuen_parameterized_2026,
	title = {Parameterized complexity of graph isomorphism testing},
	volume = {60},
	issn = {15740137},
	url = {https://linkinghub.elsevier.com/retrieve/pii/S1574013726000274},
	doi = {10.1016/j.cosrev.2026.100918},
	language = {en},
	urldate = {2026-02-22},
	journal = {Computer Science Review},
	author = {Neuen, Daniel},
	month = may,
	year = {2026},
	pages = {100918},
}

@article{dwivedi2023benchmarking,
  title={Benchmarking graph neural networks},
  author={Dwivedi, Vijay Prakash and Joshi, Chaitanya K and Luu, Anh Tuan and Laurent, Thomas and Bengio, Yoshua and Bresson, Xavier},
  journal={Journal of Machine Learning Research},
  volume={24},
  number={43},
  pages={1--48},
  year={2023}
}

@article{kim2021transformers,
  title={Transformers generalize deepsets and can be extended to graphs \& hypergraphs},
  author={Kim, Jinwoo and Oh, Saeyoon and Hong, Seunghoon},
  journal={Advances in Neural Information Processing Systems},
  volume={34},
  pages={28016--28028},
  year={2021}
}

@inproceedings{
puny2022frame,
title={Frame Averaging for Invariant and Equivariant Network Design},
author={Omri Puny and Matan Atzmon and Edward J. Smith and Ishan Misra and Aditya Grover and Heli Ben-Hamu and Yaron Lipman},
booktitle={International Conference on Learning Representations},
year={2022},
url={https://openreview.net/forum?id=zIUyj55nXR}
}

@InProceedings{pmlr-v202-kaba23a,
  title = 	 {Equivariance with Learned Canonicalization Functions},
  author =       {Kaba, S\'{e}kou-Oumar and Mondal, Arnab Kumar and Zhang, Yan and Bengio, Yoshua and Ravanbakhsh, Siamak},
  booktitle = 	 {Proceedings of the 40th International Conference on Machine Learning},
  pages = 	 {15546--15566},
  year = 	 {2023},
  editor = 	 {Krause, Andreas and Brunskill, Emma and Cho, Kyunghyun and Engelhardt, Barbara and Sabato, Sivan and Scarlett, Jonathan},
  volume = 	 {202},
  series = 	 {Proceedings of Machine Learning Research},
  month = 	 {23--29 Jul},
  publisher =    {PMLR},
  url = 	 {https://proceedings.mlr.press/v202/kaba23a.html}
}

@article{dwivedi2021generalization,
  title={A Generalization of Transformer Networks to Graphs},
  author={Dwivedi, Vijay Prakash and Bresson, Xavier},
  journal={AAAI Workshop on Deep Learning on Graphs: Methods and Applications},
  year={2021}
}

@inproceedings{
gai2025homomorphism,
title={Homomorphism Expressivity of Spectral Invariant Graph Neural Networks},
author={Jingchu Gai and Yiheng Du and Bohang Zhang and Haggai Maron and Liwei Wang},
booktitle={The Thirteenth International Conference on Learning Representations},
year={2025},
url={https://openreview.net/forum?id=rdv6yeMFpn}
}

@inproceedings{deepsets,
 author = {Zaheer, Manzil and Kottur, Satwik and Ravanbakhsh, Siamak and Poczos, Barnabas and Salakhutdinov, Russ R and Smola, Alexander J},
 booktitle = {Advances in Neural Information Processing Systems},
 editor = {I. Guyon and U. Von Luxburg and S. Bengio and H. Wallach and R. Fergus and S. Vishwanathan and R. Garnett},
 pages = {},
 publisher = {Curran Associates, Inc.},
 title = {Deep Sets},
 url = {https://proceedings.neurips.cc/paper_files/paper/2017/file/f22e4747da1aa27e363d86d40ff442fe-Paper.pdf},
 volume = {30},
 year = {2017}
}

@InProceedings{brec,
  title = 	 {An Empirical Study of Realized {GNN} Expressiveness},
  author =       {Wang, Yanbo and Zhang, Muhan},
  booktitle = 	 {Proceedings of the 41st International Conference on Machine Learning},
  pages = 	 {52134--52155},
  year = 	 {2024},
  editor = 	 {Salakhutdinov, Ruslan and Kolter, Zico and Heller, Katherine and Weller, Adrian and Oliver, Nuria and Scarlett, Jonathan and Berkenkamp, Felix},
  volume = 	 {235},
  series = 	 {Proceedings of Machine Learning Research},
  month = 	 {21--27 Jul},
  publisher =    {PMLR},
  url = 	 {https://proceedings.mlr.press/v235/wang24cl.html}
}

@book{Grohe_2017, place={Cambridge}, series={Lecture Notes in Logic}, title={Descriptive Complexity, Canonisation, and Definable Graph Structure Theory}, publisher={Cambridge University Press}, author={Grohe, Martin}, year={2017}, collection={Lecture Notes in Logic}}

@article{grohe2015pebble,
  title={Pebble games and linear equations},
  author={Grohe, Martin and Otto, Martin},
  journal={The Journal of Symbolic Logic},
  volume={80},
  number={3},
  pages={797--844},
  year={2015},
  publisher={Cambridge University Press}
}

@InProceedings{pmlr-v235-muller24c,
  title = 	 {Aligning Transformers with Weisfeiler-Leman},
  author =       {M\"{u}ller, Luis and Morris, Christopher},
  booktitle = 	 {Proceedings of the 41st International Conference on Machine Learning},
  pages = 	 {36654--36704},
  year = 	 {2024},
  editor = 	 {Salakhutdinov, Ruslan and Kolter, Zico and Heller, Katherine and Weller, Adrian and Oliver, Nuria and Scarlett, Jonathan and Berkenkamp, Felix},
  volume = 	 {235},
  series = 	 {Proceedings of Machine Learning Research},
  month = 	 {21--27 Jul},
  publisher =    {PMLR},
  url = 	 {https://proceedings.mlr.press/v235/muller24c.html}
}

@article{kreuzer2021rethinking,
  title={Rethinking graph transformers with spectral attention},
  author={Kreuzer, Devin and Beaini, Dominique and Hamilton, Will and L{\'e}tourneau, Vincent and Tossou, Prudencio},
  journal={Advances in neural information processing systems},
  volume={34},
  pages={21618--21629},
  year={2021}
}

@inproceedings{
ma2024a,
title={A Canonicalization Perspective on Invariant and Equivariant Learning},
author={George Ma and Yifei Wang and Derek Lim and Stefanie Jegelka and Yisen Wang},
booktitle={The Thirty-eighth Annual Conference on Neural Information Processing Systems},
year={2024},
url={https://openreview.net/forum?id=jjcY92FX4R}
}

@inproceedings{
ma2023laplacian,
title={Laplacian Canonization: A Minimalist Approach to Sign and Basis Invariant Spectral Embedding},
author={George Ma and Yifei Wang and Yisen Wang},
booktitle={Thirty-seventh Conference on Neural Information Processing Systems},
year={2023},
url={https://openreview.net/forum?id=1mAYtdoYw6}
}

@inproceedings{hordanspectral,
  title={Spectral Graph Neural Networks are Incomplete on Graphs with a Simple Spectrum},
  author={Hordan, Snir and Bechler-Speicher, Maya and Lifshitz, Gur and Dym, Nadav},
  booktitle={The Thirty-ninth Annual Conference on Neural Information Processing Systems},
  year={2025}
}

@inproceedings{babai,
author = {Babai, L\'{a}szl\'{o} and Grigoryev, D. Yu. and Mount, David M.},
title = {Isomorphism of graphs with bounded eigenvalue multiplicity},
year = {1982},
isbn = {0897910702},
publisher = {Association for Computing Machinery},
address = {New York, NY, USA},
url = {https://doi.org/10.1145/800070.802206},
doi = {10.1145/800070.802206},
booktitle = {Proceedings of the Fourteenth Annual ACM Symposium on Theory of Computing},
pages = {310–324},
numpages = {15},
location = {San Francisco, California, USA},
series = {STOC '82}
}

@article{tao2017random,
  title={Random matrices have simple spectrum},
  author={Tao, Terence and Vu, Van},
  journal={Combinatorica},
  volume={37},
  number={3},
  pages={539--553},
  year={2017},
  publisher={Springer},
  doi={10.1007/s00493-016-3363-4}
}

@phdthesis{seppelt_homomorphism_2024,
	address = {Aachen},
	type = {Dissertation},
	title = {Homomorphism {Indistinguishability}},
	url = {https://publications.rwth-aachen.de/record/998785},
	doi = {10.18154/RWTH-2024-11629},
	
	school = {RWTH Aachen University},
	author = {Seppelt, Tim},
	year = {2024},
}

@article{spielman,
	title = {Spectral and Algebraic Graph Theory},
	url = {http://cs-www.cs.yale.edu/homes/spielman/sagt/sagt.pdf},
	author = { Spielman, Daniel A.},
	year = {2025},
}

@article{kolesnik_lower_2014,
	title = {Lower {Bounds} for the {Isoperimetric} {Numbers} of {Random} {Regular} {Graphs}},
	volume = {28},
	issn = {0895-4801, 1095-7146},
	doi = {10.1137/120891265},
	number = {1},
	urldate = {2026-03-31},
	journal = {SIAM Journal on Discrete Mathematics},
	author = {Kolesnik, Brett and Wormald, Nick},
	month = jan,
	year = {2014},
	pages = {553--575},
}

@article{dvorak_strongly_2016,
	title = {Strongly {Sublinear} {Separators} and {Polynomial} {Expansion}},
	volume = {30},
	issn = {0895-4801, 1095-7146},
	doi = {10.1137/15M1017569},
	number = {2},
	urldate = {2026-03-31},
	journal = {SIAM Journal on Discrete Mathematics},
	author = {Dvořák, Zdeněk and Norin, Sergey},
	month = jan,
	year = {2016},
	pages = {1095--1101},
}

@article{zhou2023distance,
  title={Distance-restricted folklore weisfeiler-leman GNNs with provable cycle counting power},
  author={Zhou, Junru and Feng, Jiarui and Wang, Xiyuan and Zhang, Muhan},
  journal={Advances in Neural Information Processing Systems},
  volume={36},
  pages={14293--14337},
  year={2023}
}

@article{xu2018powerful,
  title={How powerful are graph neural networks?},
  author={Xu, Keyulu and Hu, Weihua and Leskovec, Jure and Jegelka, Stefanie},
  journal={arXiv preprint arXiv:1810.00826},
  year={2018}
}

@article{maron2019provably,
  title={Provably powerful graph networks},
  author={Maron, Haggai and Ben-Hamu, Heli and Serviansky, Hadar and Lipman, Yaron},
  journal={Advances in neural information processing systems},
  volume={32},
  year={2019}
}

@article{vaswani2017attention,
  title={Attention is all you need},
  author={Vaswani, Ashish and Shazeer, Noam and Parmar, Niki and Uszkoreit, Jakob and Jones, Llion and Gomez, Aidan N and Kaiser, {\L}ukasz and Polosukhin, Illia},
  journal={Advances in neural information processing systems},
  volume={30},
  year={2017}
}

@article{wagstaff2022universal,
  title={Universal approximation of functions on sets},
  author={Wagstaff, Edward and Fuchs, Fabian B and Engelcke, Martin and Osborne, Michael A and Posner, Ingmar},
  journal={Journal of Machine Learning Research},
  volume={23},
  number={151},
  pages={1--56},
  year={2022}
}

@article{amir2023neural,
  title={Neural injective functions for multisets, measures and graphs via a finite witness theorem},
  author={Amir, Tal and Gortler, Steven and Avni, Ilai and Ravina, Ravina and Dym, Nadav},
  journal={Advances in Neural Information Processing Systems},
  volume={36},
  pages={42516--42551},
  year={2023}
}

@article{rampavsek2022recipe,
  title={Recipe for a general, powerful, scalable graph transformer},
  author={Ramp{\'a}{\v{s}}ek, Ladislav and Galkin, Michael and Dwivedi, Vijay Prakash and Luu, Anh Tuan and Wolf, Guy and Beaini, Dominique},
  journal={Advances in Neural Information Processing Systems},
  volume={35},
  pages={14501--14515},
  year={2022}
}

@article{su2024roformer,
  title={Roformer: Enhanced transformer with rotary position embedding},
  author={Su, Jianlin and Ahmed, Murtadha and Lu, Yu and Pan, Shengfeng and Bo, Wen and Liu, Yunfeng},
  journal={Neurocomputing},
  volume={568},
  pages={127063},
  year={2024},
  publisher={Elsevier}
}

@article{Morris_Ritzert_Fey_Hamilton_Lenssen_Rattan_Grohe_2019, title={Weisfeiler and Leman Go Neural: Higher-Order Graph Neural Networks}, volume={33}, url={https://ojs.aaai.org/index.php/AAAI/article/view/4384}, DOI={10.1609/aaai.v33i01.33014602}, abstractNote={&lt;p&gt;In recent years, graph neural networks (GNNs) have emerged as a powerful neural architecture to learn vector representations of nodes and graphs in a supervised, end-to-end fashion. Up to now, GNNs have only been evaluated empirically—showing promising results. The following work investigates GNNs from a theoretical point of view and relates them to the 1-dimensional Weisfeiler-Leman graph isomorphism heuristic (1-WL). We show that GNNs have the same expressiveness as the 1-WL in terms of distinguishing non-isomorphic (sub-)graphs. Hence, both algorithms also have the same shortcomings. Based on this, we propose a generalization of GNNs, so-called &lt;em&gt;k&lt;/em&gt;-dimensional GNNs (&lt;em&gt;k&lt;/em&gt;-GNNs), which can take higher-order graph structures at multiple scales into account. These higher-order structures play an essential role in the characterization of social networks and molecule graphs. Our experimental evaluation confirms our theoretical findings as well as confirms that higher-order information is useful in the task of graph classification and regression.&lt;/p&gt;}, number={01}, journal={Proceedings of the AAAI Conference on Artificial Intelligence}, author={Morris, Christopher and Ritzert, Martin and Fey, Matthias and Hamilton, William L. and Lenssen, Jan Eric and Rattan, Gaurav and Grohe, Martin}, year={2019}, month={Jul.}, pages={4602-4609} }

@article{weisfeiler1968reduction,
  author  = {Weisfeiler, Boris and Lehman, A. A.},
  title   = {A Reduction of a Graph to a Canonical Form and an Algebra Arising during This Reduction},
  journal = {Nauchno-Technicheskaya Informatsia},
  volume  = {9},
  year    = {1968}
}

@inproceedings{zhangexpressive,
  title={On the Expressive Power of Spectral Invariant Graph Neural Networks},
  author={Zhang, Bohang and Zhao, Lingxiao and Maron, Haggai},
  booktitle={Forty-first International Conference on Machine Learning}
}

@inproceedings{grohe_logic_2021,
	address = {Rome, Italy},
	title = {The {Logic} of {Graph} {Neural} {Networks}},
	url = {https://doi.org/10.1109/LICS52264.2021.9470677},
	doi = {10.1109/LICS52264.2021.9470677},
	booktitle = {36th {Annual} {ACM}/{IEEE} {Symposium} on {Logic} in {Computer} {Science}, {LICS}},
	publisher = {IEEE},
	author = {Grohe, Martin},
	year = {2021},
	pages = {1--17},
}

@article{cai_optimal_1992,
	title = {An optimal lower bound on the number of variables for graph identification},
	volume = {12},
	issn = {1439-6912},
	url = {https://doi.org/10.1007/BF01305232},
	doi = {10.1007/BF01305232},
	number = {4},
	journal = {Combinatorica},
	author = {Cai, Jin-Yi and Fürer, Martin and Immerman, Neil},
	year = {1992},
	pages = {389--410},
}

@inproceedings{neuen_homomorphism-distinguishing_2024,
	address = {Dagstuhl, Germany},
	series = {Leibniz {International} {Proceedings} in {Informatics} ({LIPIcs})},
	title = {Homomorphism-{Distinguishing} {Closedness} for {Graphs} of {Bounded} {Tree}-{Width}},
	volume = {289},
	isbn = {978-3-95977-311-9},
	url = {https://drops.dagstuhl.de/entities/document/10.4230/LIPIcs.STACS.2024.53},
	doi = {10.4230/LIPIcs.STACS.2024.53},
	booktitle = {41st {International} {Symposium} on {Theoretical} {Aspects} of {Computer} {Science} ({STACS} 2024)},
	publisher = {Schloss Dagstuhl – Leibniz-Zentrum für Informatik},
	author = {Neuen, Daniel},
	editor = {Beyersdorff, Olaf and Kanté, Mamadou Moustapha and Kupferman, Orna and Lokshtanov, Daniel},
	year = {2024},
	note = {ISSN: 1868-8969},
	pages = {53:1--53:12},
}

@misc{roberson_oddomorphisms_2022,
	title = {Oddomorphisms and homomorphism indistinguishability over graphs of bounded degree},
	url = {http://arxiv.org/abs/2206.10321},
	urldate = {2022-06-22},
	publisher = {arXiv},
	author = {Roberson, David E.},
	month = jun,
	year = {2022},
}

\appendix

\section{Incompleteness of Weisfeiler--Leman on Simple Spectrum Graphs}\label{app:incomp-full}

\subsection{Weisfeiler--Leman Definitions}\label{app:wl-defs}

For completeness we record here the general $k$-dimensional Weisfeiler--Leman test and its associated indistinguishability relation, both used throughout this appendix and referenced from the main text.

\begin{definition}[$k$-Weisfeiler--Leman Algorithm]\label{def:k-wl}
Let $k \geq 1$ and let $g \in \mathcal{G}_n$ with adjacency matrix
$A \in \mathbb{R}^{n \times n}$. The \emph{$k$-WL algorithm} maintains a coloring
$C^{(t)} \colon [n]^k \to \mathcal{C}$ of $k$-tuples of vertices over a discrete color
set $\mathcal{C}$, and proceeds as follows.
\begin{enumerate}[label=(\roman*)]
    \item \textit{Initialization.} Each $k$-tuple $\bar{v} = (v_1, \ldots, v_k)
    \in [n]^k$ is assigned an initial color
    \[
        C^{(0)}(\bar{v}) = \mathrm{HASH}\!\left(\left(A_{v_i v_j}
        \right)_{1 \leq i,j \leq k}\right),
    \]
    encoding the edge weights among $\{v_1, \ldots, v_k\}$, where $\mathrm{HASH}$
    is a fixed injective map into $\mathcal{C}$.

  \item \textit{Refinement.} For $j \in [k]$, define the \emph{$j$-th neighborhood}
    of $\bar{v} = (v_1, \ldots, v_k) \in [n]^k$ as
    \[
        N_j(\bar{v}) = \left\{(v_1, \ldots, v_{j-1}, w, v_{j+1}, \ldots, v_k)
        \;\middle|\; w \in [n]\right\},
    \]
    the set of $n$ $k$-tuples obtained by substituting the $j$-th entry of $\bar{v}$
    with every vertex $w \in [n]$, while keeping all other entries fixed. At each
    step $t \geq 0$, the color of $\bar{v}$ is updated via
    \[
        C^{(t+1)}(\bar{v}) = \mathrm{HASH}\!\left(C^{(t)}(\bar{v}),\,
        \left(\left\{\!\!\left\{C^{(t)}(\bar{u}) \;\middle|\; \bar{u} \in
        N_j(\bar{v})\right\}\!\!\right\}\right)_{j=1}^{k}\right),
    \]
    where, for each fixed $j \in [k]$, the inner term $\{\!\{C^{(t)}(\bar{u}) \mid
    \bar{u} \in N_j(\bar{v})\}\!\}$ is the multiset of colors of all $k$-tuples
    reachable from $\bar{v}$ by substituting position $j$. These $k$ multisets are
    collected into an ordered $k$-tuple indexed by $j$, preserving the distinction
    between neighborhoods at different positions, and hashed together with the
    current color $C^{(t)}(\bar{v})$.
    \item \textit{Termination.} The algorithm terminates at the smallest $T \in
    \mathbb{N}$ such that the partition of $[n]^k$ induced by $C^{(T+1)}$ coincides
    with that induced by $C^{(T)}$. The stable coloring $C^{(T)}$ is the
    \emph{$k$-WL coloring} of $g$.

    \item \textit{Global label.} A graph-level label is computed as
    \[
        C_{\mathrm{global}}(g) = \mathrm{HASH}\!\left(\left\{\!\!\left\{C^{(T)}
        (\bar{v}) \mid \bar{v} \in [n]^k\right\}\!\!\right\}\right).
    \]
\end{enumerate}
\begin{remark}
For $k=1$, the algorithm reduces to the \emph{color refinement} test,
where each node $v \in [n]$ is initialized with a uniform color $C^{(0)}(v) = c$
for a constant $c \in \mathcal{C}$, and refined via
\[
    C^{(t+1)}(v) = \mathrm{HASH}\!\left(C^{(t)}(v),\, \left\{\!\!\left\{
    \bigl(C^{(t)}(w),\, A_{vw}\bigr) \mid w \in [n] \right\}\!\!\right\}\right).
\]
\end{remark}
\end{definition}

\begin{definition}[$k$-WL Indistinguishability]\label{def:k-wl-indist}
Two graphs $g, g' \in \mathcal{G}_n$ are \emph{$k$-WL indistinguishable}, written
$g \sim_{k} g'$, if $C_{\mathrm{global}}(g) = C_{\mathrm{global}}(g')$. The
relation $\sim_k$ is an equivalence relation on $\mathcal{G}_n$, and satisfies the
strict refinement hierarchy: for all $k \geq 1$,
\[
    g \sim_{k+1} g' \implies g \sim_{k} g',
\]
but the converse fails in general \cite{cai_optimal_1992, grohe2015pebble, Grohe_2017}, meaning that $(k+1)$-WL
distinguishes strictly more pairs of graphs than $k$-WL.
\end{definition}

\subsection{Proof of the Main Incompleteness Theorem}
\label{app:main-incomp}
In this subsection, we give a proof of \cref{thm:main-incomp}.
The bulk of the argument is to translate a pair of graphs obtained via the CFI construction considered in \cite{roberson_oddomorphisms_2022} to a pair of orthonormal matrices.
The CFI construction builds, given a highly connected graph $G$, two graphs $G_0$ and $G_1$ that cannot be distinguished by highly dimensional Weisfeiler--Leman.
More precisely, as shown in \cite{neuen_homomorphism-distinguishing_2024},
if $G$ has \emph{treewidth} at least~$k+1$,
then $G_0$ and $G_1$ are not distinguished by the $k$-dimensional Weisfeiler--Leman algorithm.
We will transfer this result to our matrix encoding of CFI graphs.
Finally, we obtain the desired matrices by instantiating the CFI construction for $3$-regular graphs $G$ of high treewidth, so-called expanders, as constructed and analyzed in \cite{dvorak_strongly_2016,kolesnik_lower_2014}.
Throughout, one should think of $G$ as being such a graph, although we will be explicit about the particular properties of $G$ needed for each lemma.
More precisely, the proof of \cref{thm:main-incomp} proceeds as follows.

\begin{enumerate}
    \item For some suitable base graph $G$ and $U \in \{0,1\}$,
    we define a matrix $\tilde{X}_{G, U}$ based on the CFI graph $G_U$. 
    This matrix will serve as the matrix of eigenvectors for the multigraphs, which we construct subsequently.
    \item In order for $\tilde{X}_{G, U}$ to encode an eigenbasis of a simple-spectrum graph, we show in \cref{lem:orthogonal} that the columns of $\tilde{X}_{G, U}$ are orthogonal.
    \item Now we choose a diagonal matrix with distinct diagonal entries $D$ and consider the matrices $A_{G, U} \coloneqq \tilde{X}_{G, U} D (\tilde{X}_{G, U})^\top$.
    Our construction will be such that $A_{G, U}$ can be ensured to have only non-negative integral entries.
    Hence, $A_{G, U} $ is a multigraph with simple spectrum.
    \item To argue that $A_{G, 0} $ and $A_{G, 1} $ are non-isomorphic, we show that $A_{G, 0} \cong A_{G, 1} $ would imply that the CFI graphs $G_0$ and $G_1$ are isomorphic, which is known not to hold \cite{roberson_oddomorphisms_2022}.
    \item
    Lastly, we show that the Weisfeiler--Leman algorithm fails to distinguish $A_{G, 0} $ and $A_{G, 1} $. 
    To that end, we show that
    running Weisfeiler--Leman on $A_{G,0} $ and $A_{G, 1}$ can be simulated by running Weisfeiler--Leman on $G_0$ and $G_1$.
    By the choice of $G$ as a high-treewidth graph, it follows that $G_0$ and $G_1$ are not distinguished by Weisfeiler--Leman \cite{neuen_homomorphism-distinguishing_2024}
    and hence the same holds for $A_{G,0}$ and $A_{G,1}$.
\end{enumerate}

\subsubsection{Definition of matrices from CFI graphs}
Let $G$ be a connected graph.
Let $G_U$ denote the CFI graph of $G$ for $U \subseteq V(G)$ as defined in \cref{def:cfi}.
By \cite[Corollary~3.7]{roberson_oddomorphisms_2022},
there are two CFI graphs for $G$ up to isomorphism.
That is, it holds that $G_U \cong G_{U'}$ if, and only if, $|U| \equiv |U'| \pmod 2$.
The two graphs are typically denoted by $G_0$, i.e., $|U| = 0$, and $G_1$, i.e., $|U| = 1$.
Subsequently, we will nevertheless work with $G_U$ in order to be able to define matrices based on CFI graphs. 

Fix an ordering on $V(G)$.
We think of the edges of $G$ as being oriented according to this ordering, i.e.\ every edge points from the vertex with lower number to the vertex of higher number.
Fix $U \subseteq V(G)$.
We define the following four matrices whose rows are indexed by $V(G_U)$.
See \cref{fig:integral-encoding} in \cref{app:integral-c3-details} for an example.

\begin{enumerate}
    \item The columns of $X_{G, U}$ are indexed by $E(G)$.
    The entries are given by
    \begin{equation}
        X_{G,U}((v,S),e) \coloneqq 
        \begin{cases} 
        1, & \text{if } e \in E(v) \text{ and } e \in S, \\ 
        -1, & \text{if } e \in E(v) \text{ and } e \notin S, \\ 
        0, & \text{if } e \notin E(v). 
        \end{cases}
        \label{eq:matrix_def}
    \end{equation}
    \item The columns of $X'_{G, U}$ are indexed by $E(G)$.
    The entries are given by 
    \begin{equation}
        X'_{G,U}((v,S),e) \coloneqq 
        \begin{cases} 
        X_{G,U}((v,S),e), & \text{if } v \text{ is the first vertex in } e, \\ 
        -X_{G,U}((v,S),e), & \text{if } v \text{ is the second vertex in } e, \\ 
        0, & \text{otherwise}.
        \end{cases}
        \label{eq:matrix_2_def}
    \end{equation}
    \item The columns of $I_{G, U}$ are indexed by $V(G)$.
    Let $W$ be the following $V(G) \times V(G)$ matrix; its first column is the all-ones vector and its remaining columns are pairwise orthogonal and orthogonal to the first column, hence form a basis for the orthogonal complement of the all-ones direction in $\mathbb{R}^{|V(G)|}$.
    \begin{equation}\label{eq:matrix-A}
        W \coloneqq \begin{pmatrix}
        1      & 1      & 1      & 1      & \cdots & 1      & 1      \\
        1      & -1     & 1      & 1      & \cdots & 1      & 1      \\
        1      & 0      & -2     & 1      & \cdots & 1      & 1      \\
        1      & 0      & 0      & -3     & \cdots & 1      & 1      \\
        \vdots & \vdots & \vdots & \vdots & \ddots & \vdots & \vdots \\
        1      & 0      & 0      & 0      & \cdots & -(|V(G)|-2) & 1      \\
        1      & 0      & 0      & 0      & \cdots & 0      & -(|V(G)|-1)
        \end{pmatrix}.
    \end{equation}
    The entries of $I_{G, U}$ are given by
    \begin{equation}
    I_{G,U}((v, S), u) \coloneqq W(v, u).
    \label{eq:matrix_3_def}
    \end{equation}

    \item Define $\tilde X_{G, U}$ as the concatenation of $X_{G, U}$, $X'_{G, U}$, and $I_{G, U}$.
    The rows of $\tilde X_{G, U}$ are indexed by $V(G_U)$.
    The columns of  $\tilde X_{G, U}$ are indexed by $E(G) \uplus E(G) \uplus V(G)$.
\end{enumerate}

We observe that $\tilde X_{G, U}$ is a square matrix when $G$ is a $3$-regular graph.

\begin{lemma}\label{lem:matrix-size}
    If $G$ is a $3$-regular $2n$-vertex graph, then $\tilde X_{G, U}$ is an $8n \times 8n$ matrix.
\end{lemma}
\begin{proof}
    The number of edges in a $3$-regular $2n$-vertex graph is $3n$. 
    For every $u \in V(G)$, the graph $G_{U}$ contains $2^{3-1}=4$ vertices $(u, S)$. Hence, there are $8n$ vertices in $G_U$.
    It follows that $\tilde X_{G, U}$ has $8n$ rows and $3n + 3n + 2n  = 8n$ columns.
\end{proof}

\subsubsection{Orthogonality}

We show that the columns of $\tilde{X}_{G, U}$ are orthogonal.
Hence, they can serve as the eigenbasis of a simple-spectrum multigraph.

\begin{lemma}\label{lem:orthog1}
 		If every vertex in $G$ has degree at least $3$, then
 		the columns of $X_{G, U}$ are orthogonal.
\end{lemma}
\begin{proof}
    Let $e_1, e_2 \in E(G)$ be distinct edges.
    If $e_1$ and $e_2$ do not share a vertex, then there is no row index $(v,S)$ such that the columns of $e_1$ and $e_2$ are both non-zero.
    Hence, in this case, the columns of $e_1$ and $e_2$ are clearly orthogonal.
    
    So suppose that $e_1 = uv$ and $e_2 = vw$ for $u \neq w$.
    The column of $vw$ is zero on the block for $u$ 
    and the column of $uv$ is zero on the block for~$w$.
    Hence, it remains to consider the block for~$v$.
    Thus, the inner product of the columns of $e_1$ and $e_2$ is
    \begin{equation}\label{eq:orthog1}
        \sum_{\substack{S \subseteq E(v) \\ |S| \equiv |\{v\} \cap U| \pmod 2}} 
        (-1)^{vu \in S}(-1)^{vw \in S}
        = 2^{d-3} - 2 \cdot 2^{d-3} + 2^{d-3} = 0
    \end{equation}
    for $d \coloneqq |E(v)|$.
    Here, the sum is over all row indices $(v, S)$ in the block for $v$.
    The condition $|S| \equiv |\{v\} \cap U| \pmod 2$ comes from the definition of the vertex set of $G_U$.
    By \cref{eq:matrix_def},
    we need to distinguish cases whether $vu \in S$ and $vw \in S$.
    If both edges or neither of them is contained in $S$, then the product of the two entries is $+1$, otherwise it is $-1$.

    Finally, we count how often each of these four scenarios arises.
    Here, we use that the sets $S$ are solutions to a system of linear equations over~$\ZZ_2$.
    The variables $x_e$ are indexed by edges $e \in E(v)$ incident to $v$. The constraints are
    \begin{align*}
        \sum_{e \in E(v)} x_e &= |\{v\} \cap U|
    \end{align*}
    and $x_{vu} = 1$ or $x_{vu} = 0$, and $x_{vw} = 0$ or $x_{vw} = 1$, depending on the given scenario.
    In other words, the sets $S$ are in bijection with solutions to a $d$-variable system of $3$ linearly independent equations over $\ZZ_2$.
    Hence, the number of solutions is $2^{d-3}$.
\end{proof}

\begin{lemma}\label{lem:orthog2}
    If every vertex in $G$ has degree at least $3$, then
    the columns of $X'_{G, U}$ are orthogonal.
\end{lemma}
\begin{proof}
    The argument is similar to \cref{lem:orthog1}.
    As above, we need to consider only the case of two columns indexed by $e_1 = uv$ and $e_2 = vw$ for $u \neq w$ (with $v$ the shared vertex).
    The definition of $X'_{G,U}$ negates each entry depending on whether $v$ is the first or second endpoint of the edge in the fixed total ordering of $V(G)$.
    Concretely, if $p_1 \in \{+1,-1\}$ is the sign contributed by edge $e_1$ (i.e., $p_1 = +1$ if $v$ is the first vertex of $e_1$, and $p_1 = -1$ otherwise), and $p_2$ is defined analogously for $e_2$, then the inner product over the block for $v$ is
    \[
    \begin{aligned}
    &\sum_{\substack{S \subseteq E(v) \\ |S| \equiv |\{v\} \cap U| \pmod 2}}
        \bigl(p_1 \cdot (-1)^{vu \in S}\bigr)\bigl(p_2 \cdot (-1)^{vw \in S}\bigr) \\
    &\qquad= p_1 p_2 \sum_{\substack{S \subseteq E(v) \\ |S| \equiv |\{v\} \cap U| \pmod 2}}
        (-1)^{vu \in S}(-1)^{vw \in S}.
    \end{aligned}
    \]
    The sum on the right equals $2^{d-3} - 2 \cdot 2^{d-3} + 2^{d-3} = 0$ by the same four-case counting as in \cref{eq:orthog1}, regardless of the values of $p_1, p_2 \in \{+1,-1\}$ (i.e., regardless of the ordering of $u,v,w$).
    This parity-symmetry argument covers all six orderings of $\{u,v,w\}$ simultaneously: the global sign factor $p_1 p_2$ is irrelevant to the cancellation, which depends only on the balanced count of subsets $S$ with $vu \in S$, $vw \in S$, both, or neither.
\end{proof}

\begin{lemma}\label{lem:orthog3}
  The columns of $W$ are pairwise orthogonal. If, in addition, $G$ is $d$-regular, the columns of $I_{G,U}$ are pairwise orthogonal as well.
\end{lemma}
\begin{proof}
  The columns of $W$ are orthogonal by direct computation of the inner products in \cref{eq:matrix-A}. When $G$ is $d$-regular, each base vertex $v$ contributes the same number $2^{d-1}$ of fiber rows replicating $W(v, \cdot)$, so $I_{G,U}^\top I_{G,U} = 2^{d-1}\, W^\top W$ and the orthogonality transfers to $I_{G,U}$.
\end{proof}

\begin{lemma}\label{lem:orthogonal}
    If $G$ is a $d$-regular graph for $d \geq 3$,
    then the columns of $\tilde X_{G, U}$ are orthogonal.
\end{lemma}
\begin{proof}
    By \cref{lem:orthog1,lem:orthog2,lem:orthog3},
    the columns of $X_{G, U}$, $X'_{G, U}$, and $I_{G, U}$ are respectively orthogonal.
    Thus, 
    it remains to consider inner products of pairs of columns stemming from different matrices $X_{G, U}$, $X'_{G, U}$, and $I_{G, U}$.
    Distinguish cases:
    \begin{enumerate}
        \item One column $e \in E(G)$ from $X_{G, U}$ or $X'_{G, U}$ and one column $u \in V(G)$ from $I_{G, U}$.

        The crucial observation is that the columns of $I_{G, U}$ are constant on sets of vertices $(v, S)$ belonging to the same vertex $v \in V(G)$.
        We argue that the columns of $X_{G, U}$ and $X'_{G, U}$ are orthogonal to the subspace of vectors in $\mathbb{R}^{V(G_U)}$ that are constant on these fibers.
        Let $u \in V(G)$ and consider the inner product of a column indexed by $e \in E(G)$ of $X_{G, U}$ or $X'_{G, U}$ with the vector in $\mathbb{R}^{V(G_U)}$ that is $1$ on vertices belonging to $u$ and zero otherwise.

        If $e$ is not incident to $u$, then the inner product is zero.
        Otherwise, up to sign, the inner product is equal to 
        \[
            \sum_{\substack{S \subseteq E(u) \\ |S| \equiv |U \cap \{u\}| \pmod 2}} (-1)^{e \in S}
            = 2^{d-2} - 2^{d-2} = 0.
        \]
        Here, as in \cref{lem:orthog1}, we observe that $S \subseteq E(u)$ such that $|S| \equiv |U \cap \{u\}| \pmod 2$ and $e \in S$ correspond to solutions to a system of linear equations over $\ZZ_2$ with $d$ variables and $2$ linearly independent constraints.

        \item One column $e \in E(G)$ from $X_{G, U}$ and one column $e' \in E(G)$ from $X'_{G, U}$.

        The case when $e$ and $e'$ intersect only in one vertex is similar to what has been argued in \cref{lem:orthog2}.
        If $e = e' = uv$ where wlog $u$ precedes $v$ in the ordering on $G$,
        then the inner product is
        \begin{align*}
            &\sum_{\substack{S \subseteq E(u) \\ |S| \equiv |U \cap \{u\}| \pmod 2}} (-1)^{e \in S}(-1)^{e' \in S}
            - 
            \sum_{\substack{T \subseteq E(v) \\ |T| \equiv |U \cap \{v\}| \pmod 2}} (-1)^{e \in T}(-1)^{e' \in T} \\
            &= 
             \sum_{\substack{S \subseteq E(u) \\ |S| \equiv |U \cap \{u\}| \pmod 2}} 1
             - \sum_{\substack{T \subseteq E(v) \\ |T| \equiv |U \cap \{v\}| \pmod 2}} 1 \\
            &=
            2^{d-1} - 2^{d-1} 
            = 0.
        \end{align*}
        Here, the assumption that $G$ is regular is crucial. \qedhere
    \end{enumerate}
\end{proof}

\subsubsection{Definition of the multigraphs}
\label{def:multigraphs}
By \cref{lem:orthogonal},
the matrix $\tilde{X}_{G, U}$ has orthogonal columns.
Furthermore, by \cref{eq:matrix-A}, the all-ones vector is one of these columns.
All entries of $\tilde{X}_{G, U}$ are integers.

We define a matrix 
\begin{equation}\label{eq:multigraph-def}
    A_{G,U} \coloneqq \tilde{X}_{G, U} D (\tilde{X}_{G, U})^\top
\end{equation}
by fixing some diagonal matrix $D$ with distinct diagonal entries.
By choosing a sufficiently large positive entry for the index corresponding to the all-ones vector in $\tilde{X}_{G, U}$, we can ensure that the matrix $ A_{G,U}$ has only positive integer entries.

Define $A_{G,0} \coloneqq A_{G, \emptyset}$ and $A_{G,1} \coloneqq A_{G, \{v\}}$ for an arbitrary vertex $v \in V(G)$.
Here, we use the same matrix $D$ for both $A_{G, 0}$ and $A_{G, 1}$.

Thus, we have defined a pair of symmetric simple-spectrum matrices with positive integer entries.
These will be the adjacency matrices of the multigraphs stipulated by \cref{thm:main-incomp}.

\subsubsection{No isomorphisms}

\begin{lemma}\label{lem:iso}
    The multigraphs $A_{G,0}$ and $A_{G, 1}$ are non-isomorphic.
\end{lemma}
\begin{proof}
    Towards a contradiction, 
    let $\pi\colon V(G_{1}) \rightarrow V(G_{0})$ be the isomorphism witnessing $A_{G,0} \cong A_{G, 1}$.
    We argue that this map gives an isomorphism $G_1 \to G_0$.
    No such map exists by \cite[Corollary~3.7]{roberson_oddomorphisms_2022}.
    
    Write $P \in \{0,1\}^{V(G_0) \times V(G_1)}$ for the permutation matrix encoding $\pi$. It holds that
    \[
        A_{G,0} = P A_{G,1} P^\top.
    \]
    It follows from this equation that $P$ maps eigenvectors of $A_{G,1}$ for eigenvalue $\lambda$
    to eigenvectors of $A_{G, 0}$ with the same eigenvalue.
    Furthermore, for any vector $v$, it holds that $Pv$ and $v$ have the same norm.
    Since the diagonal entries of the matrix $D$ in \cref{eq:multigraph-def}
    is the same for $A_{G,1}$  and $A_{G, 0}$,
    there are weights $\alpha_{e} \in \{\pm 1\}$ for every $e \in E(G)$ such that
\[
X_{G,1}((v,S),e) = \alpha_{e} \cdot X_{G,0}(\pi(v,S),e)
\]
for all $(v,S) \in V(G_{1})$ and $e \in E(G)$. We claim that $\pi$ is an isomorphism $G_{1} \rightarrow G_{0}$. Let $(v,S), (u,T) \in V(G_{U})$ be adjacent, i.e., $vu \in E(G)$ and $vu \notin S \triangle T$.

Let $e \coloneqq uv$. By \cref{eq:matrix_def}, it holds that $X_{G,1}((v,S),e) = X_{G,1}((u,T),e) \in \{\pm 1\}$. By assumption:
\[
X_{G,0}(\pi(v,S),e) = \alpha_{e} \cdot X_{G,1}((v,S),e) = \alpha_{e} \cdot X_{G,U}((u,T),e) = X_{G,U'}(\pi(u,T),e) \in \{\pm 1\}.
\]
It follows, again by \cref{eq:matrix_def}, that $\pi(v,S)$ and $\pi(u,T)$ are adjacent.
\end{proof}

\subsubsection{Weisfeiler--Leman indistinguishability}

We show that the Weisfeiler--Leman algorithm on $A_{G, 0}$ and $A_{G, 1}$ can be simulated by running it on the CFI graphs $G_0$ and $G_1$.
The CFI graphs are known not to be distinguished by Weisfeiler--Leman.

For two vertices $(v, S), (u, T) \in V(G_U)$, the entry $A_{G, U}((v, S), (u, T))$ is
\begin{align}
    A_{G, U}((v, S), (u, T))  = & \sum_{e \in E(G)} D_e \cdot X_{G, U}((v,S), e) \cdot X_{G, U}((u,T), e) \notag \\
    &+ 
    \sum_{e' \in E(G)} D'_{e'} \cdot X'_{G, U}((v,S), e') \cdot X'_{G, U}((u,T), e') \notag \\
    &+
    \sum_{w \in V(G)} D''_{w} \cdot I_{G, U}((v,S), w) \cdot I_{G, U}((u,T), w). \label{eq:multigraph-entries-expanded}
\end{align}
Here, $D, D', D''$ denote the values of the diagonal matrix from \cref{eq:multigraph-def}.
As a first step, we compute the products in each of the sums explicitly.

\begin{lemma}
Let $G$ be a graph and $u, v,w \in V(G)$ with $u \neq v$. For $U \subseteq V(G)$ and $e \in E(G)$, it holds that
\begin{align} \allowdisplaybreaks
    X_{G,U}((u,S),e)^{2} = X'_{G,U}((u,S),e)^{2} &= 
\begin{cases}
1, & \text{if } e \in E(u), \\
0 ,& \text{otherwise}.
\end{cases}
\label{eq:sq_entry} \\
X_{G,U}((u,S),e) \cdot X_{G,U}((v,T),e) &= 
\begin{cases}
1, & \text{if } e = uv \text{ and } (u,S) \sim (v,T) \\
-1, &  \text{if } e = uv \text{ and } (u,S) \not\sim (v,T) \\
0, & \text{otherwise}.
\end{cases}
\label{eq:prod_entry} \\
X'_{G,U}((u,S),e) \cdot X'_{G,U}((v,T),e) &= 
\begin{cases}
-1, & \text{if } e = uv \text{ and } (u,S) \sim (v,T) \\
1, &  \text{if } e = uv \text{ and } (u,S) \not\sim (v,T) \\
0, & \text{otherwise}.
\end{cases}
\label{eq:prod_entry_b} \\
X_{G,U}((u,S),e) \cdot X_{G,U}((u,T),e) &= 
X'_{G,U}((u,S),e) \cdot X'_{G,U}((u,T),e) \notag \\
&=
\begin{cases}
1, & \text{if } e \in E(u) \text{ and } e \notin S \triangle T, \\
-1, & \text{if } e \in E(u) \text{ and } e \in S \triangle T, \\
0, & \text{otherwise}.
\end{cases}
\label{eq:prod_entry_c}
\end{align}
\end{lemma}
\begin{proof}
    Immediate from \cref{eq:matrix_def,eq:matrix_2_def}.
    \Cref{eq:prod_entry_b} differs from \cref{eq:prod_entry} only by the sign flip.
    In \cref{eq:prod_entry_c} the sign flip is squared and hence neutralized.
\end{proof}

For a CFI graph $G_U$, write $G_{U}^{*}$ for the vertex-colored graph obtained from $G_U$ by coloring $(v, S)$ by~$v$. That is, every vertex fiber gets a unique color.
Recall \cref{def:k-wl}.
Since we are comparing WL colors of multiple graphs, we write $C^{(t)}(G, \bar{v})$ for the round-$t$  $k$-WL color of the tuple $\bar{v} \in V(G)^k$.

\begin{lemma}\label{lem:wl-cfi-to-multigraph}\label{cor:wl-graph-to-matrix}
Let $k \ge 1$ and $r \ge 0$. Let $G$ be a graph of minimum degree $\geq 3$ and $U, U' \subseteq V(G)$. For all $\bar{x} \in V(G_{U})^{k}$ and $\bar{x'} \in V(G_{U'})^{k}$,
\[
C^{(r+1)}(G_{U}^{*},\bar{x}) = 
C^{(r+1)}(G_{U'}^{*},\bar{x'}) \implies 
C^{(r)}(A_{G,U},\bar{x}) = C^{(r)}(A_{G,U'},\bar{x'}).
\]
In particular, if $G_{U}^{*}$ and $G_{U'}^{*}$ are not distinguished by the $k$-dimensional Weisfeiler--Leman algorithm, 
then $A_{G,U}$ and $A_{G,U'}$ are not distinguished by the $k$-dimensional Weisfeiler--Leman algorithm.
\end{lemma}

\begin{proof}
It suffices to establish the base case $r=0$. 
The rest follows by induction. 
To ease notation, consider the case $k=2$. Let $\bar{x}=((u,S), (v,T))$ and $\bar{x'}=((u,S'), (v,T'))$. Note that since $G_{U}^{*}$ and $G_{U'}^{*}$ are vertex-colored graphs, we may suppose that the first components $u, v$ of $\bar{x}$ and $\bar{x'}$ are the same.

We shall show that $C^{(1)}(G_{U}^{*},\bar{x}) = 
C^{(1)}(G_{U'}^{*},\bar{x'})$
implies, for every $e \in E(G)$, and $w \in V(G)$ that 
\begin{align}
    X_{G,U}((u,S),e) \cdot X_{G,U}((v,T),e) &= 
    X_{G,U'}((u,S'),e) \cdot X_{G,U'}((v,T'),e), \label{eq:ip1} \\
    X'_{G,U}((u,S),e) \cdot X'_{G,U}((v,T),e) &= 
    X'_{G,U'}((u,S'),e) \cdot X'_{G,U'}((v,T'),e), \text{ and } \label{eq:ip2}\\
    I_{G, U}((u, S), w) \cdot I_{G, U}((v, T), w)
    &= I_{G, U'}((u, S'), w) \cdot I_{G, U'}((v, T'), w).\label{eq:ip3}
\end{align}
By \cref{eq:multigraph-entries-expanded}, these identities imply that
\[
    A_{G, U}((u, S), (v,T)) = A_{G, U'}((u, S'), (v,T')),
\]
as desired.

\Cref{eq:sq_entry,eq:prod_entry,eq:prod_entry_b} readily
imply \cref{eq:ip1,eq:ip2} if $u \neq v$
since the initial color $C^{(0)}(G_{U}^{*},\bar{x}) = 
C^{(0)}(G_{U'}^{*},\bar{x'})$
encodes whether $(u,S) \sim (v, T)$.

The case $u = v$ is more subtle since the condition \cref{eq:prod_entry_c} is not encoded in the initial coloring.  
By \cref{eq:prod_entry_c}, 
it suffices to consider the case $e = uw$ for some $w \in V(G)$.

\begin{claim}\label{claim:shared-neighbor}
It holds that $e = uw \notin S \triangle T$ if, and only if, there exists a shared neighbor $(w, R)$ of $(u, S)$ and $(u, T)$.
\end{claim}
\begin{claimproof}
    For the forward direction, by the assumption that every vertex in $G$ has at least three neighbors, there exist two vertices $(w, R_{0}), (w, R_{1}) \in V(G_{U})$ such that $uw \notin R_{0}$ and $uw \in R_{1}$. Depending on whether $e \in S$ or not, it follows that $(u, S)$ and $(w, R_{0})$ or $(w, R_{1})$ are adjacent. 
    Hence, if $e \not\in S \triangle T$, that is, $S$ and $T$ \enquote{agree} about $e$,
    then $(u, S)$ and $(u, T)$ are both adjacent to one of $(w, R_{0})$, $(w, R_{1})$.
    
    Conversely, if $(u, S)$ and $(u, T)$ are both adjacent to $(w, R)$, 
    then $uw \not\in S \triangle R$ and $uw \not\in T \triangle R$ by \cref{def:cfi}.
    Hence, $uw \not\in (S \triangle R) \triangle (T \triangle R) = S \triangle T$, as desired.
\end{claimproof}

After one iteration of Weisfeiler--Leman, 
the color $C^{(1)}(G_{U}^{*},\bar{x}) = 
C^{(1)}(G_{U'}^{*},\bar{x'})$ encodes whether $(u,S)$ and $(v, T)$ 
share a neighbor of color $w$ for $e = uw$.
Finally, \cref{claim:shared-neighbor,eq:prod_entry_c} imply \cref{eq:ip1,eq:ip2} in the case $u = v$.

It remains to consider \cref{eq:ip3}.
By \cref{eq:matrix_3_def},
\[
I_{G, U}((u, S), w) \cdot I_{G, U}((v, T), w)
= W(u,w) \cdot W(v, w)
= I_{G, U'}((u, S'), w) \cdot I_{G, U'}((v, T'), w).
\]
In other words, the value of $I_{G, U}((u, S), w)$ does not depend on $S$.
\end{proof}

\subsubsection{Summary}
This concludes the preparations for the main theorem.

\mainIncomp*

\begin{proof}
By \cite{kolesnik_lower_2014} and \cite[Lemma~5 and Corollary~7]{dvorak_strongly_2016},
there exists, for every sufficiently large $n \in \mathbb{N}$,
a connected $3$-regular $2n$-vertex graph $G_n$ of treewidth $\geq \frac{\alpha}{3(\alpha+1)} \cdot 2n - 1 = \frac{n}{12} - 1$ for $\alpha = \frac17$.
By \cite[Lemma~12]{neuen_homomorphism-distinguishing_2024} and \cite[Corollary~4.7]{grohe_logic_2021}, 
the vertex-colored CFI graphs $G_{n,0}^{*}$ and $G_{n,1}^{*}$ of $G_{n}$ are not distinguished by the $\left\lceil \frac{n}{12} - 1 \right\rceil$-dimensional Weisfeiler--Leman algorithm. 
By \cref{cor:wl-graph-to-matrix}, so are multigraphs $A_{G_n,0}$ and $A_{G_n,1}$.
By \cref{lem:iso}, the multigraphs $A_{G_n,0}$ and $A_{G_n,1}$ are non-isomorphic.
They have simple spectrum by construction in \cref{def:multigraphs}.
\end{proof}

\subsection{CFI Construction on $C_3$: Detailed Example}\label{app:cfi-c3-details}
\begin{figure}[t]
    \centering
    \includegraphics[width=\textwidth]{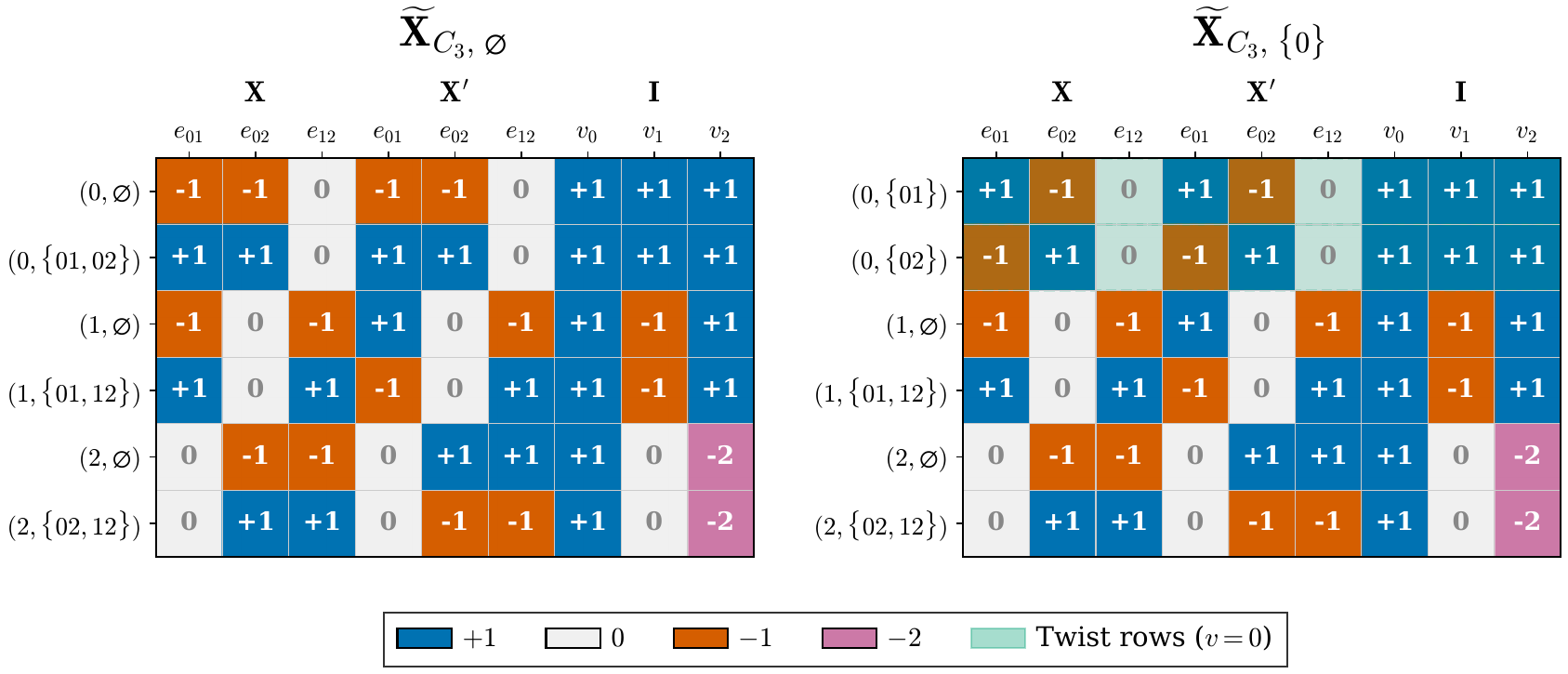}
    \caption{Integral encoding matrices for the CFI construction on $C_3$.
    \textbf{Left:}~$\tilde X_{C_3,\varnothing}$ (trivial lift): the vertex-$0$ rows have \emph{coherent} signs in the $X$-block.
    \textbf{Right:}~$\tilde X_{C_3,\{0\}}$ (non-trivial lift, highlighted rows): the parity twist at vertex~$0$ \emph{splits} the signs, making the two matrices non-isomorphic under $S_6 \times \{\pm 1\}^{9}$.}
    \label{fig:integral-encoding}
\end{figure}
We provide the full details of the CFI construction on $C_3$ (\cref{ex:cfi-c3}).
Let $G = C_3$, the triangle on vertex set $\{0, 1, 2\}$ with edges
$\{0,1\}$, $\{1,2\}$, $\{0,2\}$.
Each vertex $i$ has degree~$2$, so $|E(i)| = 2$ and the even-cardinality
subsets of $E(i)$ are $\varnothing$ and $E(i)$ itself,
yielding two fiber vertices per base vertex and six CFI vertices in total.

\paragraph{Trivial lift ($U = \varnothing$).}
The fiber of each vertex~$i$ is
$\{(i, \varnothing),\; (i, E(i))\}$.
Concretely:
\begin{align*}
    \text{vertex } 0 &: \quad (0,\varnothing), \quad (0,\{\{0,1\},\{0,2\}\}), \\
    \text{vertex } 1 &: \quad (1,\varnothing), \quad (1,\{\{0,1\},\{1,2\}\}), \\
    \text{vertex } 2 &: \quad (2,\varnothing), \quad (2,\{\{0,2\},\{1,2\}\}).
\end{align*}
For each base edge $e = \{i,j\}$, the adjacency condition
$e \notin S \triangle T$ forces $S$ and $T$ to agree on~$e$.
Since each fiber has exactly one vertex containing~$e$ and one not containing~$e$,
each base edge contributes two edges, one connecting the
``contains~$e$'' pair and one connecting the ``does not contain~$e$'' pair.
The $\varnothing$-type vertices $\{(i,\varnothing)\}_{i=0}^{2}$ therefore form a
triangle among themselves, as do the $(i,E(i))$-type vertices, with no edges between the two groups.
Hence $G_0 \cong 2\,C_3$.

\paragraph{Non-trivial lift ($U = \{0\}$).}
The parity constraint at vertex~$0$ shifts to
odd-cardinality subsets:
$(0, \{\{0,1\}\})$ and $(0, \{\{0,2\}\})$.
Vertices~$1$ and~$2$ retain their even-parity fibers.
The adjacency rewiring at base edges $\{0,1\}$ and $\{0,2\}$ produces:
\begin{alignat*}{2}
    (0,\{\{0,1\}\}) &\sim (1,\{\{0,1\},\{1,2\}\}), &\qquad
    (0,\{\{0,1\}\}) &\sim (2,\varnothing), \\
    (0,\{\{0,2\}\}) &\sim (1,\varnothing), &\qquad
    (0,\{\{0,2\}\}) &\sim (2,\{\{0,2\},\{1,2\}\}).
\end{alignat*}
Together with the unchanged edges
$(1,\varnothing) \sim (2,\varnothing)$ and
$(1,\{\{0,1\},\{1,2\}\}) \sim (2,\{\{0,2\},\{1,2\}\})$
from base edge $\{1,2\}$, the six vertices form a single cycle:
\begin{align*}
    (0,\{\{0,1\}\}) &\to (1,\{\{0,1\},\{1,2\}\}) \to (2,\{\{0,2\},\{1,2\}\}) \to (0,\{\{0,2\}\}) \\
    &\to (1,\varnothing) \to (2,\varnothing) \to (0,\{\{0,1\}\}).
\end{align*}
Hence $G_{\{0\}} \cong C_6$.

\subsection{Integral Encoding on $C_3$: Detailed Example}\label{app:integral-c3-details}

We provide the full integral encoding matrices for the CFI construction on $C_3$ (\cref{def:integral-enc}).
Let $G = C_3$ on vertices $\{0,1,2\}$ with oriented edges
$e_{01} = (0,1)$, $e_{02} = (0,2)$, $e_{12} = (1,2)$,
where the smaller vertex is always the first endpoint.
Each vertex has degree~$2$, so each fiber contains two CFI vertices
and $|V(\operatorname{CFI}(G,U))| = 6$ for every~$U$.
The integral encoding $\mathbf{X}^{*}_{G,U} = [\,\mathbf{X} \mid \mathbf{X}' \mid \mathbf{I}\,] \in \mathbb{Z}^{6 \times 9}$.

\paragraph{Trivial lift ($U = \varnothing$).}
\[
\mathbf{X}^{*}_{C_3,\varnothing} =
\bordermatrix{
              & e_{01} & e_{02} & e_{12} & e_{01} & e_{02} & e_{12} & v_0 & v_1 & v_2 \cr
(0,\varnothing)       & -1 & -1 &  0 & -1 & -1 &  0 & 1 & 1 & 1 \cr
(0,\{01,02\})         & +1 & +1 &  0 & +1 & +1 &  0 & 1 & 1 & 1 \cr
(1,\varnothing)       & -1 &  0 & -1 & +1 &  0 & -1 & 1 & -1 & 1 \cr
(1,\{01,12\})         & +1 &  0 & +1 & -1 &  0 & +1 & 1 & -1 & 1 \cr
(2,\varnothing)       &  0 & -1 & -1 &  0 & +1 & +1 & 1 & 0 & -2 \cr
(2,\{02,12\})         &  0 & +1 & +1 &  0 & -1 & -1 & 1 & 0 & -2
}
\]
The first three columns form~$\mathbf{X}$, the next three form~$\mathbf{X}'$
(which flips the sign whenever $v$ is the \emph{second} endpoint of~$e$),
and the last three form the indicator~$\mathbf{I}$.
Note that in the $\mathbf{X}$-block the vertex-$0$ rows have \emph{coherent}
signs: $(0,\varnothing)$ has $(-1,-1)$ and $(0,\{01,02\})$ has $(+1,+1)$
on its incident edges, that is both entries share the same sign.

\paragraph{Non-trivial lift ($U = \{0\}$).}
The parity twist at vertex~$0$ replaces the even-cardinality subsets
$\{\varnothing, \{01,02\}\}$ with the odd-cardinality subsets
$\{\{01\}, \{02\}\}$.  Only the first two rows change:
\[
\mathbf{X}^{*}_{C_3,\{0\}} =
\bordermatrix{
              & e_{01} & e_{02} & e_{12} & e_{01} & e_{02} & e_{12} & v_0 & v_1 & v_2 \cr
(0,\{01\})            & +1 & -1 &  0 & +1 & -1 &  0 & 1 & 1 & 1 \cr
(0,\{02\})            & -1 & +1 &  0 & -1 & +1 &  0 & 1 & 1 & 1 \cr
(1,\varnothing)       & -1 &  0 & -1 & +1 &  0 & -1 & 1 & -1 & 1 \cr
(1,\{01,12\})         & +1 &  0 & +1 & -1 &  0 & +1 & 1 & -1 & 1 \cr
(2,\varnothing)       &  0 & -1 & -1 &  0 & +1 & +1 & 1 & 0 & -2 \cr
(2,\{02,12\})         &  0 & +1 & +1 &  0 & -1 & -1 & 1 & 0 & -2
}
\]
The critical difference is in the vertex-$0$ rows of the $\mathbf{X}$-block:
the signs now \emph{split}, $(0,\{01\})$ has $(+1,-1)$ and $(0,\{02\})$
has $(-1,+1)$, encoding the parity twist.
Since no simultaneous sign flip and row permutation can convert coherent
pairs to split pairs, the two matrices are non-isomorphic under the
action of $\{{\pm 1}\}^{9} \times S_6$.

\subsection{Spectrum of CFI Graphs}\label{sub:not_simple}

As a justification for the construction above, 
we show that the CFI graphs from \cref{def:cfi} do not have simple spectrum.

\begin{theorem}\label{thm:cfi-spectrum}
    Let $G$ be a $3$-regular graph with $2n$ vertices and spectrum $\lambda_1, \dots, \lambda_{2n}$.
    Then the spectrum of the even CFI graph $G_0$ of $G$ is
    \[
        \underbrace{2 \lambda_1, \dots, 2 \lambda_{2n}}_{2n},\quad
        \underbrace{+2, \dots, +2}_{3n},\quad
        \underbrace{-2, \dots, -2}_{3n}.
    \]
    In particular, $G_0$ does not have simple spectrum.
\end{theorem}

If $G$ has treewidth at least $3$, e.g.\ when $G$ is taken to be the complete graph on $4$ vertices,
then the CFI graphs $G_0$ and $G_1$ have the same spectrum \cite{neuen_homomorphism-distinguishing_2024}, \cite[Theorem~3.3.2]{seppelt_homomorphism_2024}.
In this case, neither $G_0$ nor $G_1$ have simple spectrum regardless of whether the spectrum of $G$ is simple.
The proof of \cref{thm:cfi-spectrum} shows that the matrices $\tilde X_{G, \emptyset}$ constructed in \cref{app:main-incomp} roughly correspond to the eigenvectors of the CFI graph $G_0$.

\begin{proof}
    The first $2n$ eigenvalues of $G_0$ correspond to the eigenvalues of $G$.
    The other $6n$ eigenvalues stem from the edges of $G$.

    \paragraph{Eigenvectors constant on fibers.}
    Let $y \in \mathbb{R}^{V(G)}$ be an eigenvector of $G$ for eigenvalue $\lambda$.
    Define $x \in \mathbb{R}^{V(G_0)}$ by $x_{(v, S)} = y_v$.
    We claim that $x$ is an eigenvector of $G_0$ with eigenvalue $2 \lambda$.
    Indeed, for every $(v, S) \in V(G_0)$,
    \[
        \sum_{(u, T) \sim (v, S)} x_{(u, T)}
        = \sum_{(u, T) \sim (v, S)} y_{u}
        = \sum_{u \sim v} 2 \cdot y_u
        = 2 \lambda \cdot y_v
        = 2 \lambda \cdot x_{(v, S)}.
    \]
    Here, the second equality holds since for every $(v, S)$ and $u \sim v$ there are precisely $2$ sets $T$ such that $(u, T) \sim (v,S)$.
    As argued in \cref{app:main-incomp}, the sets $T$ correspond to solutions of a system of linear equations over $\ZZ_2$.
    The system has $3$ variables and $2$ equations: one enforcing that $|T| \equiv 0 \pmod 2$ and one guaranteeing that $uv \not\in S \triangle T$.
    This system has $2$ solutions.
    Note that these eigenvectors are constant on fibers, i.e.\ $x_{(v,S)} = x_{(v, T)}$ for all $S, T$.

    \paragraph{Eigenvectors corresponding to edges.}
    Next we construct eigenvectors corresponding to the $3n$ edges of $G$. Recall \cref{eq:matrix_def,eq:matrix_2_def}.
    Fix an ordering $V(G) = \{1, \dots, 2n\}$ of the vertices of $G$.
    For an edge $e \in E(G)$,
    define the vector $y^e \in \mathbb{R}^{V(G_0)}$ via
    \[
        y^e_{(v, S)} \coloneqq \begin{cases}
           1, & \text{if } e \in E(v) \text{ and } e \in S, \\ 
           -1, & \text{if } e \in E(v) \text{ and } e \not\in S,\\
           0, &\text{otherwise}.
        \end{cases}
    \]
    Towards verifying that $y^e$ is an eigenvector of the adjacency matrix of $G_0$, 
    let $(v, S) \in V(G_0)$.
    Distinguish cases:
    \begin{itemize}
        \item The edge $e$ is not incident to $v$ or one of its neighbors. We have
        \[
            \sum_{(u, T) \sim (v, S)} y^e_{(u, T)}
            = 0
            = y^e_{(v, S)}
            = 2 y^e_{(v, S)}.
        \]
        \item The edge $e$ is incident to two neighbors $w_1, w_2$ of $v$.
        For each $i \in \{1,2\}$, the number of even subsets $T \subseteq E(w_i)$ such that $vw_i \not\in S \triangle T$ and $e \in T$
        is one, and is thus equal to the number of even subsets $T' \subseteq E(w_i)$ such that $vw_i \not\in S \triangle T'$ and $e \not\in T'$.
        Hence each of $w_1$ and $w_2$ contributes $+1$ from the $e\in T$ term and $-1$ from the $e\notin T$ term, giving
        \[
            \sum_{(u, T) \sim (v, S)} y^e_{(u, T)}
            = \underbrace{(1 - 1)}_{w_1} + \underbrace{(1 - 1)}_{w_2}
            = 0
            = y^e_{(v, S)}
            = 2 y^e_{(v, S)}.
        \]
        \item The edge $e$ is incident to precisely one neighbor $w$ of $v$ and not to $v$.
        This case is analogous to the previous case.
        \item The edge $e$ is incident to $v$, i.e.\ $e = vw$ for some vertex $w$.
        The number of even subsets $T \subseteq E(w)$ such that $e \not\in S \triangle T$ and $e \in T$ is either zero or two, depending on whether $e \in S$.
        In both cases,
        \[
        \sum_{(u, T) \sim (v, S)} y^e_{(u, T)}
        = 2 y^e_{(v, S)}.
        \]
    \end{itemize}
    Hence, $y^e$ is an eigenvector of the adjacency matrix of $G_0$ with eigenvalue~$2$.

    For an edge $e \in E(G)$ and $(v, S) \in V(G_0)$, 
    the remaining eigenvalues are given by
    \[
        y^{-e}_{(v,S)} \coloneqq \begin{cases}
            y^e_{(v, S)} &\text{if }v \text{ is the first vertex in } e,\\
            -y^e_{(v, S)} & \text{if }v \text{ is the second vertex in } e,\\
            0, & \text{if } e \not\in E(v),
        \end{cases}
    \]
    see \cref{eq:matrix_2_def}.
    In order to argue that $y^{-e}$ is an eigenvector to eigenvalue $-2$, we conduct a case distinction as above. The first three cases are as for $y^e$.
    In the final case, only the sign is flipped, i.e.\
    \[
    \sum_{(u, T) \sim (v, S)} y^{-e}_{(u, T)}
    = - 2 \cdot y^{-e}_{(v, S)},
    \]
    as desired.
    
    By \cref{lem:orthogonal}, the eigenvectors constructed above form an orthogonal basis.
    Hence, we have mapped out the entire spectrum of the adjacency matrix of $G_0$.
\end{proof}

\section{Canonicalization of Simple Spectrum Graphs}\label{apdx:canon-simple}

This appendix accompanies \cref{sec:canon-simple-spec} and gives the formal proof of \cref{thm:s-s-canon}, working with the partition $\{C_1,\dots,C_L\}$ produced by \cref{alg:abs-value-partition,alg:balanced-refine}.

\subsection{Proof of \cref{thm:s-s-canon}}\label{apdx:canon-proof}
In this proof we identify $\{-1,1\}^k $ with the $\ZZ_2^k=\{0,1\}^k $. The action of $a\in \ZZ_2^k$ on $U\in \RR^{n\times k}$, denoted by $a.U$, is given by multiplying, for each $j\in \{1,\ldots,k\}$, the $j$-th column of $U$ by a corresponding sign $(-1)^{a_j} $. We denote the summation of two vectors $s,t\in \ZZ_2^k $ by $s\oplus t $.

We now prove that the map $\prism$ from \cref{sec:canon-simple-spec} is a valid canonicalization in the sense of \cref{def:canon}. The proof follows the standard reduction strategy: we first define a \emph{sign canonicalization} $s:\RR^{n\times k}\to\ZZ_2^k$ (\cref{def:sgn-canon}) and show that any such $s$ induces a canonicalization $c$ via lexicographic row sorting (\cref{lem:sign-to-canon}). The technical core of the proof, in \cref{apdx:proof-construction}, then constructs an explicit sign canonicalization on top of the partition $\{C_1,\dots,C_L\}$ output by \cref{alg:abs-value-partition,alg:balanced-refine}.

\subsubsection{Reduction to sign canonicalization}\label{apdx:proof-reduction}
We begin by introducing a definition of choosing unique representative of simple-spectrum eigendecompositions only up to sign ambiguities, or actions of the group $\ZZ_2^k$ for a $k$-eigendecomposition. We will then show that once we obtain this canonicalization it naturally extends to the full ambiguity group $S_n \times \ZZ_2^k$ using lexicographic sorting.
\begin{definition}\label{def:sgn-automorphism}
We say that $a\in \ZZ_2^k $ is a sign automorphism of $U\in \RR^{n \times k}$, if $a.U $ is equal to $U$, up to permutation.
\end{definition}

\begin{definition}\label{def:sgn-canon}
We say that $s:\RR^{n \times k}\to \ZZ_2^k $ is a sign-canonicalization, if 
\begin{enumerate}
	\item \textbf{Permutation invariance:} $s$ is invariant to permutations.
	\item \textbf{Sign equivariance:} $s(t.U)\oplus (t \oplus s(U)) $ is a sign automorphism of $U$, for all $U$ and $t\in \ZZ_2^k$.
\end{enumerate}
\end{definition}
We note that when $U$ has no non-trivial automorphisms, the second condition coincides with the standard notion of equivariance of $s$, namely $s(t.U)=t \oplus s(U) $ for all $U$ and $t\in \ZZ_2^k$.

\begin{lemma}\label{lem:sign-to-canon}
If $s$ is a sign canonicalization, then a canonicalization $c(U)$ can be defined by ordering $s(U).U$ in lexicographical ordering.
\end{lemma}
\begin{proof}
For a given matrix $U$, permutation $\tau$ and sign vector $t$, we need to show that $c((\tau,t).U)=c(U)$. Equivalently, we need to show that $s(t.U).(t.U) \cong_{S_n} s(U).U $, where we use $\cong_{S_n}$ to denote equivalence up to permutation.  Indeed
$$s(U).U\cong_{S_n}s(U).(s(t.U)\oplus t \oplus s(U)).U=(s(t.U)\oplus t).U=s(t.U).(t.U)\qedhere $$
\end{proof}

By \cref{lem:sign-to-canon}, it suffices to construct a sign canonicalization. The remainder of the proof does this on top of the partition output by \cref{alg:abs-value-partition,alg:balanced-refine}.

\subsubsection{Spectral refinement and the affine subspaces $A_\ell$}\label{apdx:proof-construction}

\paragraph{Setup.}
Throughout this subsection we denote the rows of $U$ by $U_i$ and write $\odot$ for the elementwise product of vectors. Let $\{C_1,\ldots,C_L\}$ be the partition of the vertex set $[n]$ produced by Stage~1 (\textsc{Partition}, \cref{alg:abs-value-partition}) followed by Stage~2 (\textsc{Refine}, \cref{alg:balanced-refine}) of \cref{sec:canon-simple-spec}. Two vertices $i,j$ lie in the same class $C_\ell$ only if they share the absolute-value signature $\sigma(v)=(|U_{v,1}|,\ldots,|U_{v,k}|)$, so within each class
\[
|U_i|=|U_j|\quad\text{for all }i,j\in C_\ell .
\]
In particular, the support set $K_\ell=\{j\in[k]:|U_{v,j}|>0\}$ is well-defined for each $C_\ell$ (it does not depend on the choice of $v\in C_\ell$). For every nonzero entry $U_{i,k}$ let $s_{i,k}$ denote its sign bit, namely $\sign\left(U_{i,k} \right)=(-1)^{s_{i,k}}$, and let $s_i=(s_{i,k})_{U_{i,k}\neq 0}\in\ZZ_2^{K_\ell}$ for $i\in C_\ell$.

\paragraph{Spectral WL refinement.}
We record the spectral Weisfeiler--Leman iteration that \cref{alg:balanced-refine} runs, since the proof uses its termination property directly. Initialize a coloring for each node $i\in[n]$ by
\begin{equation}\label{eq:init}
h_i^{(0)} =U_i\odot U_i ,
\end{equation}
and refine via
\begin{equation}\label{eq:wl-update-inv}
    h_i^{(t+1)} = \mathrm{UPDATE}_{(t)}\Big(h_i^{(t)},\big\{(h_j^{(t)}, U_i \odot U_j )\mid j=1,\ldots,n\big\}\Big),
\end{equation}
where $\{\cdot\}$ denotes a multiset and $\mathrm{UPDATE}_{(t)}$ is invariant to the order of its multiset argument. The colors $h^{(t)}$ define a sequence of partitions $C_1^{(t)},\ldots,C_{k_t}^{(t)}$ where $i,j$ are in the same class iff $h_i^{(t)}=h_j^{(t)}$. Each successive partition refines the previous one, so the iteration stabilizes at some $t\leq n$; this fixed-point partition $\{C_1,\ldots,C_L\}$ is the output of \cref{alg:balanced-refine}. By definition of the WL fixed point, classes at termination have a stable multiset of edge features:
\begin{equation}\label{eq:stable-multiset}
\{U_i \odot U_k\mid k\in C_\ell\}=\{U_j \odot U_k\mid k\in C_\ell\}\quad\text{for all }i,j\in C_\ell,\ \ell\in[L].
\end{equation}
This is the structural fact the proof relies on. The iteration \eqref{eq:wl-update-inv} is precisely \cref{alg:balanced-refine} (defined in \cref{apdx:canon-simple}); we describe it again here only to make explicit that the partition $\{C_\ell\}$ used by the proof is the spectral-WL fixed point. \Cref{rem:balanced-refine-realization} gives an equivalent $\ZZ_2$ parity-check realization used in our implementation. The same partition arises in the classical bounded-eigenvalue-multiplicity isomorphism algorithms of \citet{babai} and \citet{spielman}, which use related sign-product invariants to refine an initial partition; the novelty of our approach is to take this stable partition as the starting point for an \emph{explicit linear sign canonicalization} (\cref{alg:sign-solve}), as opposed to using it for isomorphism testing alone.

\paragraph{The affine subspaces $A_\ell$.}
For $i\in C_\ell$ define $A_\ell=\{s_i:i\in C_\ell\}\subseteq\ZZ_2^{K_\ell}$.

\begin{lemma}\label{lem:affine}
For every $\ell\in [L]$, we have that the set 
$$A_\ell=\{s_i, i\in C_\ell \} $$
is an affine subspace of $\ZZ_2^{K_{\ell}} $.
\end{lemma}
\begin{proof}
By the multiset stability \eqref{eq:stable-multiset}, for every $i,j\in C_\ell$,
$$\{U_i \odot U_k\mid k\in C_\ell \}=\{U_j \odot U_k\mid k\in C_\ell \}, $$
which under the sign-bit projection becomes
\begin{equation}\label{eq:closed}
\{s_i \oplus s_k\mid k\in C_\ell \}=\{s_j \oplus s_k\mid k\in C_\ell \}.
\end{equation}
To show $A_\ell$ is an affine subspace, it suffices to show the set $\{s_i \oplus s_k\mid k\in C_\ell\}$ is a linear subspace (translating by $s_i$ then yields $A_\ell$). It contains the zero vector $s_i\oplus s_i$. For closure under addition, given $k,j\in C_\ell$ we have $(s_i\oplus s_k)\oplus(s_i\oplus s_j)=s_j\oplus s_k$, which by \eqref{eq:closed} equals $s_i\oplus s_q$ for some $q\in C_\ell$. We note that in this proof we treated the rows $\{U_i\}_{i \in C_\ell}$ as pairwise distinct so that $\{s_i \oplus s_k \mid k \in C_\ell\}$ is a set; it is possible in general that some of the rows coincide. However, in this case, since the coloring is stable all rows will have equal multiplicity, and so the argument carries through verbatim with sets replaced by multisets. Note also that under a full simple-spectrum eigendecomposition, that is $n$-eigendecomposition, the rows are pairwise distinct since they are pairwise orthogonal.
\end{proof}

\subsubsection{Sign automorphisms of $U$}\label{apdx:proof-automorphisms}

Let $W_\ell=A_\ell \oplus A_\ell $. This is the subspace associated with the affine space $A_\ell$, and for every $i\in C_\ell$ we have $W_\ell=\{s_i \oplus s_k\mid k\in C_\ell\}$. The space $W_\ell$ has $|K_\ell|$ coordinates; we extend it to a subspace $\hat W_\ell\subseteq\ZZ_2^k$ consisting of all vectors whose restriction to $K_\ell$ lies in $W_\ell$. 

The next lemma identifies the sign-automorphism group $\hat W$ that the canonicalization must respect.

\begin{lemma}\label{lem:sign-aut}
The set of sign automorphisms of $U$ is precisely
$$\hat W:=\cap_{\ell=1}^L \hat W_\ell .$$
\end{lemma}
\begin{proof}
Let $A$ be the set of sign automorphisms of $U$. By \cref{def:sgn-automorphism}, $A = \{a \in \mathbb{Z}_2^k \mid \exists \tau \in S_n \text{ s.t. } (\tau, a) \cdot U = U \}$. We must prove $A = \hat{W}$.

\begin{enumerate}
    \item \textbf{Proof of $A \subseteq \hat{W}$:}
    Let $a \in A$. By definition, there exists a $\tau \in S_n$ such that $g = (\tau, a) \in \mathrm{Aut}(U)$.
    By construction of the partition $\{C_\ell\}_{\ell=1}^L$ from absolute-value signatures (\cref{alg:abs-value-partition,alg:balanced-refine}), any $g \in \mathrm{Aut}(U)$ has its associated permutation $\tau$ mapping each partition class to itself. (Permutations preserve absolute-value signatures, so $\tau(v) \in C_\ell$ whenever $v \in C_\ell$.) This means that for every $\ell \in [L]$, $\tau(C_\ell) = C_\ell$.
    This implies the action of $g$ stabilizes the set of rows $\{U_i\}_{i \in C_\ell}$ for each $\ell$.
    The group $\hat{W}_\ell$ is constructed from $W_\ell$, which is derived from the affine space $A_\ell = \{s_i \mid i \in C_\ell\}$. $\hat{W}_\ell$ is, by construction, the group of sign vectors that stabilizes the class $C_\ell$. Any other element $g\in \{ \pm 1\}^k \setminus \hat W_{\ell}$ that is not in this group will not stabilize $C_\ell$, as $g.A_{\ell} \neq A_{\ell}$, since the identity element will not appear in the coset $g.A_{\ell}$. Since the global automorphism $a$ must stabilize \textit{every} class $C_\ell$ (for $\ell=1, \dots, L$), $a$ must be an element of $\hat{W}_\ell$ for all $\ell$.
    Therefore, $a \in \cap_{\ell=1}^{L}\hat{W}_{\ell}$, which means $A \subseteq \hat{W}$.

    \item \textbf{Proof of $\hat{W} \subseteq A$:}
    Let $a \in \hat{W} = \cap_{\ell=1}^{L}\hat{W}_{\ell}$. By definition, $a \in \hat{W}_\ell$ for every $\ell \in [L]$.
    The fact that $a \in \hat{W}_\ell$ means that $a$ stabilizes the set of rows $\{U_i\}_{i \in C_\ell}$.
    This implies that for each $\ell$, the set $\{a \cdot U_i\}_{i \in C_\ell}$ is a permutation of $\{U_i\}_{i \in C_\ell}$.
    Therefore, for each class $\ell$, there exists a permutation $\tau_\ell: C_\ell \to C_\ell$ such that $a \cdot U_i = U_{\tau_\ell(i)}$ for all $i \in C_\ell$.
    Since the classes $\{C_\ell\}$ form a partition of $[n]$, we can construct a global permutation $\tau \in S_n$ as the disjoint union of these permutations: $\tau = \cup_{\ell=1}^L \tau_\ell$.
    By construction, this $\tau$ satisfies $(\tau, a) \cdot U = U$.
    By \cref{def:sgn-automorphism}, this means $a$ is a sign automorphism.
    Therefore, $a \in A$, which means $\hat{W} \subseteq A$.\qedhere
\end{enumerate}

\end{proof}

It will be useful to define $\pi_\ell:\ZZ_2^k\to\ZZ_2^{K_\ell}$ to be the natural projection from $\hat W_\ell$ to $W_\ell$, and choose for each $\ell$ a \emph{parity-check matrix} $T_\ell$ of $W_\ell$: a linear map $T_\ell:\ZZ_2^{K_\ell}\to\ZZ_2^{K_\ell-\dim W_\ell}$ whose kernel is exactly $W_\ell$, equivalently a matrix whose rows form a basis of the orthogonal complement $W_\ell^{\perp}\subseteq\ZZ_2^{K_\ell}$. \Cref{lem:sign-aut} implies that $s$ is an automorphism of $U$, if and only if it satisfies the homogeneous linear equations
\begin{equation}\label{eq:full_homogenous}
T_\ell \circ \pi_\ell(s)=0, \quad  \forall \ell=1,\ldots,L. 
\end{equation}
\subsubsection{Iterative construction of the sign canonicalization}\label{apdx:proof-iterative}

We now redescribe the step of the canonicalization that iteratively adds linear constraints to determine $s=s(U)$ in an equivariant fashion. The iterations are indexed by the spaces $\hat W_1,\ldots,\hat W_L$. We recall that the basic intuition is that we choose the sign vector $s\in \ZZ_2^k$ so that for $s.U$, each equivalence class $C_{\ell}$ will contain a vector whose entries are all non-negative. However, since the order inside $C_{\ell}$ is arbitrary, we must do this in a way that does not depend on the choice in $C_{\ell}$. To do this, for every given $\hat W_\ell $ we pick an arbitrary $v_\ell \in A_\ell $, and require that $\pi_\ell(s)\oplus v_\ell $ is in the subspace $W_\ell $, or equivalently, that $T_\ell(\pi_\ell(s)\oplus v_\ell)=0 $. This in turn implies that the restriction of $s.U $ to $C_\ell $ will contain a non-negative row. Thus, the condition that $s.U$ will have a non-negative row in each equivalence class translates to the non-homogeneous linear equations
$$T_\ell\circ \pi_\ell(s)=T_\ell (v_\ell), \quad  \forall \ell=1,\ldots, L $$

In general, these linear equations may be linearly dependent, and in this case the non-homogenous equations may not have a solution. To deal with this, we recursively go over the equations, ordered according to the index $\ell$ and the order of the rows inside each $T_\ell$, and remove an equation if it is linearly dependent on the equations that preceded it. This procedure is described in \cref{alg:sign-solve}, and can be performed in $O(k^3)$ time \cite{LeightonMiller79}. As a result, we get a reduced system of non-homogeneous equations
$$\eta_\ell \circ T_\ell\circ \pi_\ell(s)=\eta_\ell \circ T_\ell (v_\ell), \quad  \forall \ell=1,\ldots, L $$
where $\eta_\ell$ is the projection that removes the dependent equations found in the $\ell$-th stage.\footnote{This system is exactly the $E\cdot s=f$ of \cref{alg:sign-solve}: stacking the surviving rows of $\eta_\ell\circ T_\ell\circ\pi_\ell$ as rows of a single matrix $E$ over $\ZZ_2^k$ (zero-padding the columns outside $K_\ell$), and the surviving right-hand sides $\eta_\ell\circ T_\ell(v_\ell)$ as entries of $f$, gives the assembled global system used by the algorithm.} The solution $s=s(U)$ does not in general satisfy $\pi_\ell(s)\in A_\ell$ for every $\ell$: when $\eta_\ell$ drops a row of $T_\ell$ that lies in the row span of earlier rows of $E$, the kept rows define a strict relaxation of $\pi_\ell(s)\in A_\ell$, so $s.U$ need not have a uniformly non-negative row in $C_\ell$ on the coordinates spanned by the dropped rows, since the non-negativity is preserved only on the linearly-independent block of constraints. Nevertheless, $s$ is uniquely determined modulo $\ker(E)=\hat W$, the sign-automorphism group of $U$ (\cref{lem:sign-aut}), and depends equivariantly on $U$. We lex-minimize over the coset $s_0\oplus\hat W$ to fix a canonical representative, leading to a valid canonicalization as we now prove:

\begin{theorem}\label{thm:s-is-sign-canon}
The function $s(U)$ defined above is a sign canonicalization.
\end{theorem}
\begin{proof}
To show $s$ is a sign canonicalization, we need to prove:
\begin{enumerate}
    \item $s$ is permutation invariant.
    \item  $s(t.U) \oplus (t \oplus s(U))$ is a sign automorphism of $U$, for all $U$ and $t\in \ZZ_2^k$.
\end{enumerate}
The permutation invariance follows from the invariant construction of the partitions $C_\ell$, and the fact that the linear equations constructed do not depend on the choice of representative in $C_\ell$.

To prove (2), we note that  the sign canonicalization $s=s(U)$ is defined as the solution to the system of equations over $\ZZ_2$:
$$\eta_\ell \circ T_\ell\circ \pi_\ell(s)=\eta_\ell \circ T_\ell (v_\ell), \quad  \forall \ell=1,\ldots, L $$
and for the sign canonicalization $s'=s(t.U)$, the analogous system reads
$$\eta_\ell \circ T_\ell\circ \pi_\ell(s')=\eta_\ell \circ T_\ell (\pi_\ell(t)\oplus v_\ell), \quad  \forall \ell=1,\ldots, L $$
This implies that 
$$\eta_\ell \circ T_\ell\circ \pi_\ell(s\oplus s' \oplus t)=0 , \quad  \forall \ell=1,\ldots, L .$$
Since this homogenous linear equation is equivalent to the homogenous linear equation \eqref{eq:full_homogenous}, we deduce that  $s \oplus s'\oplus t $ is a sign automorphism, as required.
\end{proof}

\paragraph{Conclusion.}
\Cref{thm:s-is-sign-canon} produces a sign canonicalization $s$. Composing with the lexicographic row sort (\cref{lem:sign-to-canon}) yields the canonicalization $\prism$ of \cref{alg:canon-one-full} (stated in \cref{apdx:canon-details}), completing the proof of \cref{thm:s-s-canon}.

\section{Canonicalization Details}\label{apdx:canon-details}

This section collects the five algorithms used by $\prism$. The master procedure $\prism$ (\cref{alg:canon-one-full}) chains three subroutines: \textsc{Partition} (\cref{alg:abs-value-partition}) groups vertices by their absolute-value signature; \textsc{Refine} (\cref{alg:balanced-refine}) refines that partition via spectral Weisfeiler--Leman to enforce the multiset stability \eqref{eq:stable-multiset}; and \textsc{Solve} (\cref{alg:sign-solve}) assembles a $\ZZ_2$ linear system from the resulting partition and solves it for the canonical sign vector. \cref{alg:fast-sign} is a fast-path shortcut, $\fastsign$, that applies whenever the absolute-value signature is already injective. Complexity bounds are given in \cref{apdx:complexity}.

\paragraph{$\prism$, master procedure.}
\Cref{alg:canon-one-full} ties the three subroutines together: it produces the partition $\{C_1,\dots,C_L\}$, calls \textsc{Solve} to obtain the canonical sign vector $s(U)$, and lex-sorts the rows of $\diag(s(U))\cdot U$ to fix the canonical row ordering.

\begin{algorithm}[t]
\caption{$\prism$ on simple-spectrum eigenvector matrices (base procedure)}\label{alg:canon-one-full}
\begin{algorithmic}[1]
\Require Eigenvector matrix $U \in \mathbb{R}^{n \times k}$ (simple spectrum), precision $p$
\Ensure Canonicalized matrix $\tilde{U} \in \mathbb{R}^{n \times k}$
\Statex \textit{Step 1: Initial partition by absolute-value signature}
\State $\mathcal{P}_0 \gets \textsc{Partition}(U, p)$
  \Comment{\cref{alg:abs-value-partition}}
\Statex \textit{Step 2: $\ZZ_2$-balanced refinement of the partition}
\State $\{C_1,\dots,C_L\} \gets \textsc{Refine}(\mathcal{P}_0, U)$
  \Comment{\cref{alg:balanced-refine}}
\Statex \textit{Step 3: Global sign resolution via $\GF$ linear system}
\State $s(U) \gets \textsc{Solve}(\{C_\ell\}_{\ell=1}^L, U)$
  \Comment{\cref{alg:sign-solve}}
\Statex \textit{Step 4: Canonical row ordering}
\State $\tilde{U} \gets$ rows of $\mathrm{diag}(s(U)) \cdot U$ sorted lexicographically
\State \Return $\tilde{U}$
\Statex \textit{Fast path: when the signature $\sigma$ produced in Step~1 is injective, Steps 2 and 3 reduce to a single per-column scan, and the shortcut is $\fastsign$ (\cref{alg:fast-sign}).}
\end{algorithmic}
\end{algorithm}

\paragraph{$\fastsign$, shortcut for injective signatures.}
When \textsc{Partition} already separates every vertex (i.e.\ $\sigma$ is injective), the partition is trivial, \textsc{Refine} is a no-op, and \textsc{Solve} reduces to a per-column sign scan. \Cref{alg:fast-sign} runs this shortcut directly.

\begin{algorithm}[t]
\caption{$\fastsign$, Fast-path sign canonicalization (injective $\sigma$)}\label{alg:fast-sign}
\begin{algorithmic}[1]
\Require Eigenvector matrix $U \in \mathbb{R}^{n \times k}$ (simple spectrum), precision $p$
\Ensure Canonicalized matrix $\tilde{U} \in \mathbb{R}^{n \times k}$; valid whenever the absolute-value signature $\sigma$ is injective
\Statex \textit{Step 1: Absolute-value signature}
\For{each node $v \in [n]$}
    \State $\sigma(v) \gets \bigl(|U_{v,1}|,\dots,|U_{v,k}|\bigr)$ rounded to precision $p$
\EndFor
\Statex \textit{Step 2: Canonical vertex ordering}
\State $\pi \gets \mathrm{lexicographic~argsort}(\{\sigma(v)\}_{v \in [n]})$
  \Comment{$O(nk\log n)$}
\Statex \textit{Step 3: Sign resolution}
\For{$j = 1, \dots, k$}
    \State $v^* \gets$ first vertex in $\pi$ with $|U_{v^*,j}| > 0$
    \If{$U_{v^*,j} < 0$}
        \State $U_{\cdot,j} \gets -U_{\cdot,j}$
          \Comment{flip column $j$}
    \EndIf
\EndFor
\State $\tilde{U} \gets$ rows of $U$ sorted lexicographically
  \Comment{canonical row ordering}
\State \Return $\tilde{U}$
\end{algorithmic}
\end{algorithm}
\paragraph{\textsc{Partition}, subroutine 1.}
\Cref{alg:abs-value-partition} computes the absolute-value signature $\sigma(v)=(|U_{v,1}|,\dots,|U_{v,k}|)$ for every vertex and groups vertices with identical signatures into classes; the classes are then ordered lexicographically by $\sigma$.

\begin{algorithm}[t]
\caption{\textsc{Partition}, Initial partition from absolute-value signature}\label{alg:abs-value-partition}
\begin{algorithmic}[1]
\Require Eigenvector matrix $U \in \mathbb{R}^{n \times k}$, precision $p$
\Ensure Partition $\mathcal{P}_0 = (P_1,\dots,P_r)$ of $[n]$ with lexicographically ordered signatures
\For{each node $v \in [n]$}
    \State $\sigma(v) \gets \bigl(|U_{v,1}|,\dots,|U_{v,k}|\bigr)$ rounded to precision $p$
\EndFor
\State Group nodes into classes $P_i \gets \{v \in [n] : \sigma(v) = \sigma_i\}$
  \Comment{one class per distinct signature value}
\State Order the classes $(P_1,\dots,P_r)$ by lexicographic $\sigma_i$
\State \Return $\mathcal{P}_0 = (P_1,\dots,P_r)$
\end{algorithmic}
\end{algorithm}

\paragraph{\textsc{Refine}, subroutine 2.}
\Cref{alg:balanced-refine} refines $\mathcal{P}_0$ by spectral Weisfeiler--Leman: each vertex's color is updated using a multiset of edge features $U_v\odot U_u$, and iteration continues until the partition stabilizes. The fixed point satisfies the multiset stability \eqref{eq:stable-multiset} on which the proof of \cref{thm:s-s-canon} relies. \cref{rem:balanced-refine-realization} below gives an equivalent $\ZZ_2$ parity-check realization used in our implementation.

\begin{algorithm}[t]
\caption{\textsc{Refine}, spectral Weisfeiler--Leman refinement}\label{alg:balanced-refine}
\begin{algorithmic}[1]
\Require Partition $\mathcal{P}_0$ from \textsc{Partition} (\cref{alg:abs-value-partition}), eigenvector matrix $U\in\RR^{n\times k}$
\Ensure Refined partition $\{C_1,\dots,C_L\}$ on which the multiset stability \eqref{eq:stable-multiset} holds
\State $h_v^{(0)} \gets U_v \odot U_v$ for each $v \in [n]$
  \Comment{initialize colors with absolute-value signature}
\State $t \gets 0$
\Repeat
    \For{each $v \in [n]$}
        \State $h_v^{(t+1)} \gets \textsc{Hash}\Big(h_v^{(t)},\,\big\{\!\!\big\{(h_u^{(t)},\, U_v \odot U_u) : u \in [n]\big\}\!\!\big\}\Big)$
          \Comment{multiset of edge features}
    \EndFor
    \State $t \gets t+1$
\Until{the partition induced by $\{h_v^{(t)}\}$ equals the partition induced by $\{h_v^{(t-1)}\}$}
\State $\{C_1,\dots,C_L\} \gets$ classes of $\equiv_h$, ordered lexicographically by representative color
\State \Return $\{C_1,\dots,C_L\}$
\Statex \textit{Implementation note: a concrete realization with sufficient fixed-point partition uses, for each class $C$, the $\ZZ_2$-independent sign-bit columns $Q\subseteq[k]$ and splits $C$ whenever some product $p_S(v)=\prod_{\ell\in S}U_{v,\ell}$ ($S\subseteq Q$) has imbalanced $\pm$ counts. See \cref{rem:balanced-refine-realization} below.}
\end{algorithmic}
\end{algorithm}

\begin{remark}[$\ZZ_2$ realization of \textsc{Refine}]\label{rem:balanced-refine-realization}
The companion implementation realizes the spectral-WL fixed point above by a $\ZZ_2$ parity-check scheme. For each working class $C$, maintain the set $Q\subseteq[k]$ of sign-bit columns of $U|_C$ that are $\ZZ_2$-independent, and for every newly added column $j\in Q$ test all subsets $S\subseteq Q$ containing $j$: if the product $p_S(v)=\prod_{\ell\in S}U_{v,\ell}$ has unequal counts of $+x$ and $-x$ over $v\in C$ for some $|x|>0$, split $C$ along $\operatorname{round}(p_S)$ and recurse. This produces the same partition as the iteration in \cref{alg:balanced-refine} but avoids explicit color hashing.
\end{remark}

\paragraph{\textsc{Solve}, subroutine 3.}
Given the balanced partition $\{C_1,\dots,C_L\}$, \cref{alg:sign-solve} assembles the $\ZZ_2$ linear system $E\cdot s=f$ described in the proof, dropping linearly dependent rows on the fly, and returns the lex-minimal element of the solution coset $s_0\oplus\ker(E)$ together with the sign-automorphism group $\hat W=\ker(E)$.
\begin{algorithm}[!ht]
\caption{\textsc{Solve}, Solve $\GF$-linear system for canonical signs}\label{alg:sign-solve}
\begin{algorithmic}[1]
\Require Balanced partition $\{C_1,\dots,C_L\}$ from \cref{alg:balanced-refine}, eigenvector matrix $U$
\Ensure Canonical sign vector $s(U) \in \mathbb{Z}_2^k$; sign automorphism group $\hat{W} \leq \GF^k$
\State $E \gets$ empty matrix over $\GF$ ($k$ columns);\;
  $f \gets$ empty vector
\For{$\ell = 1, \dots, L$}
    \State $K_\ell \gets \{j \in [k] : |U_{v,j}| > 0,\; v \in C_\ell\}$
    \If{$K_\ell = \emptyset$} \textbf{continue} \EndIf
    \State $A_\ell \gets \bigl\{\mathbf{1}[U_{v,K_\ell} < 0] : v \in C_\ell\bigr\}$;
      \quad fix $\mathbf{a}_0 \in A_\ell$
    \State $W_\ell \gets \operatorname{rowspace}_{\GF}(\{\mathbf{a} \oplus \mathbf{a}_0 : \mathbf{a} \in A_\ell\})$
    \State $T_\ell \gets$ parity-check matrix of $W_\ell$
      \Comment{$\ker(T_\ell) = W_\ell$}
    \State $\mathbf{b}_\ell \gets T_\ell \cdot \mathbf{a}_0$
      \Comment{RHS; independent of choice of $\mathbf{a}_0 \in A_\ell$}
    \For{each row $\mathbf{r}$ of $T_\ell$ with entry $b$ of $\mathbf{b}_\ell$}
        \State Embed: $\hat{\mathbf{r}}_j \gets \mathbf{r}_j$ if $j \in K_\ell$, else $0$
        \If{$\hat{\mathbf{r}} \notin \operatorname{rowspace}(E)$}
            Append $\hat{\mathbf{r}}$ to $E$;\; append $b$ to $f$
        \EndIf
    \EndFor
\EndFor
\State Solve $E \cdot s = f$ over $\GF$;\; $s_0 \gets$ particular solution
\State $s(U) \gets \operatorname{lex\text{-}min}(s_0 \oplus \ker(E))$
  \Comment{$\ker(E) = \hat{W}$}
\State \Return $s(U)$, $\hat{W} = \ker(E)$
\end{algorithmic}
\end{algorithm}

\subsection{Complexity}\label{apdx:complexity}
The total complexity of simple-spectrum canonicalization $\prism$ is $O(nk\log n + k^3 + nk^2)$.
We account for each stage separately.

\textbf{Partition.} Computing the absolute-value signature $\sigma(v) = (|U_{v,1}|, \ldots, |U_{v,k}|)$ for all $n$ vertices costs $O(nk)$. Sorting the $n$ signatures lexicographically costs $O(nk\log n)$.  The stage total is $O(nk\log n)$.

\textbf{Refine and Solve.} The \cref{alg:sign-solve} processes at most $n$ balanced blocks (one per node in the worst case). For each block it builds the parity-check matrix $T_\ell$ and embeds its rows into the global system $E$, which has $k$ columns. Each block contributes at most $k$ rows (by linear dependence), and checking independence against the existing rowspace of $E$ costs $O(k)$ per row, giving $O(k^2)$ per block and $O(nk^2)$ total across all $n$ blocks.

The $O(k^3)$ term comes from two operations on the assembled $k$-column $\ZZ_2$ system:
\begin{enumerate}
    \item Solving $E \cdot s = f$ over $\GF$ via Gaussian elimination on a matrix with at most $k$ rows and $k$ columns: $O(k^3)$.
    \item Lex-minimizing over the coset $s_0 \oplus \ker(E)$: finding the lex-min element of an affine subspace of $\ZZ_2^k$ by a greedy pass over the basis of $\ker(E)$ costs $O(k^2)$, dominated by the $O(k^3)$ solve.
\end{enumerate}

Combining, the stage total is $O(nk^2 + k^3)$.

\textbf{Total.}  Summing over all stages, the overall complexity is $O(nk\log n + k^3 + nk^2)$, matching the bound stated in the main text.

\paragraph{Complexity of $\fastsign$}
When the absolute-value signature $\sigma$ is injective, the $\fastsign$ fast path (\cref{alg:fast-sign}) suffices: reading $nk$ absolute values is $O(nk)$; the lexicographic argsort costs $O(nk\log n)$; the per-column sign scan is $O(nk)$. The total complexity is therefore $O(nk\log n)$.

\paragraph{The cubic factor is a small constant in practice.}
The $O(k^3)$ term comes from Gaussian elimination on the $\ZZ_2$ parity-check matrix $E$ assembled by \textsc{Solve}. The row count of $E$ is upper-bounded by $r = n - L$, where $L$ is the number of equivalence classes returned by \textsc{Partition}. To see the bound, recall that a non-singleton class $C_\ell$ of size $m$ contributes to $E$ only constraints spanned by the pairwise sign-bit differences $\{s_v \oplus s_w : v,w \in C_\ell\}$, since the absolute positions of the per-vertex sign-bit vectors $\{s_v\}_{v\in C_\ell}\subseteq\ZZ_2^{K_\ell}$ flip under a column-wise sign change of $U$ while their differences do not. Fix any anchor $v_0\in C_\ell$: the $m-1$ anchored differences $\{s_v \oplus s_{v_0}\}_{v \in C_\ell \setminus\{v_0\}}$ already span this hull, since every other difference telescopes as $s_v\oplus s_w = (s_v \oplus s_{v_0}) \oplus (s_w \oplus s_{v_0})$. Hence $C_\ell$ contributes at most $m-1$ independent relative-sign relations, as singleton classes ($m=1$) contribute none, summing over all classes gives $\sum_\ell (m_\ell - 1) = n - L$. We measured $r$ on the entire validation splits of ZINC \cite{irwin2012zinc} and the eleven \textsc{OGB-Mol} datasets \cite{hu2020open} (about $62$k graphs total), using full eigendecompositions ($k = n$) of the normalized Laplacian with tolerance $10^{-6}$, and we report the statistics in \cref{tab:rbound-stats}. The median $r$ is $1$--$3$ across every dataset, and $r \leq 4$ on $64$--$88\%$ of graphs, so the cubic stage costs at most $4^3 = 64$ operations on the median graph, a constant independent of $n$ and $k$. On $14$--$45\%$ of graphs the partition is already injective ($r=0$) and PRiSM degenerates to $\fastsign$. On a further $3$--$17\%$ the partition produces exactly one non-singleton block. In that case the linear system collapses to a single sign convention, for instance forcing the lex-smallest row of the block to be non-negative, so no $\ZZ_2$ solve is performed.

\begin{table}[ht]
\centering
\setlength{\abovecaptionskip}{3pt}
\setlength{\belowcaptionskip}{1pt}
\caption{Distribution of the $\ZZ_2$ row-count bound $r = n - L$ on entire validation splits, full eigendecomposition ($k = n$) of the normalized Laplacian, tolerance $10^{-6}$. Columns are the number of validation graphs $|\textrm{val}|$, the mean number of nodes per graph $\bar n$, the fraction $r{=}0$ on which $\fastsign$ applies, the fraction ``$1$ block'' with exactly one non-singleton block (no $\ZZ_2$ solve needed), and the median $r$ together with the fraction satisfying $r \leq 4$ which summarise the residual cubic cost ($\leq 4^3 = 64$).}
\label{tab:rbound-stats}
\setlength{\tabcolsep}{4pt}
\renewcommand{\arraystretch}{0.95}
\footnotesize
\begin{tabular}{lrrrrrrr}
\toprule
\textbf{Dataset} & $|\textrm{val}|$ & $\bar n$ & $r{=}0$ & 1 block & median $r$ & $r \leq 4$ & max $r$ \\
\midrule
ZINC-12K        &  1{,}000 & 23.1 & 22.8\% & 10.9\% & 2.0 & 88.0\% & 11 \\
\textsc{MolHIV}      &  4{,}113 & 27.8 & 21.2\% &  6.4\% & 3.0 & 63.9\% & 53 \\
\textsc{MolTox21}    &     783 & 26.8 & 17.5\% &  8.6\% & 3.0 & 67.4\% & 32 \\
\textsc{MolClinTox}  &     148 & 32.8 & 27.0\% &  9.5\% & 2.0 & 70.3\% & 37 \\
\textsc{MolBBBP}     &     204 & 33.2 & 14.2\% & 17.2\% & 2.0 & 81.4\% & 32 \\
\textsc{MolBACE}     &     151 & 37.2 &  3.3\% & 10.6\% & 3.0 & 81.5\% &  6 \\
\textsc{MolESOL}     &     113 & 17.7 & 38.1\% & 14.2\% & 1.0 & 77.0\% & 15 \\
\textsc{MolFreeSolv} &      64 & 12.0 & 45.3\% &  4.7\% & 1.5 & 78.1\% & 12 \\
\textsc{MolLipo}     &     420 & 27.3 & 14.8\% &  7.9\% & 3.0 & 74.3\% & 25 \\
\textsc{MolMUV}      &  9{,}309 & 25.3 & 12.8\% &  3.3\% & 3.0 & 73.3\% & 17 \\
\textsc{MolToxCast}  &     858 & 26.2 & 20.5\% & 10.7\% & 2.5 & 69.4\% & 37 \\
\textsc{MolPCBA}     & 43{,}793 & 27.1 & 13.9\% &  4.7\% & 3.0 & 74.3\% & 71 \\
\bottomrule
\end{tabular}
\end{table}

\section{Proof of Universal Approximation (\cref{thm:uni-bdd-ev})}\label{app:universality}

\universalapprox*

Before applying the proof, we add some clarifications on the meaning of a universal set model. Recall that a graph $g\in \mathcal{K}(N,M,\delta)$ has simple spectrum, $n\leq N$ nodes and edge weights bounded by $M$. Its corresponding eigendecomposition is comprised by an $n\times n$ matrix of eigenvectors, and $n$ eigenvalues. We note that $(U_g,\vec{\lambda}_g)$ reside in a bounded set $\mathcal{C}_n\subseteq \RR^{n\times n}\oplus \RR^n$: the eigenvector matrix is orthogonal so its rows have unit norm, and the eigenvalues satisfy
$$\|\vec{\lambda}_g\|_2=\|A_g\|_F\leq M\,n,$$
where $A_g$ is the adjacency matrix of $g$. 

Regarding the eigenvalue gap constraint, the constraint $|\lambda_i-\lambda_j|\geq\delta$ ensures that $\mathcal{C}_n$ stays bounded away from the repeating-eigenvalue region, where the simple-spectrum hypothesis would fail and ensures the domain is compact, which must hold for the universality results to be correct.

We now show explicitly that $\mathcal{C}_n$ is a closed set; since we have already shown previously it is bounded, consequently it is compact. \citet{Weyl1912DasAV} proves that for any pair of symmetric matrices $A,B \in \RR^{n\times n}$ with eigenvalues $\lambda_1(A) \geq \cdots \geq \lambda_n(A)$ and $\lambda_1(B) \geq \cdots \geq \lambda_n(B)$, one has $$|\lambda_i(A)-\lambda_i(B)| \leq \|A-B\|_{\mathrm{op}} \leq \|A-B\|_F,$$ for every $i$ , so the ordered eigenvalue map $A \mapsto \vec\lambda(A)$ is $1$-Lipschitz, hence continuous. 

Take a sequence $(U_k,\vec{\lambda}_k) \in \mathcal{C}_n$ converging to some $(U_*,\vec{\lambda}_*)$. The associated adjacency matrices $A_k = U_k \diag(\vec{\lambda}_k) U_k^\top$ converge in Frobenius norm to $A_* = U_* \diag(\vec{\lambda}_*) U_*^\top$. By Weyl, $\vec\lambda(A_k) \to \vec\lambda(A_*)$, and by uniqueness of the eigendecomposition we identify $\vec\lambda(A_k) = \vec{\lambda}_k$ and $\vec\lambda(A_*)=\vec{\lambda}_*$. Hence the closed conditions $\{\|A\|_\infty \leq M\}$ and $\{|\lambda_i-\lambda_j| \geq \delta\}$ both pass to the limit point, so $A_*$ is the adjacency of a graph in $\mathcal{K}(N,M,\delta)$ and $(U_*,\vec{\lambda}_*) \in \mathcal{C}_n$. Combining with the boundedness above, $\mathcal{C}_n$ is compact, and so is the finite union $\mathcal{C}_N=\cup_{n\leq N}\mathcal{C}_n$, as we set out to prove. 

Accordingly, when we discuss universal set models, we mean model classes which can uniformly approximate to arbitrary accuracy on $\mathcal{C}_N$ all functions which are invariant to permutation of the rows of $U$. Examples of such models are DeepSets \cite{deepsets} or Transformers \cite{vaswani2017attention,kim2021transformers}, where care needs to be taken to adapt these models to process the additional eigenvalue input. For example, a DeepSets model which would be appropriate for this theorem is 
$$\psi(U,\vec{\lambda})=\mathrm{MLP}^{(2)}\left( \vec{\lambda},\sum_{i=1}^n \mathrm{MLP}^{(1)} (U(i,:)) \right) .$$
To make sure the model is defined and universal for all $n\leq N$ simultaneously we can apply padding to make sure the input is always of the same dimension. 
\begin{proof}[Proof of \cref{thm:uni-bdd-ev}]
 on $\mathcal{C}_N$ define the function $\tilde f(U,\vec{\lambda})=f(U\diag(\vec{\lambda})U^T) $. Then $\tilde f$ is
 \begin{enumerate}
     \item \textbf{Continuous} as a composition of continuous functions, 
     \item \textbf{Invariant} For any graph $g\in \mathcal{K}(N,M,\delta)$ with $n\leq N$ nodes, any $n\times n$ permutation matrix $P$ and diagonal matrix $S$ with diagonal entries in $\{-1,1\}$, we have that 
     $$\tilde f(PVS,\vec{\lambda})=f(PVS\diag(\vec{\lambda})SV^TP^T)\stackrel{(*)}{=}f(U\diag(\vec{\lambda})U^T)=\tilde f(U,\vec{\lambda}),$$
 where $(*)$ is because $f$ is isomorphism-invariant, diagonal matrices commute, and $S^2=I_n$. 
 \item \textbf{Consistent} By which we mean that $\tilde{f}(U_g,\vec{\lambda}_g)=f(g) $ for any simple-spectrum graph $g$.
 \end{enumerate}
For given $\epsilon>0$, since $\tilde f$ is permutation invariant, we can approximate $\tilde f$ with a universal set model $\psi$ to $\epsilon$ accuracy on $\mathcal{C}_N$. It follows that for all $g\in \mathcal{K}(N,M,\delta)$

\begin{align*}
|f(g)-\psi \circ \prism(U_g,\vec{\lambda}_g)   |&\stackrel{\text{consistency}}{=} |\tilde{f}(U_g,\vec{\lambda}_g)-\psi \circ \prism(U_g,\vec{\lambda}_g)   |\\
&\stackrel{\text{invariance}}{=}|\tilde{f}\circ \prism (U_g,\vec{\lambda}_g)-\psi \circ \prism(U_g,\vec{\lambda}_g)   |<\epsilon.
\end{align*}
\end{proof}

\begin{remark}
There is a fundamental trade-off between continuity and completeness on continuous-edge-weight simple-spectrum graphs. Existing approaches sacrifice completeness in different ways: SignNet \cite{lim2023sign} retains continuity via sign-equivariant architectures but collapses distinct eigendecompositions, trading injectivity for continuity. MAP \cite{ma2023laplacian} and OAP \cite{ma2024a} are heuristic and fail on both axes, being neither injective nor continuous. Our scheme achieves completeness on simple-spectrum graphs but sacrifices continuity on $\mathcal{K}(N,M,\delta)$.
\end{remark}

\section{Experimental Details}\label{app:exp-details}

\paragraph{Experimental protocol.} All eigendecompositions are computed using \texttt{torch.linalg.eigh} in float64 precision. Eigenvalue multiplicity statistics for our molecular benchmarks are reported in \cref{app:exp-multiplicity}. Two eigenvalues are considered equal when their absolute difference falls below $10^{-8}$. This threshold is shared across all canonicalization methods. We verified that varying this threshold between $10^{-6}$ and $10^{-10}$ does not alter the reported scores. All methods use an identical GNN backbone with a DeepSets \cite{deepsets} eigenvector embedding and matched parameter budget ($\sim$12K parameters), isolating the effect of the canonicalization method.

\paragraph{$\prism$ on graphs with higher eigenvalue multiplicities.}
The BREC and molecular experiments apply the hybrid extension of $\prism$ described in \cref{alg:babai-hybrid}. This extension is heuristic at multiplicities $m_i > 1$ and is not complete on $\mathcal{G}_n^M$ for $M > 1$; on a simple-spectrum graph every eigenspace has multiplicity~1, so the extension reduces to the base $\prism$ of \cref{alg:canon-one-full} and inherits its completeness. The method proceeds in two stages:

\textit{Stage~1: Canonical vertex ordering via the norm-value signature.} We compute the eigendecomposition of the Laplacian, grouping eigenvalues into $K$ eigenspaces of multiplicity $m_1, \ldots, m_K$ (within tolerance $\tau$). Let $\pi_i$ denote orthogonal projection onto the $i$-th eigenspace. The norm-value signature $\sigma(v) = (\|\pi_1 v\|, \ldots, \|\pi_K v\|)$ is a permutation-equivariant function of the vertex $v$: relabeling the graph permutes $\sigma$ rows accordingly, so the lexicographic argsort of $\sigma$ produces a canonical vertex ordering up to ties between equal-signature vertices. Ties (when present) are broken sequentially in Stage~2 using canonical eigenvector values from already-processed eigenspaces.

\textit{Stage~2a: Joint sign canonicalization on the simple-spectrum block.} Let $\mathcal{I}_1 = \{i : m_i = 1\}$ and let $U^{(1)} \in \RR^{n \times k}$ be the matrix formed by stacking the eigenvectors of all $m=1$ eigenspaces, where $k = |\mathcal{I}_1|$. We run \textsc{Refine} (\cref{alg:balanced-refine}) on $U^{(1)}$ starting from the $\sigma$-induced partition $\mathcal{P}_0$, obtaining a refined partition $\{C_1, \dots, C_L\}$. We then run \textsc{Solve} (\cref{alg:sign-solve}) on this partition to obtain $s \in \{\pm 1\}^k$, and replace $U^{(1)}$ by $U^{(1)}\,\diag(s)$, writing the sign corrections back into the corresponding $U_i$ for $i \in \mathcal{I}_1$. The refined partition $\{C_\ell\}$ together with the canonical values of $U^{(1)}$ are also used to break $\sigma$-ties in the vertex ordering before Stage~2b.

\textit{Stage~2b: Per-eigenspace $O(m_i)$ resolution for $m_i \geq 2$, in increasing multiplicity.} For each remaining eigenspace the eigenvectors are defined up to an $O(m_i)$ rotation, which we resolve according to the multiplicity:
\begin{itemize}
    \item $m_i = 2$ (rotation + reflection): Align the first non-trivial pivot to the positive $x$-axis via an $SO(2)$ rotation; resolve the residual $\ZZ_2$ reflection from the next pivot.
    \item $m_i = 3$ (Householder + rotation + reflection): Householder-reflect the first pivot to the $x$-axis; $x$-axis-rotate the second pivot's $yz$-component to the positive $y$-axis; flip column sign for the residual $\ZZ_2$.
    \item $m_i \geq 4$ (generalized QR): Select $m_i$ pivot rows via Gram-Schmidt in canonical order (preferring unique-$\sigma$ vertices); apply QR with the Stewart sign convention (positive diagonal of $R$).
\end{itemize}
After each higher-multiplicity eigenspace is canonicalized, when $\sigma$-ties remain we refine the vertex ordering using the canonical values produced so far before moving on to the next eigenspace. The resulting eigenvectors form the input features to the downstream model. \cref{alg:babai-hybrid} collects the full procedure in pseudocode.

\begin{algorithm}[t]
\caption{$\prism$ on graphs with bounded eigenvalue multiplicity (hybrid extension of \cref{alg:canon-one-full})}\label{alg:babai-hybrid}
\begin{algorithmic}[1]
\Require Adjacency matrix $A \in \{0,1\}^{n \times n}$, multiplicity tolerance $\tau$
\Ensure Canonical eigenvector matrix $[U_1 \,\|\, \cdots \,\|\, U_K]$
\Statex \textit{Stage 1: Canonical vertex ordering via norm-value signature.}
\State $L \gets D - A$;\quad eigendecompose $L$ into eigenspaces $\{(\lambda_i, U_i)\}_{i=1}^{K}$ with multiplicities $m_i$ (grouping within tolerance $\tau$)
\State $\sigma(v) \gets (\|\pi_1 v\|, \ldots, \|\pi_K v\|)$ for each vertex $v$
  \Comment{norm-value signature; $\pi_i$ is projection onto $U_i$}
\State $\mathrm{order} \gets \mathrm{lex\text{-}argsort}(\sigma)$
  \Comment{canonical vertex ordering up to $\sigma$-ties}
\Statex \textit{Stage 2a: Joint sign canonicalization on the simple-spectrum block.}
\State $\mathcal{I}_1 \gets \{i : m_i = 1\}$;\quad $U^{(1)} \gets [U_i]_{i \in \mathcal{I}_1} \in \RR^{n \times k}$
  \Comment{$k = |\mathcal{I}_1|$, the number of $m=1$ eigenspaces}
\State $\{C_1,\dots,C_L\} \gets \textsc{Refine}(\mathcal{P}_0, U^{(1)})$, with $\mathcal{P}_0$ induced by $\sigma$
  \Comment{spectral WL refinement on the $m=1$ block; \cref{alg:balanced-refine}}
\State $s \gets \textsc{Solve}(\{C_\ell\}_{\ell=1}^L, U^{(1)}) \in \{\pm 1\}^k$
  \Comment{$\GF$ system over the $m=1$ columns; \cref{alg:sign-solve}}
\State $U^{(1)} \gets U^{(1)} \, \diag(s)$;\quad write $\diag(s)$ back into the corresponding $U_i$ for $i \in \mathcal{I}_1$
\State Refine $\mathrm{order}$ using $\{C_\ell\}$ together with the canonical values of $U^{(1)}$
  \Comment{spectral WL + absolute-norm signature ordering}
\Statex \textit{Stage 2b: Per-eigenspace $O(m_i)$ resolution for $m_i \geq 2$, in increasing multiplicity.}
\For{eigenspace $i$ with $m_i \geq 2$, taken in order $m_i = 2, 3, \ldots$}
    \If{$\sigma$-ties remain \textbf{and} prior eigenspaces canonicalized}
        \State Refine $\mathrm{order}$ using canonical values from already-processed eigenspaces
    \EndIf
    \If{$m_i = 2$}
        \State Align the first non-trivial pivot to the positive $x$-axis via an $SO(2)$ rotation; resolve the residual $\ZZ_2$ reflection from the next pivot
    \ElsIf{$m_i = 3$}
        \State Householder-reflect the first pivot to the $x$-axis; $x$-axis-rotate the second pivot's $yz$-component to the positive $y$-axis; flip column sign for the residual $\ZZ_2$
    \Else\Comment{$m_i \geq 4$, generalized QR}
        \State Pick $m_i$ pivot rows via Gram--Schmidt in $\mathrm{order}$ (preferring unique-$\sigma$ vertices); apply QR with the Stewart sign convention (positive diagonal of $R$)
    \EndIf
\EndFor
\State \Return $[U_1 \,\|\, \cdots \,\|\, U_K]$
\end{algorithmic}
\end{algorithm}

\subsection{ZINC 12K Dataset and Protocol}
\label{app:exp-zinc}

ZINC 12K \cite{dwivedi2023benchmarking} is a subset of the ZINC molecular database consisting of approximately $12{,}000$ drug-like molecules. The task is regression on constrained solubility (logP score), evaluated by mean absolute error (MAE). Molecules have a mean of approximately $23$ nodes and $25$ edges; the graph structure is derived from the molecular bond graph. The dataset is split into $10{,}000$ / $1{,}000$ / $1{,}000$ train/val/test molecules following the standard split of \citet{dwivedi2023benchmarking}.

We evaluate two backbone architectures. The GatedGCN \cite{bresson2017residual} backbone uses gated graph convolution with residual connections and edge features; the PNA \cite{corso2020principal} backbone uses principal neighbourhood aggregation with degree-based scalers. Both backbones process positional encodings produced by the canonicalization and fuse them with node features additively, following the GPS framework \cite{rampavsek2022recipe}. The sign-equivariant embedding network (sign-inv-net) uses a DeepSets architecture \cite{deepsets} applied independently to the positive and negative parts of each eigenvector column, with their outputs summed to produce a sign-invariant embedding.

\paragraph{Hyperparameter configuration.}

\begin{table}[h]
\centering
\caption{Hyperparameter configuration for ZINC 12K experiments. Values marked (default) were not exhaustively tuned and follow the conventions of \citet{dwivedi2023benchmarking}.}
\label{tab:hp-zinc}
\setlength{\tabcolsep}{6pt}
\begin{tabular}{lcc}
\toprule
\textbf{Hyperparameter} & \textbf{GatedGCN} & \textbf{PNA} \\
\midrule
Learning rate            & $1\times10^{-3}$  & $5\times10^{-4}$ \\
Batch size               & $128$             & $128$ \\
Hidden dimension         & $128$             & $80$ \\
Number of layers         & $4$               & $4$ \\
Dropout                  & $0.0$             & $0.05$ \\
Optimizer                & AdamW             & AdamW \\
LR scheduler             & OneCycleLR        & OneCycleLR \\
Training epochs          & $200$             & $200$ \\
Weight decay             & $0.0$ (default)   & $0.0$ (default) \\
Number of seeds          & $4$               & $4$ \\
Parameter budget ($\approx$) & $472$K        & $473$K \\
\bottomrule
\end{tabular}
\end{table}

\paragraph{Training protocol.} Models are trained for $200$ epochs with the OneCycleLR scheduler and early stopping based on validation MAE with patience $50$. Performance is reported as mean $\pm$ std over four independent seeds with different random initializations. The evaluation metric is MAE on the held-out test set (lower is better). Eigendecompositions use all $n-1$ non-trivial eigenvectors of the normalized graph Laplacian, computed in float64 precision via \texttt{torch.linalg.eigh}.

\subsection{OGB Molecular Datasets and Protocol}
\label{app:exp-ogb}

We evaluate on four datasets from the Open Graph Benchmark (OGB) molecular suite \cite{hu2020open}. All are drawn from MoleculeNet and use scaffold-based train/val/test splits provided by OGB.

\begin{itemize}
  \item \textbf{MolTox21}: 7{,}831 molecules (6{,}264 / 783 / 784 train/val/test). Multi-task binary classification across 12 toxicology assays. Mean $\approx 18$ nodes, $\approx 19$ edges per molecule. Metric: mean ROC-AUC (higher is better).
  \item \textbf{MolHIV}: 41{,}127 molecules (32{,}901 / 4{,}113 / 4{,}113 train/val/test). Single binary classification task (HIV inhibition). Mean $\approx 25$ nodes, $\approx 27$ edges. Metric: ROC-AUC.
  \item \textbf{MolClinTox}: 1{,}477 molecules (1{,}181 / 148 / 148 train/val/test). Binary classification across 2 clinical toxicology assays. Mean $\approx 26$ nodes, $\approx 28$ edges. Metric: mean ROC-AUC.
  \item \textbf{MolBBBP}: 2{,}039 molecules (1{,}631 / 204 / 204 train/val/test). Binary classification for blood-brain barrier penetration. Mean $\approx 24$ nodes, $\approx 26$ edges. Metric: ROC-AUC. \textsc{MolBBBP} appears in the Transformer-backbone main results (\cref{tab:transformer-results}) using the no-rank-embedding $\prism$ configuration (sequential refinement only); the RoPE-augmented configuration degraded on this small dataset and is reported in the supplementary archive at \texttt{exps/}. We additionally report eigenvalue multiplicity statistics in \cref{tab:multiplicity-stats}.
\end{itemize}

The GNN experiments on MolTox21 use GatedGCN \cite{bresson2017residual} and PNA \cite{corso2020principal} backbones with the same GPS-style fusion used for ZINC (\cref{app:exp-zinc}).

\paragraph{Hyperparameter configuration.}

\begin{table}[h]
\centering
\caption{Hyperparameter configuration for OGB MolTox21 GNN experiments. Values marked (default) follow OGB benchmark conventions.}
\label{tab:hp-ogb-gnn}
\setlength{\tabcolsep}{6pt}
\begin{tabular}{lcc}
\toprule
\textbf{Hyperparameter} & \textbf{GatedGCN} & \textbf{PNA} \\
\midrule
Learning rate            & $5\times10^{-4}$  & $5\times10^{-4}$ \\
Batch size               & $512$             & $512$ \\
Hidden dimension         & $128$             & $128$ \\
Number of layers         & $4$               & $4$ \\
Dropout                  & $0.1$             & $0.1$ \\
Optimizer                & AdamW             & AdamW \\
LR scheduler             & ReduceLROnPlateau & ReduceLROnPlateau \\
Training epochs          & $100$ (default)   & $100$ (default) \\
Weight decay             & $0.0$ (default)   & $0.0$ (default) \\
Number of seeds          & $4$               & $4$ \\
Parameter budget ($\approx$) & $1.54$M       & n/a \\
\bottomrule
\end{tabular}
\end{table}

\paragraph{Training protocol.} Models are trained with early stopping based on validation ROC-AUC with patience $20$. Performance is reported as mean $\pm$ std over four seeds. The sign-inv-net uses the same DeepSets \cite{deepsets} architecture as in the ZINC experiments. SignNet, MAP, and OAP baselines were not evaluated on the PNA backbone in our pipeline, and these entries are marked ``n/a'' in \cref{tab:molecular-results}.

\subsection{Transformer Backbone Configuration}
\label{app:exp-transformer}

The Transformer backbone experiments on OGB molecular datasets and \textsc{Alchemy} (\cref{subsec:transformer-ogb}) use the graph Transformer architecture of \citet{kim2021transformers} with Laplacian positional encodings. Canonical eigenvectors (CanonLPE) replace the standard random-sign Laplacian PE; the canonical vertex rank produced by \cref{alg:babai-hybrid} is then encoded into token positions in one of three ways, with the choice fixed per dataset: (i)~Rotary Position Embedding (RoPE) \cite{su2024roformer} for relative positions on \textsc{MolHIV} and \textsc{Alchemy}; (ii)~sinusoidal absolute encoding combined with sequential ordering refinement on \textsc{MolClinTox}; (iii)~sequential ordering refinement only, with no rank embedding, on \textsc{MolBBBP}. ``Sequential refinement'' uses canonical eigenvector values from already-processed eigenspaces to break ties in the vertex ordering for subsequent eigenspaces. The Transformer encoder uses multi-head self-attention with pre-norm layer normalization.

\paragraph{Hyperparameter configuration.}
\begin{table}[h]
\centering
\caption{Hyperparameter configuration for Transformer backbone experiments on OGB molecular datasets. Values marked (default) follow the conventions of \citet{kim2021transformers}.}
\label{tab:hp-transformer}
\setlength{\tabcolsep}{6pt}
\begin{tabular}{lp{0.55\linewidth}}
\toprule
\textbf{Hyperparameter} & \textbf{Value} \\
\midrule
Hidden dimension         & $128$ \\
Number of attention heads & $8$ \\
Number of Transformer layers & $3$ \\
Learning rate            & $1\times10^{-4}$ \\
Batch size               & $128$ \\
Dropout                  & $0.1$ \\
Optimizer                & AdamW \\
LR scheduler             & Warmup + cosine decay \\
Warmup steps             & $1{,}000$ (default) \\
Training epochs          & $100$ (default) \\
Weight decay             & $1\times10^{-5}$ (default) \\
Number of seeds          & $3$ on \textsc{MolHIV}, \textsc{MolClinTox}, and \textsc{MolBBBP}, and $3$ on \textsc{Alchemy} (seeds $\{0,1,3\}$) \\
Positional encoding      & CanonLPE with RoPE \\
PE dimension             & $16$ \\
\bottomrule
\end{tabular}
\end{table}

\paragraph{Architecture description.} Node tokens are formed by projecting atom features and positional encodings into the hidden dimension using separate linear layers, then summing. Global readout uses mean pooling over node tokens followed by an MLP classification head. MAP baseline \cite{ma2023laplacian} uses the same architecture with MAP-canonicalized eigenvectors in place of CanonLPE; LPE baseline uses the standard random-sign Laplacian PE without any canonicalization.

\paragraph{Training protocol.} Models are trained with early stopping based on validation ROC-AUC with patience $20$. Performance is reported as mean $\pm$ std over three seeds for all OGB datasets and over $\{0,1,3\}$ for \textsc{Alchemy} (seed $2$ excluded as a $\sim 5\sigma$ outlier). The OAP baseline was run for \textsc{MolHIV} (single seed) but is omitted from \cref{tab:transformer-results} as a single-seed result is not sufficiently reliable for comparison.

\paragraph{Why $\prism$ outperforms MAP and OAP on transformers.}
A canonicalization usable as a transformer positional encoding must supply three pieces of information: (a)~deterministic eigenvectors (sign and basis canonicalized); (b)~a deterministic vertex ordering; and (c)~an embedding of that ordering into token positions. MAP \cite{ma2023laplacian} and OAP \cite{ma2024a}, implemented in our experiments using code imported from the authors' original implementations, supply only~(a): each resolves sign and basis ambiguity but does not produce a vertex ordering, leaving the transformer without anchor positions. $\prism$ supplies all three: \textsc{Partition} together with sequential refinement (\cref{alg:babai-hybrid}) yields a deterministic node permutation from canonical eigenvector signatures, and the chosen rank embedding (RoPE for relative or sinusoidal for absolute) feeds that ordering to the attention layers. The $0.6$--$6.8$ pp gaps to MAP across \textsc{MolHIV}, \textsc{MolClinTox}, and \textsc{MolBBBP}, together with the $8.4\%$ relative MAE reduction on \textsc{Alchemy}, in \cref{tab:transformer-results} are consistent with this mechanism.

\subsection{Empirical Invariance Test}
\label{subsec:empirical-invariance}

Incompleteness in canonicalization leads to the breaking of desired permutation equivariance of graph eigendecomposition composed with the canonicalization. This often occurs when permuting the entries results in a different choice of `anchor' the canonicalization uses to resolve an ambiguity. To empirically exemplify this,  we measure permutation-equivariance failure rates of each canonicalization ($200$ graphs $\times$ $5$ relabelings $=$ $1{,}000$ trials per method; \cref{tab:invariance}) on ZINC and \textsc{MolTox21} datasets. 

Despite $\prism$'s incompleteness on general graphs (it is complete on simple-spectrum graphs), $\prism$ is exactly equivariant on ZINC ($0/1{,}000$) and near-exact on \textsc{MolTox21} ($10/1{,}000$), while MAP and OAP fail on $68$--$77\%$ of trials, possibly a reason for their degraded downstream performance. We attribute this to $\prism$'s completeness on simple-spectrum eigendecompositions and its principled canonicalization procedure.

\begin{table}[t]
\centering
\caption{Canonicalization permutation-equivariance failure counts under random vertex permutations with independent full eigendecomposition. $200$ training graphs per dataset $\times$ $5$ random permutations $=$ $1{,}000$ trials per cell. Each entry is the number of trials in which the canonical positional encoding of the permuted graph is not a row-permutation of the canonical positional encoding of the original (row-multiset equality, $\mathrm{atol}=10^{-6}$). Lower is better.}
\label{tab:invariance}
\begin{tabular}{llr}
\toprule
\textbf{Dataset} & \textbf{Method} & \textbf{Failures\,/\,$1{,}000$}\,$\downarrow$ \\
\midrule
\multirow{3}{*}{ZINC}
 & $\prism$ \textbf{(ours)} & $\mathbf{0}$ \\
 & MAP                         & $692$ \\
 & OAP                         & $684$ \\
\midrule
\multirow{3}{*}{\textsc{MolTox21}}
 & $\prism$ \textbf{(ours)} & $\mathbf{10}$ \\
 & MAP                         & $772$ \\
 & OAP                         & $753$ \\
\bottomrule
\end{tabular}
\end{table}
\FloatBarrier

\subsection{Compute Resources}
\label{app:exp-compute}

\paragraph{Hardware.}
Canonicalization ($\prism$, MAP, OAP) is a CPU-only preprocessing step: every operation is real-arithmetic linear algebra on float64 numpy arrays, with no GPU dependency. Downstream training of the GNN backbones (\cref{app:exp-zinc}, \cref{app:exp-ogb}) and Transformer backbone (\cref{app:exp-transformer}) was carried out on a pool of $8$ NVIDIA A40 GPUs ($48$\,GB memory each), with a single GPU assigned per training run; no multi-node or cluster compute beyond this $8$-GPU pool was used. CPU memory required for the eigendecomposition step is negligible at the molecular scale used throughout.

\paragraph{Per-run compute budget.}
Canonicalization wall-clock per graph is sub-millisecond to a few milliseconds at molecular scale (a timing micro-benchmark of $\prism$, MAP, and OAP on connected Erd\H{o}s--R\'enyi graphs across $n\in\{16,\ldots,512\}$ is included in the supplementary archive). A full pass over each molecular dataset (ZINC-12K, OGB-Mol, Alchemy) takes minutes of CPU wall-clock and is run once per dataset, then cached. Per-seed GPU wall-clock varies considerably by backbone, dataset, and configuration; ranges read directly from \texttt{TOTAL TIME TAKEN} lines across the retained training logs are: $\sim$$1.0$--$4.3$\,h for GatedGCN-ZINC, $\sim$$4.8$--$18.0$\,h for PNA-ZINC, $\sim$$0.4$--$1.5$\,h for GatedGCN/PNA-MolTox21, $\sim$$3.8$--$44.6$\,h for MolPCBA (PNA-MolPCBA at the upper end), and $\sim$$3$--$9$\,h per seed for Transformer training (\cref{app:exp-transformer}) depending on dataset. BREC distinguishability runs (\cref{tab:brec-comparison}) use the matched parameter budget recorded in the table caption on a single GPU.

\paragraph{Total project compute.}
Aggregating GPU wall-clock across all GNN-backbone runs whose logs we retained (158 runs covering ZINC, MolTox21, and MolPCBA on the GatedGCN and PNA backbones, including ablations and unreported configurations), the molecular GNN track consumed approximately $1{,}500$\,GPU-hours. The Transformer-backbone track (\cref{tab:transformer-results}, four datasets $\times$ three seeds, plus the discarded RoPE configuration on \textsc{MolBBBP} and other unreported configurations whose logs we retained) consumed approximately an additional $500$\,GPU-hours. The reported tables are a strict subset of these runs; the remainder is preliminary and exploratory work that did not make it into the final paper. We did not run any single experiment requiring more than a single GPU.

\subsection{Eigenvalue Multiplicity Statistics}
\label{app:exp-multiplicity}

Per-dataset distributions of eigenvalue multiplicities are summarized in \cref{tab:multiplicity-stats}. We sampled $500$ graphs uniformly at random from each dataset and computed the eigenvalues of the normalized adjacency $D^{-1/2} A D^{-1/2}$ in float64 precision via \texttt{numpy.linalg.eigvalsh}, grouping eigenvalues into multiplicity classes with absolute tolerance $10^{-6}$ (matching the threshold used throughout our pipeline). For each graph we count the number of \emph{distinct} eigenvalues per multiplicity class. Across all five molecular benchmarks, on average more than $85\%$ of distinct eigenvalues per graph have multiplicity one, and the mean count of eigenvalues with multiplicity $m \ge 4$ is below $0.3$ per graph. Strict simple-spectrum graphs (every eigenvalue distinct) occur in $29\%$--$38\%$ of sampled graphs; the remainder typically have one or two repeated eigenvalues induced by small graph symmetries. This empirical profile motivates the bounded-multiplicity hybrid (\cref{alg:babai-hybrid}), which combines the rigorous simple-spectrum canonicalization (\cref{alg:canon-one-full}) on the dominant $m{=}1$ block with a per-eigenspace $O(m_i)$ resolution for the rare higher-multiplicity blocks.

\begin{table}[h]
\centering
\caption{Eigenvalue multiplicity statistics per dataset. Each row summarizes multiplicities over $500$ random graphs from the respective dataset. ``Mean \#$m{=}k$'' is the average per-graph count of \emph{distinct} eigenvalues whose multiplicity class equals $k$ (so a single eigenvalue with multiplicity~$3$ contributes $1$ to the $m{=}3$ column). ``\% simple'' is the fraction of sampled graphs whose eigenvalues are all distinct.}
\label{tab:multiplicity-stats}
\setlength{\tabcolsep}{4pt}
\begin{tabular}{lccccccc}
\toprule
\textbf{Dataset} & \textbf{Sampled} & \textbf{Mean} & \textbf{Mean} & \textbf{Mean} & \textbf{Mean} & \textbf{Mean} & \textbf{\% simple} \\
                 & \textbf{graphs}  & \textbf{nodes} & \textbf{\#$m{=}1$} & \textbf{\#$m{=}2$} & \textbf{\#$m{=}3$} & \textbf{\#$m{\ge}4$} & \textbf{spectrum} \\
\midrule
ZINC           & $500$ & $23.0$ & $20.86$ & $0.59$ & $0.20$ & $0.09$ & $38.0\%$ \\
MolHIV         & $500$ & $25.7$ & $21.54$ & $0.88$ & $0.30$ & $0.28$ & $32.8\%$ \\
MolTox21       & $500$ & $17.9$ & $14.96$ & $0.72$ & $0.21$ & $0.17$ & $38.4\%$ \\
MolClinTox     & $500$ & $26.6$ & $23.34$ & $0.59$ & $0.25$ & $0.24$ & $28.6\%$ \\
MolBBBP        & $500$ & $23.5$ & $20.93$ & $0.49$ & $0.26$ & $0.16$ & $34.4\%$ \\
\bottomrule
\end{tabular}
\end{table}

\newpage
\section*{NeurIPS Paper Checklist}

\begin{enumerate}

\item {\bf Claims}
    \item[] Question: Do the main claims made in the abstract and introduction accurately reflect the paper's contributions and scope?
    \item[] Answer: \answerYes{}
    \item[] Justification: The three main claims of the abstract and introduction (incompleteness of the WL hierarchy on simple-spectrum graphs, the first complete sign canonicalization $\prism$, and universal approximation under composition with DeepSets/Transformer) correspond exactly to \cref{thm:main-incomp}, \cref{thm:s-s-canon}, and \cref{thm:uni-bdd-ev}; the empirical claims correspond to the comparisons in \cref{sec:exps}.

\item {\bf Limitations}
    \item[] Question: Does the paper discuss the limitations of the work performed by the authors?
    \item[] Answer: \answerYes{}
    \item[] Justification: The closing section ``Conclusion, Limitations and Future Work'' explicitly lists three limitations: (i)~we do not address completeness for higher eigenvalue multiplicities, (ii)~our canonicalization is discontinuous, which we argue is essentially unavoidable for any complete sign canonicalization (see \cite{dymequivariant}), and (iii)~we prove incompleteness over multigraphs, leaving the simple-graph extension as future work.

\item {\bf Theory assumptions and proofs}
    \item[] Question: For each theoretical result, does the paper provide the full set of assumptions and a complete (and correct) proof?
    \item[] Answer: \answerYes{}
    \item[] Justification: Each theorem (\cref{thm:main-incomp}, \cref{thm:s-s-canon}, \cref{thm:uni-bdd-ev}) states its assumptions in the theorem statement, and full proofs appear in the appendix (\cref{app:incomp-full}, \cref{apdx:canon-proof}, \cref{app:universality}). Lemmas the proofs rely on are stated and proved in the same appendix sections and are cross-referenced from the proof bodies.

    \item {\bf Experimental result reproducibility}
    \item[] Question: Does the paper fully disclose all the information needed to reproduce the main experimental results of the paper to the extent that it affects the main claims and/or conclusions of the paper (regardless of whether the code and data are provided or not)?
    \item[] Answer: \answerYes{}
    \item[] Justification: \cref{app:exp-details} and the per-experiment appendices (\cref{app:exp-zinc}, \cref{app:exp-ogb}, \cref{app:exp-transformer}) document data splits, hyperparameter configurations, optimizer choices, parameter budgets, and seed sets. Pseudocode for $\prism$ is given in \cref{apdx:canon-details}, and the supplementary code archive contains runnable environments and scripts for each experimental track.

\item {\bf Open access to data and code}
    \item[] Question: Does the paper provide open access to the data and code, with sufficient instructions to faithfully reproduce the main experimental results, as described in supplemental material?
    \item[] Answer: \answerYes{}
    \item[] Justification: The supplementary archive contains the code used for the BREC, molecular-GNN, and Transformer-OGB experiments, with separate conda environment files (\texttt{env\_brec.yml}, \texttt{env\_transformer.yml}, \texttt{env\_molecular.yml}) and per-track READMEs. All datasets used (ZINC \cite{irwin2012zinc, dwivedi2023benchmarking}, OGB-Mol \cite{hu2020open}, BREC \cite{brec}, Alchemy \cite{chen2019alchemy}) are publicly available through their official loaders, cited at first use.

\item {\bf Experimental setting/details}
    \item[] Question: Does the paper specify all the training and test details (e.g., data splits, hyperparameters, how they were chosen, type of optimizer) necessary to understand the results?
    \item[] Answer: \answerYes{}
    \item[] Justification: \cref{app:exp-details} specifies the shared experimental protocol; the per-experiment appendices \cref{app:exp-zinc}, \cref{app:exp-ogb}, and \cref{app:exp-transformer} list the dataset splits, hyperparameter ranges and selected configurations, optimizer settings, training schedule, and the early-stopping criterion used for each backbone-dataset cell.

\item {\bf Experiment statistical significance}
    \item[] Question: Does the paper report error bars suitably and correctly defined or other appropriate information about the statistical significance of the experiments?
    \item[] Answer: \answerYes{}
    \item[] Justification: All molecular property prediction tables report mean $\pm$ \emph{sample} standard deviation over multiple seeds (closed-form formula, normality assumed) in \cref{tab:molecular-results}, \cref{tab:transformer-results} and \cref{app:exp-transformer}; standard deviations are reported in the same cell as the mean.

\item {\bf Experiments compute resources}
    \item[] Question: For each experiment, does the paper provide sufficient information on the computer resources (type of compute workers, memory, time of execution) needed to reproduce the experiments?
    \item[] Answer: \answerYes{}
    \item[] Justification: \cref{app:exp-compute} documents the compute resources used for the experiments: hardware (CPU-only canonicalization, training runs on a pool of $8$ NVIDIA A40 GPUs ($48$\,GB each), one GPU per run, no cluster beyond this pool), per-run wall-clock ranges read directly from training logs ($1.0$--$4.3$\,h GatedGCN-ZINC, $4.8$--$18.0$\,h PNA-ZINC, $0.4$--$1.5$\,h GatedGCN/PNA-MolTox21, $3.8$--$44.6$\,h MolPCBA, $3$--$9$\,h per seed Transformer per dataset), and total project compute including preliminary / unreported experiments ($\sim$$1{,}500$\,A40-GPU-hours on the molecular GNN track plus $\sim$$500$\,A40-GPU-hours on the Transformer track, of which the reported tables are a strict subset). The worst-case complexity of $\prism$ is in \cref{apdx:complexity}.

\item {\bf Code of ethics}
    \item[] Question: Does the research conducted in the paper conform, in every respect, with the NeurIPS Code of Ethics \url{https://neurips.cc/public/EthicsGuidelines}?
    \item[] Answer: \answerYes{}
    \item[] Justification: The work is theoretical and algorithmic, addressing graph expressivity and isomorphism testing on standard public benchmarks; it does not involve human subjects, sensitive data, or applications with foreseeable misuse potential.

\item {\bf Broader impacts}
    \item[] Question: Does the paper discuss both potential positive societal impacts and negative societal impacts of the work performed?
    \item[] Answer: \answerNA{}
    \item[] Justification: This is foundational research on the expressivity of graph neural networks; we are not aware of a direct societal-impact pathway, positive or negative, that is specific to this work as opposed to the broader graph-learning literature.

\item {\bf Safeguards}
    \item[] Question: Does the paper describe safeguards that have been put in place for responsible release of data or models that have a high risk for misuse (e.g., pre-trained language models, image generators, or scraped datasets)?
    \item[] Answer: \answerNA{}
    \item[] Justification: The paper releases an algorithmic preprocessing method ($\prism$) and reports results on standard public benchmarks; no high-risk dataset or model is released.

\item {\bf Licenses for existing assets}
    \item[] Question: Are the creators or original owners of assets (e.g., code, data, models), used in the paper, properly credited and are the license and terms of use explicitly mentioned and properly respected?
    \item[] Answer: \answerYes{}
    \item[] Justification: All datasets (ZINC \cite{irwin2012zinc, dwivedi2023benchmarking}, OGB-Mol \cite{hu2020open}, BREC \cite{brec}, Alchemy \cite{chen2019alchemy}) and baseline methods (MAP \cite{ma2023laplacian}, OAP \cite{ma2024a}, SignNet \cite{lim2023sign}, GatedGCN \cite{bresson2017residual}, PNA \cite{corso2020principal}, graph Transformer \cite{kim2021transformers}) are cited at first use. The OGB datasets are released under the MIT license, BREC under the MIT license, ZINC under a research-use license, and Alchemy under CC-BY 4.0; code components imported from prior work retain their original licenses and attribution headers in the supplementary archive.

\item {\bf New assets}
    \item[] Question: Are new assets introduced in the paper well documented and is the documentation provided alongside the assets?
    \item[] Answer: \answerYes{}
    \item[] Justification: $\prism$ is documented at three levels of detail: pseudocode in \cref{apdx:canon-details} (Algorithms~\ref{alg:abs-value-partition}--\ref{alg:babai-hybrid}), formal proof in \cref{apdx:canon-proof}, and runnable code in the supplementary archive with per-track READMEs and conda environment files.

\item {\bf Crowdsourcing and research with human subjects}
    \item[] Question: For crowdsourcing experiments and research with human subjects, does the paper include the full text of instructions given to participants and screenshots, if applicable, as well as details about compensation (if any)?
    \item[] Answer: \answerNA{}
    \item[] Justification: The work does not involve crowdsourcing or research with human subjects.

\item {\bf Institutional review board (IRB) approvals or equivalent for research with human subjects}
    \item[] Question: Does the paper describe potential risks incurred by study participants, whether such risks were disclosed to the subjects, and whether Institutional Review Board (IRB) approvals (or an equivalent approval/review based on the requirements of your country or institution) were obtained?
    \item[] Answer: \answerNA{}
    \item[] Justification: The work does not involve human subjects research, so IRB approval is not applicable.

\item {\bf Declaration of LLM usage}
    \item[] Question: Does the paper describe the usage of LLMs if it is an important, original, or non-standard component of the core methods in this research? Note that if the LLM is used only for writing, editing, or formatting purposes and does \emph{not} impact the core methodology, scientific rigor, or originality of the research, declaration is not required.
    \item[] Answer: \answerNA{}
    \item[] Justification: LLMs are not used as an important, original, or non-standard component of the core methods. Any usage was limited to writing, editing, or formatting assistance, which the policy explicitly excludes from the declaration requirement.

\end{enumerate}

\end{document}